%% file: 0-main.tex
\documentclass[12pt,letterpaper]{article}

\newcommand{\blind}{1}

\usepackage{amsmath,amssymb,amsfonts,amsthm,bbm}
\usepackage{mathtools}   
\usepackage{stmaryrd} 

\usepackage{enumitem}   

\usepackage{authblk} 

\usepackage[titletoc,title]{appendix}  

\usepackage[normalem]{ulem}

\usepackage{kpfonts}
\usepackage[T1]{fontenc}

\usepackage[table,xcdraw,dvipsnames]{xcolor}   
\usepackage{hyperref}
\newcommand\myshade{85}
\colorlet{mylinkcolor}{YellowOrange}
\colorlet{mycitecolor}{Aquamarine}
\colorlet{myurlcolor}{violet}

\hypersetup{
  linkcolor  = mylinkcolor!\myshade!black,
  citecolor  = mycitecolor!\myshade!black,
  urlcolor   = myurlcolor!\myshade!black,
  colorlinks = true,
}

\usepackage{caption}
\usepackage{subcaption}

\input{macros/algorithm.tex}

\input{macros/self-defined.tex}
\input{macros/box-theorems.tex}

\input{macros/algorithm.tex}
\providecommand{\keywords}[1]
{
  \small	
  \textbf{Keywords:} #1
}

\usepackage[top=1in, bottom=1in, left=1in, right=1in]{geometry}

\usepackage{setspace}
\def\spacingset#1{\renewcommand{\baselinestretch}%
{#1}\small\normalsize} \spacingset{1}

\usepackage[authoryear]{natbib}  

\usepackage{graphicx}
\usepackage{array}
\graphicspath{{figs/}}
\usepackage{booktabs}
\usepackage{multirow}

\usepackage{comment}

\DeclareMathOperator*{\mat1}{\text{mat}_1}
\DeclareMathOperator*{\vec1}{\text{vec}}
\DeclareMathOperator*{\matk}{{\text{mat}}_m}

\def\ideal{(\text{\footnotesize ideal})}

\renewcommand{\hat}{\widehat}
\renewcommand{\tilde}{\widetilde}
\newcommand{\widebar}{\overline}

\begin{document}

%
%
%
\def\TITLE{
High-Dimensional Tensor Discriminant Analysis with Incomplete Tensors
}


\if1\blind
{
  \title{\bf \TITLE}
  \author{
    Elynn Chen$^\diamond$ \hspace{8ex}
    Yuefeng Han $^\natural$\thanks{Corresponding author.
    } \hspace{8ex}
    Jiayu Li$^\dag$ \\ \normalsize
    \medskip
    $^{\diamond,\dag}$New York University \hspace{6ex}
    $^{\natural}$ University of Notre Dame
    }
    \date{}
  \maketitle
} \fi

\if0\blind
{
  \begin{center}
    {\LARGE\bf \TITLE}
  \end{center}
  \medskip
} \fi

\smallskip

\begin{abstract}
\spacingset{1}
\noindent
Tensor classification is gaining importance across fields, yet handling partially observed data remains challenging. In this paper, we introduce a novel approach to tensor classification with incomplete data, framed within high-dimensional tensor linear discriminant analysis. Specifically, we consider a high-dimensional tensor predictor with missing observations under the Missing Completely at Random (MCR) assumption and employ the Tensor Gaussian Mixture Model (TGMM) to capture the relationship between the tensor predictor and class label. We propose a Tensor Linear Discriminant Analysis with Missing Data (Tensor LDA-MD) algorithm, which manages high-dimensional tensor predictors with missing entries by leveraging the decomposable low-rank structure of the discriminant tensor. Our work establishes convergence rates for the estimation error of the discriminant tensor with incomplete data and minimax optimal bounds for the misclassification rate, addressing key gaps in the literature. Additionally, we derive large deviation bounds for the generalized mode-wise sample covariance matrix and its inverse, which are crucial tools in our analysis and hold independent interest.
Our method demonstrates excellent performance in simulations and real data analysis, even with significant proportions of missing data. 
\end{abstract}

\keywords{tensor classification; high-dimensional statistics; linear discriminant analysis; missing data; Tucker low-rank; optimal rate of convergence}


\newpage
\spacingset{1.9} 

\addtolength{\textheight}{.1in}%


\input{0-all.tex}




%
%

\bibliographystyle{apalike}
\bibliography{main}

%
%
\newpage
\setcounter{page}{1}
\appendix
    \begin{center}
        {\Large Supplementary Material of ``\TITLE''}

        {Elynn Chen, Yuefeng Han and Jiayu Li}
    \end{center}

    \input{0-appendix.tex}

\end{document}

%% file: macros/algorithm.tex
\usepackage[linesnumbered,ruled,vlined]{algorithm2e}

\SetKwInput{KwInput}{Input}                
\SetKwInput{KwOutput}{Output}              

%% file: macros/self-defined.tex
\usepackage{stmaryrd} 

\renewcommand{\bar}{\widebar}
\renewcommand{\hat}{\widehat}
\renewcommand{\tilde}{\widetilde}

\newcommand{\bfm}[1]{\ensuremath{\boldsymbol{#1}}} 

\def\bbone{\mathbbm{1}} 

   \def\bA{\bfm A}  
   \def\bB{\bfm B}  
     
   \def\bD{\bfm D}  
\def\be{\bfm e}   \def\bE{\bfm E}  \def\EE{\mathbb{E}}
\def\bff{\bfm f}    
\def\bg{\bfm g}     
\def\bh{\bfm h}     
   \def\bI{\bfm I}

   \def\bM{\bfm M}  
     
     \def\OO{\mathbb{O}}
     \def\PP{\mathbb{P}}
     
     \def\RR{\mathbb{R}}
   \def\bS{\bfm S}  
     
\def\bu{\bfm u}   \def\bU{\bfm U}  
\def\bv{\bfm v}   \def\bV{\bfm V}  
   \def\bW{\bfm W}  
   \def\bX{\bfm X}  
   \def\bY{\bfm Y}  
\def\bz{\bfm z}   \def\bZ{\bfm Z}

\def\calB{{\cal  B}} \def\cB{{\cal  B}}
 
 \def\cD{{\cal  D}}
\def\calE{{\cal  E}} \def\cE{{\cal  E}}
 \def\cF{{\cal  F}}
 \def\cG{{\cal  G}}
\def\calH{{\cal  H}} \def\cH{{\cal  H}}
\def\calI{{\cal  I}} \def\cI{{\cal  I}}

 \def\cL{{\cal  L}}
\def\calM{{\cal  M}} \def\cM{{\cal  M}}
 \def\cN{{\cal  N}}
\def\calO{{\cal  O}} 
\def\calP{{\cal  P}}

\def\calS{{\cal  S}} \def\cS{{\cal  S}}
\def\calT{{\cal  T}} \def\cT{{\cal  T}}
 
 \def\cV{{\cal  V}}
 
\def\calX{{\cal  X}} \def\cX{{\cal  X}}
\def\calY{{\cal  Y}} \def\cY{{\cal  Y}}
\def\calZ{{\cal  Z}} \def\cZ{{\cal  Z}}

\newcommand{\bfsym}[1]{\ensuremath{\boldsymbol{#1}}}

 \def\bbeta{\bfsym \beta}
              
            \def\bDelta {\bfsym {\Delta}}
               
 \def\bmu{\bfsym {\mu}}                 
 
 \def\btheta{\bfsym {\theta}}

              \def\bSigma{\bfsym \Sigma}
         \def\bLambda {\bfsym {\Lambda}}

\providecommand{\norm}[1]{\left\lVert#1\right\rVert}

\providecommand{\paran}[1]{\left( #1 \right)}

\providecommand{\bbrackets}[1]{\llbracket #1 \rrbracket}

\usepackage{mathtools}
\DeclarePairedDelimiterX{\infdivx}[2]{(}{)}{%
  #1 \; \delimsize\| \; #2%
}

\newcommand{\E}[1]{{\mathbb{E}} \left[ #1 \right]}

\DeclareMathOperator{\Var}{Var}

\DeclareMathOperator{\vect}{vec}

\DeclareMathOperator{\LSVD}{LSVD}

\newtheorem{definition}{Definition}[section]
\newtheorem{assumption}[definition]{Assumption}
\newtheorem{lemma}[definition]{Lemma}

\newtheorem{theorem}[definition]{Theorem}

\theoremstyle{definition}
\newtheorem{remark}{Remark}

\definecolor{royalpurple}{rgb}{0.47, 0.32, 0.66}
\definecolor{greenfresh}{HTML}{00897B}
\definecolor{bluefresh}{HTML}{1E88E5}
\definecolor{redfresh}{HTML}{E53935}

\definecolor{royalpurple}{rgb}{0.47, 0.32, 0.66}

\def\beq{\begin{equation}}
\def\eeq{\end{equation}}

\def\bet{\begin{theorem}}
\def\eet{\end{theorem}}

\def\bel{\begin{lemma}}
\def\eel{\end{lemma}}

\def\tr{\mbox{tr}}

\def\eps{\varepsilon}
\def\lam {\lambda}

\usepackage{setspace}
\usepackage{etoolbox}
\BeforeBeginEnvironment{equation*}{\begin{singlespace}\vspace*{-\baselineskip}}
	\AfterEndEnvironment{equation*}{\end{singlespace}\noindent\ignorespaces}
\BeforeBeginEnvironment{equation}{\begin{singlespace}\vspace*{-\baselineskip}}
	\AfterEndEnvironment{equation}{\end{singlespace}\noindent\ignorespaces}
\BeforeBeginEnvironment{align}{\begin{singlespace}\vspace*{-\baselineskip}}
	\AfterEndEnvironment{align}{\end{singlespace}\noindent\ignorespaces}
\BeforeBeginEnvironment{align*}{\begin{singlespace}\vspace*{-\baselineskip}}
	\AfterEndEnvironment{align*}{\end{singlespace}\noindent\ignorespaces}
\BeforeBeginEnvironment{eqnarray}{\begin{singlespace}\vspace*{-\baselineskip}}
	\AfterEndEnvironment{eqnarray}{\end{singlespace}\noindent\ignorespaces}
\BeforeBeginEnvironment{eqnarray*}{\begin{singlespace}\vspace*{-\baselineskip}}
	\AfterEndEnvironment{eqnarray*}{\end{singlespace}\noindent\ignorespaces}

%% file: macros/box-theorems.tex
\usepackage[framemethod=TikZ]{mdframed} 

\usepackage{fancybox}

%% file: 0-all.tex

\section{Introduction}  \label{sec:intro}

In the data-driven era, tensor structures have become indispensable for capturing multidimensional information across diverse fields, including neuroimaging \citep{zhou2013tensor}, multivariate space time data \citep{hoff2015multilinear,chen2024semi}, economics \citep{chen2020constrained, chen2021statistical}, and transportation \citep{chen2019modeling,han2022rank,han2024cp}. Tensor classification, an extension of traditional vector-based classification, has gained significant research attention in recent years. By preserving tensor structures, researchers have developed techniques that demonstrate empirical improvements in classification accuracy \citep{hao2013linear,wen2024tensorview,jahromi2024variational} and enhanced theoretical understanding of the prediction process \citep{wimalawarne2016theoretical,SONG2023430}. 

As tensor classification advances, a critical challenge emerges: tensor data are often partially observed in real-world applications. This issue stems from factors such as sensor malfunctions, limited data access permissions, and natural barriers in data collection \citep{missingreason1, missingreason3, missingreason2}. Despite its significance, the problem of missing data in tensor classification remains understudied. While theoretical progress has been made in related areas like tensor regression \citep{TensorRegressionImportanceSketching,liang2023imputed} and tensor completion \citep{zhang2019cross, xia2021statistically,ibriga2022covariate}, these frameworks do not directly translate to the classification domain. 


In this paper, we address the challenge of high-dimensional tensor classification with partially observed data, using the framework of Tensor Linear Discriminant Analysis with Missing Data (Tensor LDA-MD). We consider a high-dimensional tensor predictor $\mathcal{X} \in \mathbb{R}^{d_1 \times \cdots \times d_M}$ with Missing Complete at Random (MCR) pattern, where entries are indicated by an independent binary tensor $\mathcal{S} \in \{0,1\}^{d_1\times \cdots \times d_M}$. Specifically, $\mathcal{X}_{i_1...i_M}$ is observed if $\mathcal{S}_{i_1...i_M} = 1$, and missing otherwise. Our approach applies the Tensor LDA framework under the assumption of a Tensor Gaussian Mixture Model (TGMM) for the pair $(\mathcal{X}, Y)$, where $Y \in [K]$ represents the class label for $K \ge 2$ classes. To simultaneously address high-dimensionality and missing data, we propose a novel two-step method. First, we construct a new sample estimator for the discriminant tensor in the presence of missing data. We then refine this estimator by exploiting a Tucker low-rank structure, which achieves consistent estimation in the high-dimensional setting. This approach diverges from traditional high-dimensional LDA by exclusively utilizing the low-rank structure of tensors, without relying on any sparsity assumptions.

Building upon this foundation, our work addresses significant gaps in the existing literature on tensor classification. While previous research has established theoretical foundations for Tensor LDA frameworks with overall sparsity discriminant structures \citep{pan2019covariate}, and tensor envelope extensions \citep{wang2024parsimonious}, these studies have not tackled the problem of missing data or provided rate-optimal guarantees regarding classification accuracy. We bridge these gaps by developing, for the first time, an optimal theory for misclassification rate under the MCR model. Tensor logistic regression is considered a special case of generalized tensor estimation in \cite{HanWillettZhang2022}. Our work relaxes their assumption of independent tensor elements. Additionally, our theoretical analysis of tensor estimation under the tensor-based MCR model is non-trivial due to the complexity of high-dimensional tensor structures and is of independent interest. Our extensive simulations and real data analyses demonstrate excellent classification performance, even under significant proportions of missing data. 

\subsection{Our Contribution}
The contributions of this paper are twofold:
\underline{\it Methodological Innovations:} We introduce Tensor LDA-MD, a novel tensor classification method within the TGMM framework, designed to handle high-dimensionality and missing data. By leveraging the Tucker low-rank structure, our method ensures consistent estimation of the discriminant tensor and provides an optimal classification rule. Besides, we propose a novel estimation method for generalized mode-wise covariance matrices under the MCR model. These innovations broaden the applicability of our approach to a range of tensor learning scenarios in the presence of missing data. 
\underline{\it Theoretical Advancements:}
We establish the first tight minimax optimality bounds for the misclassification rate in tensor classification with missing data. These bounds provide valuable insights into classification accuracy under varying levels of missingness. The technical analysis for the missing data case is significantly more challenging than for complete data, even though the classification procedure and resulting convergence rates appear similar. To facilitate the theoretical analysis, we establish a key technical tool, which is the development of large deviation bounds for the generalized mode-wise sample covariance matrices and their inverse, under the MCR model. 
This technical tool could have broader applications, as it may prove valuable for addressing other related challenges in high-dimensional tensor inference with missing data. Additionally, we derive the convergence rate for discriminant tensor estimation using iterative projection, shedding light on the algorithmic behavior of our method.
Our theoretical and methodological contributions fill a significant gap in the literature, setting the stage for future research in tensor-based machine learning with incomplete data.

\subsection{Related Work}

Our research tackles a unique challenge in tensor learning: 
classification in the presence of missing observations.
Unlike tensor completion \citep{zhang2019cross, xia2021statistically, ibriga2022covariate, zhen2024nonnegative}, which aims to reconstruct missing entries, we focus on supervised learning. While we estimate discriminant tensors, our approach differs fundamentally from low rank tensor decomposition \citep{pmlr-v22-allen12, anandkumar2014guaranteed, anandkumar2014tensor, tensorestimation,zhang2018tensor, chen2024factor} in both learning objectives and theoretical analysis.
In contrast to unsupervised tensor clustering methods \citep{mai2021doubly,HanLuoWangZhang2022,LuoZhang2022}, our supervised approach leverages labeled data to learn decision boundaries and handles partial observations with theoretical guarantees.
Additionally, we advance beyond Tensor Multilinear Discriminant Analysis (MDA) \citep{lai2013sparse,Qunli2014,Franck2023}. Unlike MDA, which reduces tensors to vectors for traditional classification, our method performs direct tensor classification.

Tensor classification has shown great potential across various fields \citep{hao2013linear,wen2024tensorview}, with methods ranging from support tensor machines \citep{hao2013linear,guoxian2016}, tensor neural networks \citep{Chien2018,jahromi2024variational}, and tensor logistic regression \citep{wimalawarne2016theoretical,Liuregression2020,SONG2023430}. However, these approaches often overlook the dependence structure of tensor predictors and lack theoretical guarantees. In tensor Linear Discriminant Analysis (LDA), key advancements include work on sparse discriminant tensor assumptions \citep{pan2019covariate}, and tensor envelopes \citep{wang2024parsimonious}. Yet, these methods do not indicate whether they achieve optimal classification rates and are not equipped to handle missing data.
Our work addresses these gaps by leveraging the tensor MCR model, effectively managing high dimensionality and missing data through a Tucker low-rank assumption on the discriminant tensor. This approach enables us to derive minimax optimal misclassification rates, explicitly accounting for incomplete data.

High-dimensional discriminant analysis has been extensively studied in the vector domain, rooted in Fisher's discriminant analysis \citep{cai2011direct, witten2011penalized, fan2012road, mai2012direct,cai2018rate}. More advanced methods have since been developed, such as feature screening for multiclass LDA \citep{Pan2016} and quadratic discriminant analysis \citep{cai2021convex}. These methods typically impose sparsity on the discriminant vector, enabling efficient computation and interpretable feature selection in high-dimensional spaces.
However, applying these vector-based methods to tensor data can be limiting, as their reliance on sparsity fails to capture the richer structures inherent in multi-dimensional discriminant tensors. Our approach, which focuses on the Tucker low-rank structure, captures complex interactions and dependencies in tensor data, leveraging its intrinsic low-rank nature for a more comprehensive analysis than traditional vector-based methods.

Our work extends the estimation of mode-wise covariance and precision matrices to the MCR model in the tensor setting, building on vector-based studies such as \cite{CAI201655} and \cite{cai2019high}. This extension tackles the unique challenges of high-dimensional tensor structures, particularly in capturing complex relationships across multiple tensor modes. For vector observations with one tensor mode fixed, their methods directly result in i.i.d. random variables, eliminating the need to account for dependencies in the other tensor modes. However, for general tensor observations, when estimating the covariance and precision matrices of a particular tensor mode, we adopt a mild assumption that the tensor variables are correlated across all the other modes, leading to a complex dependency structure.

\subsection{Organization}
The paper is organized as follows: We begin with an introduction to basic tensor analysis notations and preliminaries. Section \ref{sec:model} presents tensor classification with incomplete data within the Tensor LDA framework. In Section \ref{sec:method}, we introduce Tensor LDA-MD, detailing the Tucker low-rank discriminant tensor, estimation of the generalized sample discriminant tensor under the MCR model, and a refinement algorithm for improved estimation. Section \ref{sec:theorems} establishes theoretical results, providing statistical upper bounds on the estimation error and deriving minimax optimal rates for misclassification. Section \ref{sec:extensions} explores the method’s applicability to more general tensor distributions. Section \ref{sec:simu} evaluates the finite-sample performance through simulations, including scenarios with tensor predictors following elliptical distributions. Section \ref{sec:appl} illustrates the practical utility of the method with real-world tensor data. Finally, Section \ref{sec:summ} concludes the paper. All proofs and technical lemmas are provided in the supplementary appendix.

\section{Tensor Classification} \label{sec:model}

\subsection{Preliminary on Tensor Statistics}
For a matrix $A\in\RR^{m\times n}$, its spectral norm and Frobenius norm are denoted as $\|A\|_2$ and $\|A\|_{\rm F}$, respectively. Denote its singular values $\sigma_{\max}(A) = $ $\sigma_1(A)\ge\sigma_2(A)\ge \cdots\ge \sigma_{\min\{m,n\}}(A)\ge 0$ in descending order. An $M$-th order tensor $\calX \in \RR^{d_1 \times \cdots \times d_M}$ is characterized by its multi-index $\calI = (i_1, \dots, i_M)$. The inner product between two tensors $\calX$ and $\calY$ is defined as $\langle \mathcal{X},\mathcal{Y} \rangle = \sum_{\calI} \calX_{\calI} \calY_{\calI}$, where the sum is taken over all indices $\calI$. Using this inner product, the Frobenius norm of $\calX$ is given by $\| \mathcal{X} \|_{\rm F} = \langle \mathcal{X},\mathcal{X} \rangle^{1/2}$. For a matrix $\mathnormal{A} \in \mathbb{R}^{\tilde{d_m} \times d_m}$, the mode-$m$ product $\mathcal{X} \times_m \mathnormal{A}$ yields an $M$-th order tensor of size $d_1 \times \cdots \times d_{m-1} \times \tilde{d_m} \times d_{m+1} \times \cdots \times d_M$. Elementwise, $(\mathcal{X} \times_m A)_{i_1, \ldots, i_{m-1}, j, i_{m+1}, \ldots, i_M} = \sum_{i_m=1}^{d_m} \mathcal{X}_{i_1, \ldots, i_M} \cdot A_{j, i_m}$. The mode-$m$ matricization, denoted as $\text{mat}_m(\mathcal{X})$, rearranges the tensor into a matrix of size $d_m \times \prod_{k\neq m} d_k$. In this representation, the tensor element $\mathcal{X}_{i_1\dots i_M}$ is mapped to the matrix element $(i_m, j)$, where $j=1+\sum\nolimits_{k \neq m} (i_k -1) \prod_{l<k,l \neq m}d_l$. For a subset $S \subseteq \{1, \dots, M\}$, the multi-mode matricization ${\rm mat}_S(\calX)$ is a $\prod_{m \in S} d_m$-by-$\prod_{m \notin S} d_m$ matrix with $(i,j)$-th element mapped as: $i = 1 + \sum_{m \in S} (i_m - 1) \prod_{\ell \in S, \ell < m} d_\ell, \;  j = 1 + \sum_{m \notin S} (i_m - 1) \prod_{\ell \notin S, \ell < m} d_\ell.$
The vectorization is denoted as $\text{vec}(\mathcal{X}) \in \mathbb{R}^d$, where $d=\prod_{m=1}^M d_m$. We also denote $d_{-m} = d/d_m$. Lastly, let $\odot$ be the Hadamard (element-wise) product of matrices and tensors. For a comprehensive review of tensor algebra and operations, refer to \cite{kolda2009tensor}.

The Tensor Normal (TN) distribution is an extension of the matrix normal distribution \citep{hoff2011TN}. To construct a tensor with a TN distribution, consider a random tensor $\calZ \in \RR^{d_1 \times \cdots \times d_M}$ whose elements are independently drawn from $N(0,1)$. Given a mean tensor $\mathcal{M} \in \mathbb{R}^{d_1 \times \cdots \times d_M}$ and mode-wise covariance matrices $\Sigma_m \in \mathbb{R}^{d_m \times d_m}$ for $m \in [M]$, define $\calX = \calM + \calZ \times_{m=1}^M \Sigma_m^{1/2}$. The resulting tensor $\calX$ follows a TN distribution, denoted as $\calX \sim \cT\cN(\cM; \bSigma)$, where $\bSigma := [\Sigma_m]_{m=1}^M$ and $\Sigma_m=\EE[{\rm mat}_m(\calX-\calM){\rm mat}_m^\top(\calX-\calM)]$. The probability density function of $\calX$ is
\begin{equation*} 
f(\calX | \cM, \bSigma) = (2 \pi)^{-d/2}
\left(\prod\nolimits_{m=1}^M |{\Sigma_m}|^{-\frac{d}{2 d_m}}\right)
\exp\left(- \norm{ (\calX - \cM) \times_{m=1}^M \Sigma_m^{-1/2}}_{\rm F}^2/2\right).
\end{equation*}
A notable property of the TN distribution is its equivalence to a multivariate normal distribution under vectorization $\calX \sim \cT\cN(\cM; \bSigma) \Leftrightarrow \vect(\calX) \sim \cN(\vect(\cM); \Sigma_M\otimes\cdots\otimes\Sigma_1)$, where $\otimes$ represents the Kronecker product.

\subsection{Classification with Randomly Missing Data}\label{subsection: Tensor LDA}

We consider a classification problem involving tensor-valued predictors, leveraging the Tensor Gaussian Mixture Model (TGMM) framework. Let $\mathcal{X} \in \mathbb{R}^{d_1 \times d_2 \times \cdots \times d_M}$ denote a high-dimensional tensor predictor and $Y \in \{1, 2, ..., K\}$ represent the class label, where $K \ge 2$. The joint distribution of $(\mathcal{X}, Y)$ under the TGMM is characterized as follows:

\begin{equation}
\label{eqn:tgmm2}
\PP(\mathcal{X}, Y) = \sum_{k=1}^K \pi_k \cdot \mathcal{T}\mathcal{N}(\mathcal{X}; \mathcal{M}_k, \bSigma^{(k)}), \quad \text{where } \pi_k = \PP(Y = k), \quad \sum_{k=1}^K \pi_k = 1.
\end{equation}
Here, $\mathcal{M}_k \in \mathbb{R}^{d_1 \times \cdots \times d_M}$ represents the mean tensor for class $k$, $\bSigma^{(k)} := [\Sigma_m^{(k)}]_{m=1}^M$ denotes the within-class covariance matrices, and $0 < \pi_k < 1$ is the prior probability for class $k$. To account for missing data, we introduce a binary tensor $\mathcal{S} \in \{0,1\}^{d_1 \times \cdots \times d_M}$, independent of $\mathcal{X}$, which indicates the observed entries of $\mathcal{X}$:
\begin{equation}
\mathcal{S}_{i_1...i_M} = \begin{cases}
1, & \text{if } \mathcal{X}_{i_1...i_M} \text{ is observed}, \\
0, & \text{if } \mathcal{X}_{i_1...i_M} \text{ is missing}.
\end{cases}
\end{equation}
The missing data mechanism is assumed to follow the Missing Completely at Random (MCR) model, which can be formally stated as:
\begin{assumption}[Missing Completely at Random (MCR)] \label{asmp:missing}
Define $\mathscr{S}=\{\cS_i^{(k)}\in\{0,1\}^{d_1\times \cdots \times d_{M}}$, $  i=1,...,n_k, \; k=[K]\}$, where $\cS_i^{(k)}$ can be either deterministic or random, but independent of $(\cX,Y)$.  
\end{assumption}

\begin{remark}
Our MCR assumption generalizes various missing data mechanisms explored in the literature. It encompasses uniformly random missing entries \citep{Lounici2011,Wainwright2012,chen2019inference,xiadong2022}, deterministic missing patterns \citep{Deterministicsampling}, and observations following general distributions \citep{Klopp2014, sun2024completionriemannian}. 
\end{remark}

For ease of presentation, we consider the special case of $K = 2$, and observe independently from each class: 
\begin{equation*}
\mathcal{X}_1^{(1)}, \ldots, \mathcal{X}_{n_1}^{(1)} \overset{\text{i.i.d.}}{\sim} \mathcal{T}\mathcal{N}(\mathcal{M}_1; \boldsymbol{\Sigma}), \quad
\mathcal{X}_1^{(2)}, \ldots, \mathcal{X}_{n_2}^{(2)} \overset{\text{i.i.d.}}{\sim} \mathcal{T}\mathcal{N}(\mathcal{M}_2; \boldsymbol{\Sigma}).
\end{equation*}
Here, we assume a common covariance structure across classes (i.e. $\bSigma^{(1)} = \bSigma^{(2)} = \bSigma = [\Sigma_m]_{m=1}^M$). Furthermore, each observation $\mathcal{X}_i^{(k)}$ is subject to missing data, characterized by its corresponding indicator tensor $\mathcal{S}_i^{(k)}$. The classification task involves assigning a future data point $\mathcal{Z}$, drawn from the mixture distribution $\mathcal{Z} \sim \pi_1 \mathcal{T}\mathcal{N}(\mathcal{M}_1; \boldsymbol{\Sigma}) + \pi_2 \mathcal{T}\mathcal{N}(\mathcal{M}_2; \boldsymbol{\Sigma})$, to one of the two classes.

In the scenario where all parameters $\btheta := (\pi_1, \pi_2, \cM_1, \cM_2, \bSigma)$ are known, a classification rule generalizing Fisher’s linear discriminant \citep{anderson2003} for tensor data is
\begin{equation}
\label{eqn:lda-rule}
\delta(\cZ) 
= \bbone\left\{ \langle \cZ - \cM, \; \calB \rangle +\log\Big(\pi_2/\pi_1\Big) \ge 0 \right\},
\end{equation}
where $\cM = (\cM_1 + \cM_2)/2$, $\cD=(\cM_2-\cM_1)$, and $\calB = \cD \times_{m=1}^M \Sigma_m^{-1}$ is the \underline{\it discriminant tensor}. The indicator function $\bbone(\cdot)$ assigns $\cZ$ to class 1 (value 0) or 2 (value 1). As noted by \cite{bishop2006pattern}, under Gaussian assumption with shared covariance matrices, this rule is equivalent to \textit{Bayes' rule}, which is known to be optimal. And the optimal misclassification rate is given by
$R_{\text{opt}}=\pi_1\phi(\Delta^{-1}\log(\pi_2/\pi_1)-\Delta/2)+\pi_2(1-\phi(\Delta^{-1}\log(\pi_2/\pi_1)+\Delta/2)),$
where $\phi$ is the standard normal cumulative distribution function (CDF) and $\Delta=\sqrt{\langle \calB, \; \cD \rangle}$ represents the signal-to-noise ratio.


In practice, estimating the discriminant tensor $\calB$ is challenging. This involves estimating both $\cD=\cM_2-\cM_1$ and the inverse covariance matrices $[\Sigma_m^{-1}]_{m=1}^M$, which together contain $\calO(d)$ free parameters without further assumptions. The difficulty is exacerbated when dealing with incomplete data, as the effective sample size $n_0$ becomes smaller than $n_1+n_2$, and may be significantly smaller than $d$. Additionally, estimation errors in $\cD$ and $[\Sigma_m^{-1}]_{m=1}^M$ propagate to $\mathcal{B}$, impacting the accuracy of the discriminant rule. It is crucial to recognize that the optimal misclassification rate depends on the signal-to-noise ratio $\Delta = \sqrt{\langle \mathcal{B}, \; \mathcal{D} \rangle}$, where inconsistent estimation of $\calB$ precludes a consistent classification rule. 

To address these challenges, we propose a novel two-step approach. First, we develop a method for estimating the sample mean tensor and mode-wise covariance matrices with missing data. This method takes advantage of the fact that estimating mode-wise covariance matrices leverages $\calO(n_0d_{-m})$ column vectors, instead of only using effective sample size $n_0$. And a generalized sample discriminant tensor $\hat \calB$ can be obtained thereafter. Secondly, we enhance the estimation by exploiting a Tucker low-rank structure on $\mathcal{B}$. The motivation for this structure will be discussed in detail in the next section. This approach offers a robust solution for accurate tensor-based classification, as will be demonstrated in the subsequent theoretical analysis.

\begin{remark} 
High-dimensional LDA with randomly missing data has been studied in the vector case by \cite{cai2019high}. One could vectorize the tensor-variate $\calX$ and its indicator tensor $\calS$ to apply their adaptive estimation of $\bbeta$, yielding the discriminant rule: 
\begin{equation*} \delta(\bz) = \bbone\left\{ \hat \bbeta^\top \left(\bz - \frac{\hat \bmu_1 + \hat \bmu_2}{2}\right) + \log\left(\hat \pi_2 / \hat \pi_1\right) \ge 0 \right\}, \end{equation*} 
where $\bz = \vect(\calZ)$ and $\hat \bmu_k = \vect(\hat \calM_k)$. However, this approach assumes overall sparsity and may overlook key tensor structures. Additionally, vectorizing high-dimensional tensors leads to challenges in ultrahigh dimensions. For instance, their method for estimating $\bbeta$ involves solving linear programming (LP) problems, which, for a tensor $\calX \in \RR^{30 \times 30 \times 30}$, requires handling 27,000 inequalities, making it computationally intensive. 
\end{remark}

\section{Tensor LDA-MD with Tucker Low-Rank Structure} \label{sec:method}

In this paper, we propose imposing a Tucker low-rank structure on the discriminant tensor $\calB$. 
To introduce Tucker low-rankness, we decompose $\cD$ as $\cD = \tilde{\cD} \times_1 \bA_1 \times_2 \cdots \times_M \bA_M$, where $\tilde{\cD} \in \mathbb{R}^{r_1 \times \cdots \times r_M}$ is a core tensor and $\bA_m \in \mathbb{R}^{d_m \times r_m}$ ($r_m \ll d_m$) are loading matrices with $\text{rank}(\bA_m) = r_m$. This yields a low-rank representation of $\calB$ as:
\[
\calB = \tilde{\cD} \times_1 (\Sigma_1^{-1}\bA_1) \times_2 \cdots \times_M (\Sigma_M^{-1}\bA_M).
\]
Due to the identification issue \citep{HanLuoWangZhang2022, chen2024semi}, we follow the standard approach and represent the Tucker low-rank discriminant tensor as
\begin{equation}
\calB = \cF \times_1 \bU_1 \times_2 \cdots \times_M \bU_M, \label{eqn:lda-tucker}
\end{equation}
where $\cF \in \mathbb{R}^{r_1 \times \cdots \times r_M}$ is the core tensor, $\bU_m \in \mathbb{R}^{d_m \times r_m}$ are orthogonal factor loadings, and $(r_1, \ldots, r_M)$ denotes the Tucker rank. Incorporating a low-rank structure not only reduces the degrees of freedom and addresses high dimensionality, but is also essential for handling missing data, as shown in the tensor completion literature \citep{missingreason3, liang2023imputed, xia2021statistically}. 


\begin{remark}
The degrees of freedom (DF) for tensors with a Tucker low-rank structure provides key insights. For a rank-$(r_1, \dots, r_M)$ tensor $\calT \in \mathbb{R}^{d_1 \times \cdots \times d_M}$, \cite{zhang2019cross} defines the Tucker DF as: $\text{Tucker DF}(\calT) = r + \sum_{m=1}^M r_m (d_m - r_m),$
where $r = \prod_{m=1}^M r_m$. This serves as a benchmark for the sample size required to ensure convergence in low-rank tensor estimation. Our algorithm achieves a convergence rate proportional to $\sqrt{r + \sum_{m=1}^M d_m r_m}$, which matches this benchmark given that $r_m \ll d_m$. For consistency and simplicity, we roughly define the degrees of freedom for the discriminant tensor $\calB$ as:
\begin{equation*}
    \text{DF}(\calB) = r + \sum_{m=1}^M d_m r_m.
\end{equation*}
\end{remark}

The estimation strategy, as will be detailed in the following subsections, addresses two key aspects. Firstly, we introduce a novel method for constructing a generalized sample estimator $\hat \calB$ under the MCR model, which provids an approximation of the true discriminant tensor with controlled errors. Secondly, algorithmically leveraging the low-rank assumption, we develop an efficient procedure to refine this sample estimator by iteratively projecting it onto its low-dimensional space, and filter out noise, with provable convergence guarantees.

\subsection{Generalized Sample Discriminant Tensor} 
We propose a method to estimate the mean tensor and mode-wise covariance matrices from incomplete data. Recall that $\calS_i^{(k)} \in \RR^{d_1 \times \cdots \times d_M}$ denotes the indicator tensor corresponding to $\calX_i^{(k)}$ for $i=1, \dots, n_k$ and $k=1, 2$. We define the generalized sample mean tensor $\overline{\calX}^{(k)}$ element-wise for $k=1, 2$,
\begin{align*}
\widebar\calX_{\calI}^{(k)} = \frac{1}{n_{*,\calI}^{(k)}} \sum_{i=1}^{n_k} \calX_{i, \calI}^{(k)} \calS_{i, \calI}^{(k)} \quad \text{for all}\ \calI=(i_1, i_2, \dots, i_M), \ \text{where}\ n_{*,\calI}^{(k)} =\sum_{i=1}^{n_k} \calS_{i, \calI}^{(k)}.    
\end{align*}
Define
\begin{equation}
\label{eqn: observed sample size}
n_{*,jt,\ell t} = \sum_{k=1}^2\sum_{i=1}^{n_k} \matk(\calS_{i}^{(k)})_{j t} \matk(\calS_{i}^{(k)})_{\ell t}, \quad n_*=\min_{j,\ell,t} n_{*,jt,\ell t}, 
\end{equation}
for $j,\ell \in [d_m],\; t\in [d_{-m}]\;\; \text{and}\;\; m \in [M]$.
The generalized mode-wise sample covariance matrices for each mode $m$ are then defined as
\begin{align*}
\widehat\Sigma_{m, j\ell} &= \sum_{t=1}^{d_{-m}}\sum_{k=1}^2\sum_{i=1}^{n_k} \frac{\paran{\matk(\calX_{i}^{(k)}-\widebar\calX^{(k)})_{j t} \;\odot\; \matk(\calS_i^{(k)})_{j t}} \paran{{\rm mat_m}^{\top}(\calX_{i}^{(k)}-\widebar\calX^{(k)})_{\ell t} \;\odot \;{\rm mat_m}^{\top}(\calS_i^{(k)})_{\ell t}} }{ d_{-m} \cdot n_{*,jt,\ell t} } , 
\end{align*}
where $j,l \in [d_m]$, $m \in [M]$, $\odot$ denotes the Hadamard product, $d=\prod_{m=1}^M d_m$, and $d_{-m}=d/d_m$. For identifibility issue, update $\hat\Sigma_{M}$ by
\begin{align*}
\hat{\Sigma}_{M} = \hat C_{\sigma}^{-1} \cdot \widehat\Sigma_{M},    
\end{align*}
where
\begin{align*}
\hat C_{\sigma} &=\prod_{m=1}^{M} \hat\Sigma_{m,11}/\hat{\Var}(\cX_{1\cdots1}), \quad
\hat{\Var}(\cX_{1\cdots1}) =(n_{*,1\cdots1}^{(1)}+n_{*,1\cdots1}^{(2)})^{-1} \sum_{k=1}^2\sum_{t=1}^{n_k} (\calX_{t,1\cdots1}^{(k)}-\widebar\calX_{1\cdots1}^{(k)})^2 \cdot \calS_{t,1\cdots1}^{(k)}.
\end{align*}
The generalized sample discriminant tensor is defined as
\begin{equation}\label{eqn:lda-discrim-tensor}
\widehat\calB = (\widebar\calX^{(2)}-\widebar\calX^{(1)}) \times_{m=1}^{M} \widehat\Sigma_m^{-1}.
\end{equation}
Notably, our generalized covariance matrix estimation can be extended to address other incomplete tensor learning problems, making it potentially useful for further research.

We estimate the class probabilities $\pi_1$ and $\pi_2$ using their empirical counterparts: $
\widehat\pi_1=\frac{n_{1}}{n_{1}+n_{2}}, \; \widehat\pi_2=\frac{n_{2}}{n_{1}+n_{2}}.$
For the discriminant tensor $\calB$ to be well-defined, we require positive definiteness of $\widehat\Sigma_m$. Our computations, inspired by \cite{drton2021existence}, suggest that this condition is met with probability 1 if $n_0 > d_m/d_{-m}$, where $n_0$ is the effective sample size. This condition is particularly mild when the dimensions of each mode are comparable (i.e., $d_1\asymp\cdots\asymp d_M$). When this condition is not satisfied, alternative approaches can be employed. One option is to use a perturbed estimator: $\widehat\Sigma_m^{'}=\widehat\Sigma_m+\gamma I_{d_m}$ \citep{ledoit2004well}. The alternatives can be used as plug-in estimators to form $\hat\calB$.


Imposing a low-rank structure on $\calB$ is crucial for high-dimensional LDA. This approach is not merely about improving the estimation accuracy of $\calB$ in the context of missing data, but is {\em fundamental} for developing a consistent LDA rule in high-dimensional settings. As \cite{cai2021convex} demonstrate for complete data scenarios, when dimensions exceed the sample size (i.e. $d \gtrsim (n_1+n_2)$) and class means ($\cM_1,\cM_2$) are unknown, no data-driven method can achieve optimal classification performance, even with known identity covariance matrices.

\subsection{Discriminant Tensor with Tucker Low-rankness}
To recover true $\calB$ from $\hat \calB$, we carefully design projection directions that preserve the signal strength of $\calB$ while reducing the impact of the perturbation $(\hat\calB - \calB)$. To see the mechanism with Tucker low-rankness, let's consider estimating the mode-$1$ loading $\bU_1$ given the true values for other mode-$m$ loadings $\bU_{m}, 2\le m \le M$. We project the noisy $\hat\calB$ onto these known loadings:
\begin{equation} \label{eq:tucker-ideal}
\cZ_{1}
= \hat\calB \times_{m=2}^M \bU_{m}^\top
= \underbrace{\cF\times_1 \bU_{1}}_\text{signal}
+ \underbrace{(\hat\calB - \calB) \times_{m=2}^M \bU_{m}^\top}_\text{noise}.
\end{equation}
\vspace{-2ex}

\noindent
This projection preserves the signal in $\cF$ (since $\bU_1$ is orthogonal) while filtering noise through a carefully designed subspace of dimension $d_1 \times r_2 \times \cdots \times r_M$.
For practical scenarios with unknown loadings, we developed a modified higher-order orthogonal iteration (HOOI) algorithm \citep{Lathauwer2000}. It initializes via higher-order SVD and uses adapted power iterations.
Algorithm \ref{alg:tensorlda-tucker} presents the pseudo-code. The classification rule is then obtained by plugging $\widehat\calB^{\rm tucker}$ into \textit{Fisher's rule} \eqref{eqn:lda-rule}: 
\begin{equation}
\widehat\delta_{\rm tucker}(\cZ) = \bbone\left\{ \langle \cZ - (\widebar\calX^{(1)}+\widebar\calX^{(2)})/2, \; \widehat\calB^{\rm tucker} \rangle + \log\Big(\widehat\pi_2/\widehat\pi_1\Big) \ge 0 \right\}.   \label{eqn:lda-rule-tucker}
\end{equation}
It assigns $\cZ$ to class $1$ with $\hat\delta_{\rm tucker}(\cZ)=0$ or $2$ with $\hat\delta_{\rm tucker}(\cZ)=1$.

\begin{algorithm}[htpb!]
    \SetKwInOut{Input}{Input}
    \SetKwInOut{Output}{Output}
    \Input{Initial tensor $\widehat\calB$, 
    Tucker rank $(r_1,...,r_M)$, tolerance parameter $\epsilon>0$, maximum number of iterations $T$}
    \Output{$\widehat\calB^{\rm tucker}=\widehat\calB\times_{m=1}^M \widehat \bU_{m}\widehat \bU_{m}^\top$, and 
    $\widehat \cF = \big|\widehat\calB\times_{m=1}^M \widehat \bU_{m}^\top\big|$,
    where $\widehat \bU_{m}=\widehat \bU_{m}^{(t)}$.}
    
    
    Let $t=0$, initiate via high order SVDs, 
    $\widehat \bU_m^{(0)}={\rm LSVD}_{r_m}({\rm mat}_m(\widehat\calB)),\ 1\le m\le M$, where LSVD$_{r_m}$ represents top $r_m$ left singular vectors
    
    \Repeat{$t = T$ {\bf or} $\max_{m}\|\widehat \bU_{m}^{(t)}\widehat \bU_{m}^{(t)\top}- \widehat \bU_{m}^{(t-1)}\widehat \bU_{m}^{(t-1)\top} \|_{2}\le \epsilon$}{
     Set $t=t+1$. 
     
         \For{$m = 1$ to $M$}{
            Calculate $\cZ_m=\widehat\calB \times_1 \widehat \bU_{1}^{(t)\top} \times_2 \cdots \times_{m-1} \widehat \bU_{m-1}^{(t)\top} \times_{m+1} \widehat \bU_{m+1}^{(t-1)\top} \times_{m+2}\cdots\times_M \widehat \bU_{M}^{(t-1)\top}$
            
            Perform high order SVD,
            $\widehat \bU_m^{(t)}={\rm LSVD}_{r_m}({\rm mat}_m(\cZ_m))$.
         }
    }
    \caption{A generic iterative orthogonal projection algorithm}
    \label{alg:tensorlda-tucker}
\end{algorithm}

\section{Theoretical Results under Tucker Low-Rankness}
\label{sec:theorems}
In this section, we establish an optimality theory for high-dimensional Tensor LDA with missing data. We analyze the theoretical properties of Algorithm \ref{alg:tensorlda-tucker} and derive both upper and lower bounds for the excess misclassification risk. These bounds collectively yield the minimax rates of misclassification, demonstrating that Tensor LDA-MD achieves optimal performance. To simplify notations, denote $d=\prod_{m=1}^M d_m$, $d_{-m}=d/d_m$, and $r=\prod_{m=1}^M r_m$, $r_{-m}=r/r_m$. Besides, denote $n=\min\{n_{1},n_{2}\}$. 

Under the MCR model, suppose that the missingness pattern $\mathscr S$ is a realization of a distribution $\cG$. We consider the distribution space $\Upsilon(n_0)$ given by
\begin{align}
\label{eqn: distribution space}
\Upsilon(n_0) =\{ \cG: \PP_{\mathscr S\sim \cG}  (c_1 n_0 \le n^*(\mathscr S) \le c_2 n_0 ) \ge 1-  d^{-c_3} \},
\end{align}
for some constants $c_1, c_2 > 0$, and $n^*(\mathscr S)$ is defined for $\mathscr S$ as in \eqref{eqn: observed sample size}.

The following theorem provides probabilistic upper bounds on the estimation error of $\hat \calB^{\rm tucker}$ from Algorithm \ref{alg:tensorlda-tucker}, using the generalized sample discriminant tensor as the initialization. 

\begin{theorem}[Tucker Loadings and Low-Rank Discriminant Tensor.]
\label{theorem: tucker} 
Suppose there exist constant $C_0 > 0$ such that $C_0^{-1} \leq \lambda_{\min}(\otimes_{m=1}^M \Sigma_m) \leq \lambda_{\max}(\otimes_{m=1}^M \Sigma_m) \leq C_0,$ where $\lambda_{\min}(\cdot)$ and $\lambda_{\max}(\cdot)$ being the smallest and largest eigenvalues of a matrix, respectively. Recall $[\Sigma_m]_{m=1}^M$ are the mode-wise covariance matrices of the tensor predictor $\calX$. Consider the distribution space $\Upsilon(n_0)$ under the MCR model in \eqref{eqn: distribution space}. Denote $\sigma_m = \sigma_{r_m} (\matk(\calB))$. Then 
there exists constant $C_{\rm gap}>0$ which do not depend on $d_m, r_m, \sigma_m$, such that whenever 
\begin{align*}
\sigma_m \geq C_{\rm gap}  \left(\sqrt{\frac{d_m+d_{-m}}{n_0}} + \frac{\|\matk(\calM_2 - \calM_1)\|_2 \max_{k\le M} d_k}{\sqrt{n_0 d}} \right), m=1,...,M,    
\end{align*}
after certain number of iterations in Algorithm \ref{alg:tensorlda-tucker}, the following upper bounds hold with probability at least $1 - n_0^{-c} -d^{-c}- \sum_{m=1}^M \exp(-cd_m)$,
\begin{align}
\left\|\hat \bU_{m}\hat \bU_{m}^\top - \bU_{m} \bU_{m}^\top \right\|_{2} & \le C \max_{1\le m\le M} \left( \frac{\sqrt{d_m+r_{-m}}}{\sqrt{n_0}\sigma_m}+\frac{\|\matk(\cM_2-\cM_1)\|_2\max_k d_k}{\sqrt{n_0 d} \sigma_m} \right), \label{eqn:lda-loading-tucker} \\
\left\| \hat \calB^{\rm tucker}  -  \calB \right\|_{\rm F} & \le C \sqrt{ \frac{r+\sum_{m=1}^M d_m r_m}{n_0 }  }  + C\|\cM_2-\cM_1 \|_{\rm F} \cdot \frac{\max_{1\le m\le M}d_m}{\sqrt{n_0 d}} , \label{eqn:lda-b-tucker} \\
\frac{\left\| \hat \calB^{\rm tucker}  -  \calB \right\|_{\rm F} }{\left\| \calB \right\|_{\rm F}} & \le \frac{C}{{\left\| \calB \right\|_{\rm F}}} \sqrt{ \frac{r+\sum_{m=1}^M d_m r_m}{n_0}  }  + \frac{C\max_{1\le m\le M}d_m}{\sqrt{n_0 d}}  \label{eqn:lda-c-tucker}  .
\end{align}
\end{theorem}

\begin{remark}
The error bounds for the complete data case can be derived as a special case of our results. In this scenario, we have $n^*(\mathscr S)=n$. By substituting $n_0$ with $n$ in \eqref{eqn:lda-loading-tucker}, \eqref{eqn:lda-b-tucker}, and \eqref{eqn:lda-c-tucker}, we obtain the corresponding error bounds for the complete data setting.
\end{remark}

\begin{remark}
When $\Sigma_m=c_m \bI_{d_m}$ for all $1\le m\le M$, Theorem 3 and Lemma 4 in \cite{cai2018rate} allow for an improved signal-to-noise ratio condition, given by
\begin{align*}
\sigma_m \geq C_{\rm gap}  \left( \left(\frac{d}{n_0}\right)^{1/4}+\sqrt{\frac{d_m}{n_0}} + \frac{\|\matk(\calM_2 - \calM_1)\|_2 \max_{k\le M} d_k}{\sqrt{n_0 d}} \right), m=1,...,M.    
\end{align*}
The improvement arises from the i.i.d. assumption $\Sigma_m=c_m \bI_{d_m}$ for the mode-wise correlation structure. For a tensor-variate $\calX \in \RR^{d_1 \times \cdots \times d_M}$ distributed as $\cT\cN (0; \bSigma)$, this assumption leads to an i.i.d. column-wise distribution in $\matk(\calX)$. Consequently, these columns can be treated as realizations of i.i.d. Gaussian random vectors $\matk(\calX)_{\cdot 1}$, effectively increasing the ``sample size'' 
and thereby improving estimation accuracy. In contrast, our theoretical analysis assumes a general mode-wise covariance structure, where these columns are neither independent nor identically distributed, which leads to more conservative error bounds.
\end{remark}

\begin{remark}
\label{remark:7}
The statistical convergence rates in \eqref{eqn:lda-loading-tucker}, \eqref{eqn:lda-b-tucker}, and \eqref{eqn:lda-c-tucker} contain two parts. The first part is analogous to the rate in conventional tensor Tucker decomposition, while the second part reflects the estimation accuracy of the mode-m precision matrices $\Sigma_m^{-1}$. Unlike \cite{wang2024parsimonious}, where $\|\calM_2-\calM_1\|_{\rm F}$ is bounded by a constant, we impose no such constraint on this difference. This lack of constraint makes the normalized estimation error bound in \eqref{eqn:lda-c-tucker} critical for understanding the algorithm's performance in classification tasks. 
By normalizing the error relative to $\|\calB\|_{\rm F}$, the bound in \eqref{eqn:lda-c-tucker} reveals a clearer convergence rate, where the second term does not increase with the difference. Additionally, recall that $\calB = (\calM_2-\calM_1) \times_{m=1}^M \Sigma_m^{-1}$. Larger differences between the mean tensors result in a larger $\|\calB\|_{\rm F}$, which reduces the normalized estimation error, aligning with the expectation that classification becomes easier as the differences between class means diverge. This normalization idea is also employed in deriving upper bounds for the excess misclassification rate, as discussed later.
\end{remark}


The analysis in Theorem \ref{theorem: tucker} leverages newly developed large deviation inequalities for the generalized mode-wise sample covariance matrices and their inverse, as well as the convergence rate of the normalization constant under the MCR model, as shown in Lemma \ref{lemma:precision matrix}. It also applies spectrum perturbation theorems \citep{wedin1972perturbation,cai2018rate,han2020iterative} to analyze the projected discriminant tensor estimator after each iteration. These theorems are used under conditions of sufficiently strong signal strength, which helps to quantify the impact of tensor perturbations on estimation accuracy.

Next, we develop an optimality theory for high-dimensional Tensor LDA-MD method. Specifically, consider two random samples $
\calX_1^{(1)}, \cdots, \calX_{n_{1}}^{(1)} \stackrel {\text{i.i.d.}}{\sim} \cT\cN(\cM_1; \bSigma)$ and $\calX_1^{(2)}, \cdots, \calX_{n_{2}}^{(2)} \stackrel {\text{i.i.d.}}{\sim} \cT\cN(\cM_2; \bSigma)$, where $\bSigma= [\Sigma_m]_{m=1}^M$. We aim to classify a future data point $\calZ$ into one of two classes by tensor LDA rule defined in \eqref{eqn:lda-rule-tucker}, where $\calZ$ is drawn from these distributions with prior probabilities $\pi_1$ and $\pi_2$, respectively. The classifier's performance is measured by the misclassification error
\[
R_{\btheta}(\hat\delta_{\rm tucker}) = \PP_{\btheta}\big(label(\cZ) \neq \hat\delta_{\rm tucker}(\cZ)\big),
\]
where $\PP_{\btheta}$ denotes the probability with respect to $\cZ \sim \pi_1 \cT\cN(\cM_1; \bSigma) + \pi_2 \cT\cN(\cM_2; \bSigma)$ and $label(\cZ)$ is the true class of $\cZ$. We are interested in the excess misclassification risk $R_{\btheta}(\hat\delta_{\rm tucker}) - R_{\rm opt}(\btheta)$, to evaluate the classifier's performance relative to the oracle Fisher's rule, whose misclassification rate is given by $R_{\text{opt}}=\pi_1\phi(\Delta^{-1}\log(\pi_2/\pi_1)-\Delta/2)+\pi_2(1-\phi(\Delta^{-1}\log(\pi_2/\pi_1)+\Delta/2))$. Recall $\phi$ is the CDF of standard normal and $\Delta=\sqrt{\langle \calB, \; \cD \rangle}$.

\begin{theorem}[Upper bound of misclassification rate]
\label{theorem: upper bound}
Assume the conditions in Theorem \ref{theorem: tucker} all hold. Consider the situation $(\sum_{m=1}^M d_m r_m + r)/n_0 = o(\Delta)$ with $n_0 \rightarrow \infty$.

\noindent (i) If $\Delta\le c_0$, for some $c_0>0$, then with probability at least $1-n_0^{-c}-d^{-c}-\sum_{m=1}^M e^{-cd_m }$, the misclassification rate of classifier $\hat\delta_{\rm tucker}$ satisfies
\begin{equation}\label{eqn:lda-mis-tucker1}
R_{\btheta}(\hat\delta_{\rm tucker}) -R_{\rm opt}(\btheta) \le C\left(\frac{\sum_{m=1}^M d_m r_m+ r}{n_0}   \right),
\end{equation}
for some constant $C,c>0$.\\
(ii) If $\Delta\to\infty$ as $n\to \infty$, then there exists $\vartheta_n=o(1)$, with probability at least $1-n_0^{-c}-d^{-c}-\sum_{m=1}^M e^{-cd_m }$, the misclassification rate of classifier $\hat\delta_{\rm tucker}$ satisfies
\begin{equation}\label{eqn:lda-mis-tucker2}
R_{\btheta}(\hat\delta_{\rm tucker}) -R_{\rm opt}(\btheta) \le C \exp\left\{-\left(\frac18+\vartheta_n\right)\Delta^2 \right\} \left(\frac{\sum_{m=1}^M d_m r_m+ r}{n_0}   \right),
\end{equation}
for some constant $C,c>0$.
\end{theorem}

Theorem \ref{theorem: upper bound} presents upper bounds for the excess misclassification risk. Unlike the bounds in Theorem \ref{theorem: tucker} which contain two parts, the bounds in Theorem \ref{theorem: upper bound} contain only one part. This indicates that the error from estimating mode-$m$ precision matrices becomes negligible in the excess misclassification risks, which aligns with Remark \ref{remark:7}'s discussion on the normalized estimation error bound: After normalization, the second term in \eqref{eqn:lda-c-tucker} is $\mathcal{O}(\sqrt{d_m}/\sqrt{n_0 d_{-m}})$, and the error bound is dominated by the first term.

The low-rankness assumption inherent in our approach leads to substantially sharper rates, characterized by $(\sum_{m=1}^M d_m r_m + r) / n_0$, compared to methods without structural assumptions which typically involve terms with $d$ (the ambient dimension). This improvement highlights the efficiency gained by leveraging the tensor structure and low-rank properties in high-dimensional classification tasks, resulting in more precise and informative bounds on the excess misclassification risk.

To understand the challenges of high-dimensional tensor linear discriminant analysis and establish our classifier's optimality, we derive minimax lower bounds for the excess misclassification risk. Consider the parameter space of Tucker low-rank discriminant tensors:

\begin{align*}
\mathcal{H}= \Big\{\;& \theta=(\mathcal{M}_1, \mathcal{M}_2, \boldsymbol{\Sigma}): \mathcal{M}_1,\mathcal{M}_2 \in \mathbb{R}^{d_1 \times \cdots \times d_M}, \; \boldsymbol{\Sigma} = [\Sigma_m]_{m=1}^M, \; \Sigma_m \in \mathbb{R}^{d_m \times d_m}, \; \text{for some} \; C_0>0, \\
&C_0^{-1} \leq \lambda_{\min}(\otimes_{m=1}^M \Sigma_m) \leq \lambda_{\max}(\otimes_{m=1}^M \Sigma_m) \leq C_0, \; \mathcal{B} =\mathcal{F}\times_{m=1}^M \bU_m\; \text{with}\; \mathcal{F}\in \mathbb{R}^{r_1 \times \cdots \times r_M} \\
&\text{and} \; \bU_m \; \text{is a} \; d_m\times r_m \; \text{orthogonal matrix} \Big\}   . 
\end{align*}
Theoretical analysis involves carefully constructing a finite collection of subsets of $\mathcal{H}$ that characterize the problem's hardness. For each parameter space $\mathcal{H}_{\ell} \subset \calH$, we employ a key technique: transferring the excess misclassification risk to an alternative risk defined as $L_{\boldsymbol{\theta}}(\delta) = \mathbb{P}_{\boldsymbol{\theta}}(\delta(\mathcal{Z}) \neq \delta_{\theta}(\mathcal{Z}))$, where $\delta_{\theta}(\mathcal{Z})$ is oracle Fisher's rule. This transformation, established in Lemma \ref{lemma:the first reduction}, facilitates the analysis as the alternative risk function is much easier to handle mathematically.

\begin{theorem}[Lower bound of misclassification rate] \label{theorem: lower bound}
Under the TGMM, the minimax risk of excess misclassification error over the parameter space $\calH$ satisfies the following conditions.

\noindent (i) If $c_1<\Delta \le c_2$ for some $c_1,c_2>0$, then for any $\alpha>0$, there exists some constant $C_{\alpha}>0$ such that
\begin{equation}\label{eqn:lda-lbd-tucker1}
\inf_{\hat \delta_{\rm tucker}} \sup_{\theta \in \calH} \PP\left(R_{\btheta}(\hat\delta_{\rm tucker}) -R_{\rm opt}(\btheta) \ge C_{\alpha}  \frac{\sum_{m=1}^M d_m r_m + r}{n_0}  \right) \ge 1-\alpha   .   
\end{equation}
(ii) If $\Delta\to\infty$ as $n\to \infty$, then then for any $\alpha>0$, there exists some constant $C_{\alpha}>0$ and $\vartheta_n=o(1)$ such that
\begin{equation}\label{eqn:lda-lbd-tucker2}
\inf_{\hat \delta_{\rm tucker}} \sup_{\theta \in \calH} \PP\left(R_{\btheta}(\hat\delta_{\rm tucker}) -R_{\rm opt}(\btheta) \ge C_{\alpha} \exp\left\{-\left(\frac18+\vartheta_n\right)\Delta^2 \right\} \frac{\sum_{m=1}^M d_m r_m + r}{n_0}  \right) \ge 1-\alpha   .   
\end{equation}
\end{theorem}
Combined with the upper bounds of the excess misclassification risk in Theorem \ref{theorem: upper bound}, the convergence rates are minimax rate optimal.

\begin{remark}
The misclassification error bounds for the complete data case can similarly be derived as a specific instance of our results. In this situation, $n^*(\mathscr S)=n$. By replacing $n_0$ with $n$ in Theorems \ref{theorem: upper bound} and \ref{theorem: lower bound}, we obtain the corresponding error bounds for the complete data scenario.
\end{remark}

\section{Extensions}
\label{sec:extensions}
We extend the Tensor LDA-MD framework by generalizing the tensor predictor $\calX$ to non-Gaussian tensor variables. Following \cite{shao2011sparse}, we relax the normality assumption of the tensor predictor such that for any $d$-dimensional deterministic vector $\bh$ with $\norm{\bh}_2=1$ and any $t \in \RR$, we assume
\begin{equation}
\label{eqn: nonnormal}
\PP \left( \bh^\top (\Sigma_M\otimes\cdots\otimes \Sigma_1)^{-1/2}  \text{vec}(\mathcal{X} - \mathcal{M}) \leq t \right) =: \Psi(t),
\end{equation}
where $\calM$ is the mean tensor, $\bSigma=:[\Sigma_m]_{m=1}^M$ are the mode-wise covariance matrices, and $\Psi(t)$ is an unknown distribution function symmetric about 0 and independent of $\bh$. One notable class of distributions satisfying \eqref{eqn: nonnormal} is the tensor elliptical distribution, with a density of the form:
\begin{equation}
\label{eqn: elliptical}
f(\mathcal{X}) = c \cdot g\left( \text{vec}(\mathcal{X} - \mathcal{M})^\top \left(\Sigma_M\otimes\cdots\otimes \Sigma_1 \right)^{-1} \text{vec}(\mathcal{X} - \mathcal{M}) \right)
\end{equation}
where $g(\cdot)$ is a monotone function on $[0, \infty)$, and $c$ is a normalization constant. As demonstrated by \cite{fang1990statistical}, for Tensor LDA with known $\calM_k$ and common covariance matrices $\bSigma$, Fisher's rule remains optimal when the tensor variable follows a tensor elliptical distribution. Furthermore, for binary classification, the LDA achieves the optimal misclassification rate given by 
\begin{equation*} 
R_{\text{opt}}=\pi_1 \cdot \Psi(\Delta^{-1}\log(\pi_2/\pi_1)-\Delta/2)+\pi_2 \cdot \bar{\Psi}(\Delta^{-1}\log(\pi_2/\pi_1)+\Delta/2),
\end{equation*} 
where $\bar{\Psi}(\cdot) = 1 - \Psi(\cdot)$ and $\Delta = \sqrt{\langle \calB, \cD \rangle}$.

In scenarios involving incomplete data, we can derive the conditional classification error of the Tensor LDA rule \eqref{eqn:lda-rule}. Given $\calX^{(k)}$ drawn from a tensor elliptical distribution and the indicator tensor $\calS^{(k)}$, $k=1,2$, this error extends to:
\begin{align*}
R(\hat\delta) =& \frac{n_1}{n_1 + n_2} \Psi\left(\hat \Delta^{-1}\log(n_2/n_1) -\frac{\langle \hat \cM - \cM_1, \; \hat \calB \rangle}{\hat \Delta} \right) \\
& + \frac{n_2}{n_1 + n_2} \bar \Psi\left(\hat \Delta^{-1}\log(n_2/n_1) - \frac{\langle \hat \cM - \cM_2, \; \hat \calB \rangle}{\hat \Delta} \right).
\end{align*}
Here, $\hat\delta$ denotes the Tensor LDA rule, while $\hat \Delta$, $\hat \cM$, and $\hat \calB$ represent estimates derived from the incomplete data. In Section \ref{sec:simu}, we conduct numerical experiments applying our proposed classifier $\hat \delta_{\rm tucker}$ on two concrete examples of tensor-valued elliptical distributions:
\begin{enumerate}
    \item[i.] The tensor $t$-distribution, denoted as $\mathcal{T}t(\mathcal{M}, \bSigma, \nu)$ (defined in Definition 1 in \cite{wang2023regression}). The degrees of freedom parameter $\nu$ controls the tail behavior, with smaller values leading to heavier tails.
    
    \item[ii.] The tensor-variate (asymmetric) Laplace distribution, denoted as $\mathcal{T}\mathcal{L}(\mathcal{M}, \bSigma, \lambda)$ (Definition 2.3 in \cite{yurchenko2021}). The scale parameter $\lambda$ serves as the rate parameter for an exponential distribution.
\end{enumerate}
The empirical results show that our method remains efficient when tensor predictors follow elliptical distributions and performs well even with a significant proportion of missing data.

\section{Simulation}
\label{sec:simu}
In this section, we conduct extensive simulations to validate the theoretical properties of Tensor LDA-MD, focusing on scenarios with missing data. Diverse experimental setups and evaluation metrics corroborate our findings, and comparisons with advanced classifiers like CATCH \citep{pan2019covariate} demonstrate our method's effectiveness.

\subsection{Data Generation}

Our simulations utilize tensor-variate training data $\calX^{(k)}$ from $\cT\cN(\cM_k; \bSigma)$ for $k \in \{1, 2\}$, with $n_k$ samples per class and an additional 500 samples per class for testing. We simplify the common covariance matrices $\bSigma=[\Sigma_m]_{m=1}^M$ to diagonal form, noting that more general forms would only affect the theoretical upper bound by a multiplicative constant. Specifically, each $\Sigma_m$ is an identity matrix scaled by $c^{1/M}$, while $\calM_1=0$ and $\calM_2=c \cdot\calB$, as defined in \eqref{eqn:lda-rule}. The discriminant tensor $\mathcal{B}$ is generated in two steps. First, we generate the core tensor $\mathcal{F} \in \mathbb{R}^{r_1 \times \cdots \times r_M}$, ensuring similar magnitudes of singular values across matricizations, with controlled minimum singular values $\sigma_m$. Second, we generate Tucker loading matrices $\bU_m \in \mathbb{R}^{d_m \times r_m}$ via QR decomposition of random matrices, aligning $\mathcal{B}$'s signal strength with the core tensor. The missing mechanism independently observes each tensor entry with probability $p=1-\epsilon$, where $\epsilon \in (0,1)$ is the missing rate. We aim to explore the joint impact of $n_0$, signal strength $\sigma_m$, and dimensions $d_m$ and $r_m$ on the upper bounds of estimation errors.

Our simulations explore various scenarios. We first fix dimensions at $d_1 = d_2 = d_3 = 40$ and low-rank components at $r_1 = r_2 = r_3 = 5$, while varying the missing rate $\epsilon$ from 0 to 0.7. We also examine the complete data scenario ($\epsilon=0$), fixing sample sizes at $n_1 = n_2 = 600$, scaling factor $c = 1$, signal strength $\sigma_m = 1.5$, and Tucker low-rank at 5 for each mode. We then vary the dimensions, exploring $d_1 = d_2 = d_3 = 30, 40, 60$ for order-3 tensors, and $d_1 = d_2 = d_3 = d_4 = 20, 30$ for order-4 tensors.


\vspace{-1em}
\subsection{Validation of Theoretical Results}
We evaluate our method using three key metrics: the loading matrix estimation error $\max_{1 \leq m \leq M} \|\hat{\bU}_m \hat{\bU}_m^\top - \bU_m \bU_m^\top \|_2$, the discriminant tensor relative error $\left\| \hat \calB^{\rm tucker} - \calB \right\|_{\rm F} / \left\| \calB \right\|_{\rm F}$, and the misclassification rate $R_{\btheta}(\hat\delta)$. Here, $\hat\delta$ represents either the generalized sample discriminant tensor based classifier $\hat\delta_{\rm sample}$ or the proposed Tucker low-rank discriminant tensor based classifier $\hat\delta_{\rm tucker}$. 

\noindent \textbf{Convergence Rate of the Estimation Error:} Figures \ref{fig:estimation_errors}--\ref{fig:max_loading_estimation_errors} show the estimation error decreasing with increasing signal strength and sample size, and higher missing data rates consistently lead to larger errors, aligning with our theoretical predictions. The error for both $\calB$ and loading matrices $\bU_m$ is inversely proportional to $\sqrt{n_0}$, where $n_0$ is the effective sample size accounting for missing data, defined in \eqref{eqn: observed sample size}. Based on the missing mechanism, these errors are also inversely proportional to $\sqrt{n}$, where the relationship can be expressed as:
\begin{equation*}
\max_{1 \leq m \leq M} \|\hat{\bU}_m \hat{\bU}_m^\top - \bU_m \bU_m^\top \|_2 \quad \text{and} \quad \left\| \hat \calB^{\rm tucker}  -  \calB \right\|_{\rm F} /\left\| \calB \right\|_{\rm F} \propto 1/\sqrt{n_0} \approx 1/(\sqrt{n}(1-\epsilon))  .
\end{equation*}

\begin{figure}[ht!]
    \centering
    \begin{subfigure}[b]{0.4\textwidth}
        \centering
        \includegraphics[width=\textwidth]{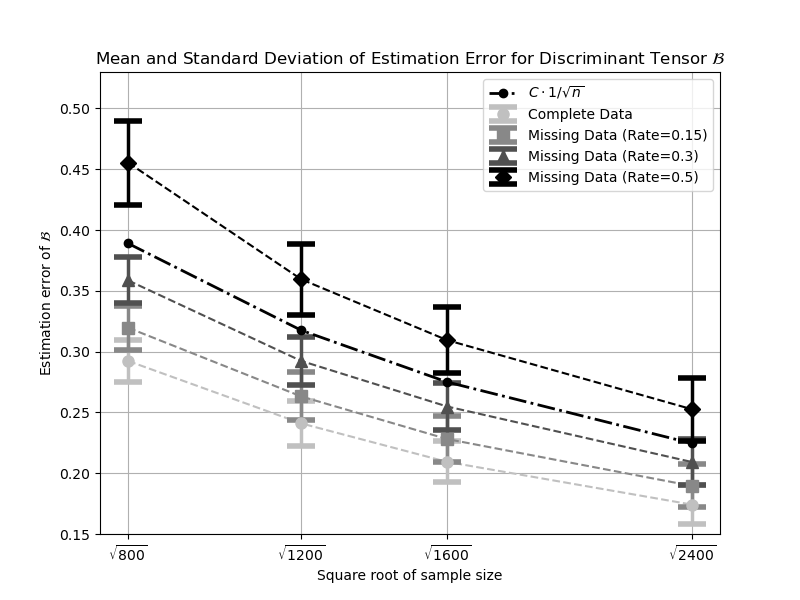}
        \caption{\small Estimation Error for \( \calB \) by Sample Size}
        \label{fig:max_loading_error}
    \end{subfigure}
    \hfill
    \begin{subfigure}[b]{0.4\textwidth}
        \centering
        \includegraphics[width=\textwidth]{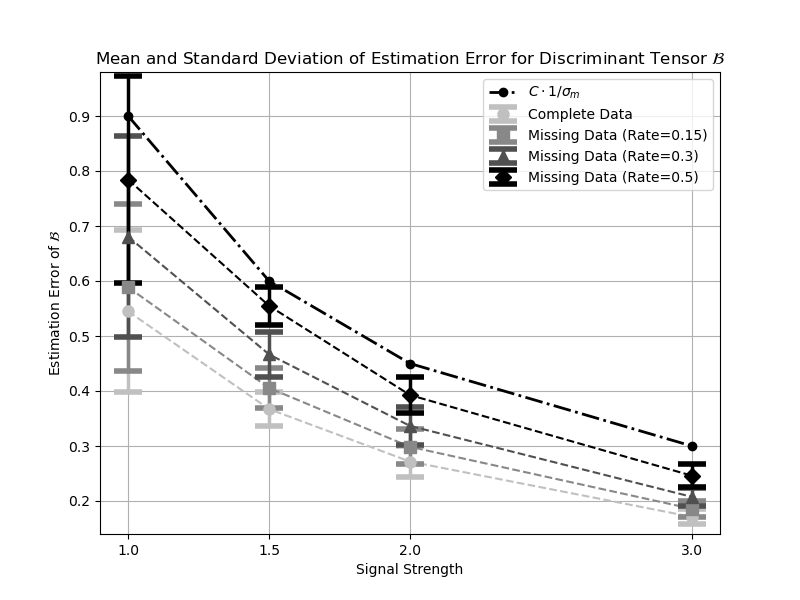}
        \caption{\small Estimation Error for \( \calB \) by Signal}
        \label{fig:estimation_error_B_signal}
    \end{subfigure}
    \caption{Figures (a) and (b) compare estimation errors of $\calB$ (i.e. $\left\| \hat \calB^{\rm tucker}  -  \calB \right\|_{\rm F} /\left\| \calB \right\|_{\rm F}$) under different settings. (a) keeps signal strength $\sigma_m=1.5$ and $c = 4/5$ while increasing sample size, whereas (b) keeps sample size $n=800$ and $c=1$ while increasing signal strength.}
    \label{fig:estimation_errors}
\end{figure}

\begin{figure}[ht!]
    \centering
    \begin{subfigure}[b]{0.4\textwidth}
        \centering
        \includegraphics[width=\textwidth]{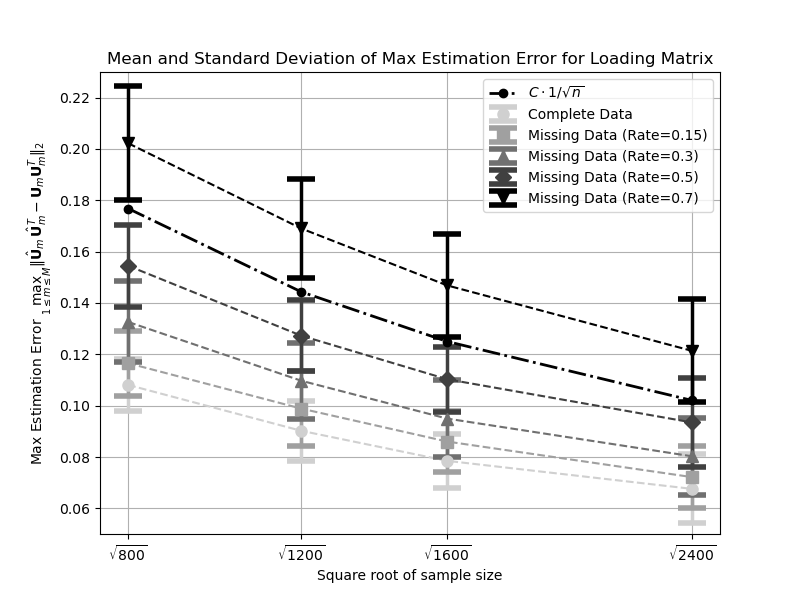}
        \caption{\small Estimation Error for Loading Matrices by Sample Size}
        \label{fig:max_loading_error_signal}
    \end{subfigure}
    \hfill
    \begin{subfigure}[b]{0.4\textwidth}
        \centering
        \includegraphics[width=\textwidth]{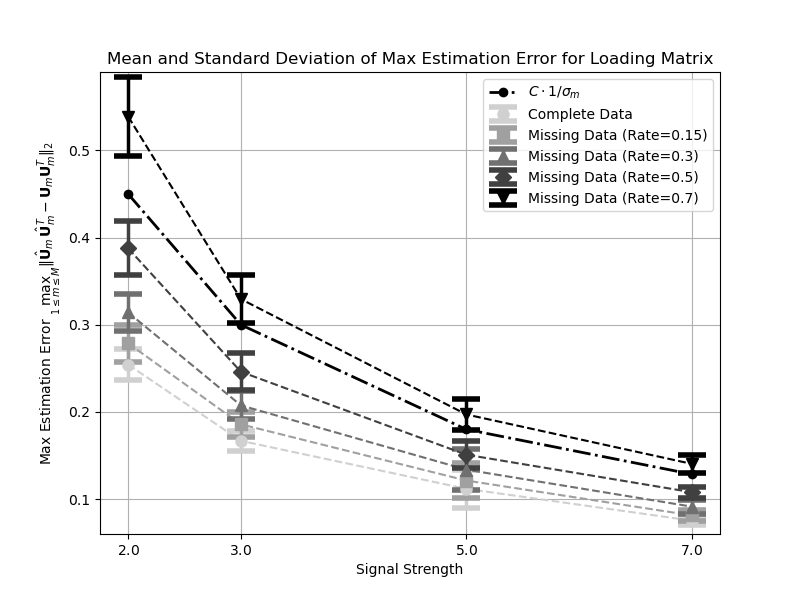}
        \caption{\small Estimation Error for Loading Matrices by Signal}
        \label{fig:max_loading_error_sample}
    \end{subfigure}
    \caption{Figures (a) and (b) compare estimation errors of loading matrices (i.e. $\max_{1 \leq m \leq M} \|\hat{\bU}_m \hat{\bU}_m^\top - \bU_m \bU_m^\top \|_2$) under different settings. (a) keeps signal strength $\sigma_m= 5.0$ while increasing sample size, whereas (b) keeps sample size $n= 800$ while increasing signal strength.}
    \label{fig:max_loading_estimation_errors}
\end{figure}

\noindent \textbf{Impact of Tensor Dimensions:} Figure \ref{fig:tucker order change} demonstrates an approximately linear relationship between the estimation error for $\calB$ and the square root of the degree of freedom (DF) as tensor order and dimensions increase, supporting our theoretical analysis:
\begin{equation}
\|\hat \calB^{\rm tucker}  - \calB \|_{\rm F}^2 \asymp O_{\PP}\left( \sqrt{\frac{\text{DF}(\calB)}{n_0}} \right).
\end{equation}

\begin{figure}[ht!]
    \centering
    \includegraphics[width=0.4\textwidth]{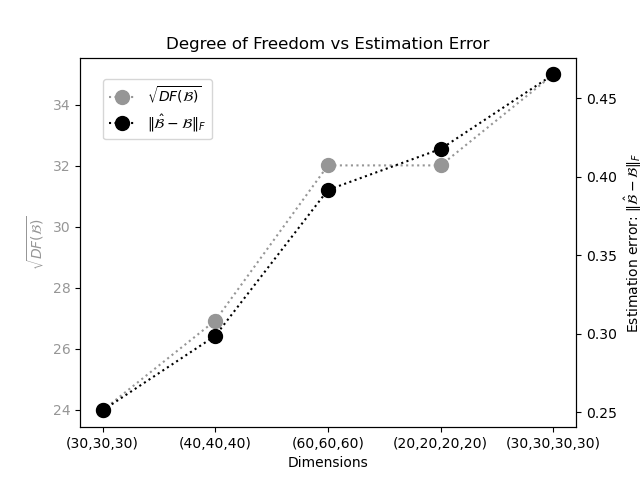}
    \caption{\small The plot compares square root of the degree of freedom (i.e. $\sqrt{r + \sum_{m=1}^{M} d_m r_m}$) with the estimation error $\|\hat{\calB} - \calB\|_{\rm F}$. Throughout the settings, keep $n=1200$ and signal strength $\sigma_m=1.5$.}
    \label{fig:tucker order change}
\end{figure}

\noindent \textbf{Comparative Performance}
While not shown in the figures, our simulation study also compared the proposed Tucker low rank Tensor LDA-MD (hereafter referred to as Tucker-LDA-MD) with other methods such as CATCH and the generalized sample discriminant tensor. The results, detailed in Tables \ref{tab:rate fix signal} and \ref{tab:rate fix sample size}, demonstrate that Tucker generally outperforms these methods, especially in scenarios with missing data. The comparison is based on both the estimation error of $\calB$ and the misclassification rates $R_{\btheta}(\hat\delta_{\rm sample})$ and $R_{\btheta}(\hat\delta_{\rm tucker})$.

\begin{table}[ht!]
\centering
\small
\makebox[\linewidth]{%
    \hspace*{0.0cm}\resizebox{\textwidth+0.0cm}{!}{%
\begin{tabular}{>{\centering\arraybackslash}p{3.5cm}*{4}{>{\centering\arraybackslash}p{1.5cm}}*{4}{>{\centering\arraybackslash}p{1.6cm}}}
\toprule
\multirow{2}{2.8cm}{\centering Method \\ (Missing Rate)} & \multicolumn{4}{c}{$R_{\btheta}(\hat\delta_{\rm sample})$} & \multicolumn{4}{c}{$R_{\btheta}(\hat\delta_{\rm tucker})$} \\
\cmidrule(lr){2-5} \cmidrule(lr){6-9}
& \small $n = 800$ & \small $n = 1200$ & \small $n = 1600$ & \small $n = 2400$ & \small $n = 800$ & \small $n = 1200$ & \small $n = 1600$ & \small $n = 2400$ \\
\midrule
{\small Tucker-LDA-MD(0)} & 0.148\textsubscript{(0.04)} & 0.110\textsubscript{(0.03)} & 0.082\textsubscript{(0.03)} & 0.057\textsubscript{(0.03)} & 0.002\textsubscript{(0.00)} & 0.001\textsubscript{(0.00)} & 0.001\textsubscript{(0.00)} & 0.001\textsubscript{(0.00)} \\
{\small CATCH(0)} & / & / & / & / & 0.283\textsubscript{(0.04)} & 0.216\textsubscript{(0.05)} & 0.163\textsubscript{(0.04)} & 0.102\textsubscript{(0.04)} \\
\midrule
{\small Tucker-LDA-MD(.15)} & 0.166\textsubscript{(0.04)} & 0.126\textsubscript{(0.03)} & 0.097\textsubscript{(0.03)} & 0.070\textsubscript{(0.03)} & 0.002\textsubscript{(0.00)} & 0.001\textsubscript{(0.00)} & 0.001\textsubscript{(0.00)} & 0.001\textsubscript{(0.00)} \\
{\small Tucker-LDA-MD(.3)} & 0.190\textsubscript{(0.04)} & 0.145\textsubscript{(0.03)} & 0.116\textsubscript{(0.03)} & 0.080\textsubscript{(0.03)} & 0.002\textsubscript{(0.00)} & 0.001\textsubscript{(0.00)} & 0.001\textsubscript{(0.00)} & 0.001\textsubscript{(0.00)} \\
{\small Tucker-LDA-MD(.5)} & 0.219\textsubscript{(0.04)} & 0.182\textsubscript{(0.04)} & 0.149\textsubscript{(0.03)} & 0.114\textsubscript{(0.03)} & 0.003\textsubscript{(0.00)} & 0.002\textsubscript{(0.00)} & 0.002\textsubscript{(0.00)} & 0.001\textsubscript{(0.00)} \\
{\small Tucker-LDA-MD(.7)} & 0.275\textsubscript{(0.03)} & 0.235\textsubscript{(0.04)} & 0.205\textsubscript{(0.04)} & 0.162\textsubscript{(0.04)} & 0.015\textsubscript{(0.01)} & 0.004\textsubscript{(0.00)} & 0.002\textsubscript{(0.00)} & 0.002\textsubscript{(0.00)} \\
\bottomrule
\end{tabular}%
}%
}
\caption{\small Misclassification errors for the sample $\calB$ and estimators obtained from Tucker-LDA-MD and CATCH, respectively. The signal strength is fixed at $\sigma_m=1.5$, while sample size $n$ and missing rate $\epsilon$ are varied.}
\label{tab:rate fix signal}
\end{table}

\begin{table}[ht!]
\centering
\small
\makebox[\linewidth]{%
    \hspace*{0.0cm}\resizebox{\textwidth+0.0cm}{!}{%
\begin{tabular}{>{\centering\arraybackslash}p{3.5cm}*{4}{>{\centering\arraybackslash}p{1.5cm}}*{4}{>{\centering\arraybackslash}p{1.6cm}}}
\toprule
\multirow{2}{2.8cm}{\centering Method \\ (Missing Rate)} & \multicolumn{4}{c}{$R_{\btheta}(\hat\delta_{\rm sample})$} & \multicolumn{4}{c}{$R_{\btheta}(\hat\delta_{\rm tucker})$} \\
\cmidrule(lr){2-5} \cmidrule(lr){6-9}
& \small $\sigma_m = 1.0$ & \small $\sigma_m = 1.5$ & \small $\sigma_m = 2.0$ & \small $\sigma_m = 3.0$ & \small $\sigma_m = 1.0$ & \small $\sigma_m = 1.5$ & \small $\sigma_m = 2.0$ & \small $\sigma_m = 3.0$ \\
\midrule
{\small Tucker-LDA-MD(0)} & 0.279\textsubscript{(0.11)} & 0.204\textsubscript{(0.03)} & 0.087\textsubscript{(0.04)} & 0.006\textsubscript{(0.01)} & 0.042\textsubscript{(0.04)} & 0.005\textsubscript{(0.00)} & 0.000\textsubscript{(0.00)} & 0.000\textsubscript{(0.00)} \\
{\small CATCH(0)} & / & / & / & / & 0.320\textsubscript{(0.11)} & 0.318\textsubscript{(0.04)} & 0.206\textsubscript{(0.04)} & 0.036\textsubscript{(0.03)} \\
\midrule
{\small Tucker-LDA-MD(.15)} & 0.291\textsubscript{(0.10)} & 0.222\textsubscript{(0.03)} & 0.100\textsubscript{(0.03)} & 0.008\textsubscript{(0.01)} & 0.055\textsubscript{(0.05)} & 0.006\textsubscript{(0.00)} & 0.000\textsubscript{(0.00)} & 0.000\textsubscript{(0.00)} \\
{\small Tucker-LDA-MD(.3)} & 0.303\textsubscript{(0.10)} & 0.243\textsubscript{(0.03)} & 0.121\textsubscript{(0.04)} & 0.012\textsubscript{(0.01)} & 0.078\textsubscript{(0.07)} & 0.007\textsubscript{(0.00)} & 0.001\textsubscript{(0.00)} & 0.000\textsubscript{(0.00)} \\
{\small Tucker-LDA-MD(.5)} & 0.331\textsubscript{(0.09)} & 0.275\textsubscript{(0.03)} & 0.153\textsubscript{(0.03)} & 0.023\textsubscript{(0.02)} & 0.126\textsubscript{(0.10)} & 0.012\textsubscript{(0.01)} & 0.000\textsubscript{(0.00)} & 0.000\textsubscript{(0.00)} \\
{\small Tucker-LDA-MD(.7)} & 0.368\textsubscript{(0.08)} & 0.318\textsubscript{(0.03)} & 0.206\textsubscript{(0.03)} & 0.054\textsubscript{(0.03)} & 0.217\textsubscript{(0.16)} & 0.052\textsubscript{(0.03)} & 0.001\textsubscript{(0.00)} & 0.000\textsubscript{(0.00)} \\
\bottomrule
\end{tabular}%
}%
}
\caption{\small Misclassification errors for the sample $\calB$ and estimators obtained from Tucker-LDA-MD and CATCH, respectively. The sample size is fixed at $n=800$, while signal strength $\sigma_m$ and missing rate $\epsilon$ are varied.}
\label{tab:rate fix sample size}
\end{table}

\subsection{Tensor Classification for Non-Gaussian Tensor Predictor}

We extend our simulation study to non-Gaussian tensor predictors to evaluate Tensor LDA-MD's robustness under more general scenarios. Samples are generated from tensor $t$-distribution $\cT t(\cM_k, \Sigma, \nu)$ with $\nu = 5$ and tensor Laplace distribution $\cT \cL(\cM_k, \Sigma, \lambda)$ with $\lambda = 2$. Parameters match the Gaussian case: dimensions $d_1 = d_2 = d_3 = 40$, low-rank values $r_1 = r_2 = r_3 = 5$, scaling factor $c = 1$, signal strength $\sigma_m = 1.5$, and varying sample sizes $n_1 = n_2 = 400, 600, 800, 1200$.

Tables \ref{tab:t distribution} and \ref{tab:laplace distribution} show results for $t$ and Laplace distributions. Our method (referred as Tucker) demonstrates robust performance across these non-Gaussian distributions, with trends similar to the Gaussian case. The proposed classifier ($\hat\delta_{\rm tucker}$) consistently outperforms the generalized sample discriminant tensor based classifier ($\hat\delta_{\rm sample}$), especially with higher missing rates and smaller sample sizes. The $t$-distribution leads to slightly higher estimation errors compared to the Gaussian case, particularly for smaller sample sizes, while Laplace distribution performance is generally closer to the Gaussian case. Convergence rates appear similar across distributions, suggesting potential extension of our theoretical results to non-Gaussian scenarios. Additionally, CATCH performs worse than $\hat\delta_{\rm sample}$ in the complete data setting, which is expected since the discrinimant tensor does not involve group sparsity.

\begin{table}[ht!]
\centering
\small
\makebox[\linewidth]{%
    \hspace*{0.0cm}\resizebox{\textwidth+0.0cm}{!}{%
\begin{tabular}{>{\centering\arraybackslash}p{3.5cm}*{4}{>{\centering\arraybackslash}p{1.5cm}}*{4}{>{\centering\arraybackslash}p{1.6cm}}}
\toprule
\multirow{2}{2.8cm}{\centering Method \\ (Missing Rate)} & \multicolumn{4}{c}{$R_{\btheta}(\hat\delta_{\rm sample})$} & \multicolumn{4}{c}{$R_{\btheta}(\hat\delta_{\rm tucker})$} \\
\cmidrule(lr){2-5} \cmidrule(lr){6-9}
& \small $n = 800$ & \small $n = 1200$ & \small $n = 1600$ & \small $n = 2400$ & \small $n = 800$ & \small $n = 1200$ & \small $n = 1600$ & \small $n = 2400$ \\
\midrule
{\small Tucker-LDA-MD(0)} & 0.292\textsubscript{(0.05)} & 0.245\textsubscript{(0.04)} & 0.218\textsubscript{(0.05)} & 0.182\textsubscript{(0.04)} & 0.032\textsubscript{(0.01)} & 0.028\textsubscript{(0.01)} & 0.026\textsubscript{(0.01)} & 0.023\textsubscript{(0.01)} \\
{\small CATCH(0)} & / & / & / & / & 0.374\textsubscript{(0.05)} & 0.325\textsubscript{(0.04)} & 0.279\textsubscript{(0.04)} & 0.218\textsubscript{(0.04)} \\
\midrule
{\small Tucker-LDA-MD(.15)} & 0.315\textsubscript{(0.04)} & 0.262\textsubscript{(0.04)} & 0.227\textsubscript{(0.04)} & 0.194\textsubscript{(0.03)} & 0.032\textsubscript{(0.01)} & 0.027\textsubscript{(0.01)} & 0.025\textsubscript{(0.01)} & 0.024\textsubscript{(0.01)} \\
{\small Tucker-LDA-MD(.3)} & 0.332\textsubscript{(0.04)} & 0.279\textsubscript{(0.04)} & 0.256\textsubscript{(0.03)} & 0.210\textsubscript{(0.02)} & 0.038\textsubscript{(0.01)} & 0.028\textsubscript{(0.01)} & 0.028\textsubscript{(0.01)} & 0.025\textsubscript{(0.01)} \\
{\small Tucker-LDA-MD(.5)} & 0.367\textsubscript{(0.04)} & 0.315\textsubscript{(0.04)} & 0.278\textsubscript{(0.03)} & 0.241\textsubscript{(0.02)} & 0.070\textsubscript{(0.03)} & 0.036\textsubscript{(0.01)} & 0.028\textsubscript{(0.01)} & 0.026\textsubscript{(0.01)} \\
{\small Tucker-LDA-MD(.7)} & 0.407\textsubscript{(0.04)} & 0.360\textsubscript{(0.03)} & 0.323\textsubscript{(0.03)} & 0.284\textsubscript{(0.03)} & 0.194\textsubscript{(0.01)} & 0.088\textsubscript{(0.01)} & 0.049\textsubscript{(0.01)} & 0.029\textsubscript{(0.01)} \\
\bottomrule
\end{tabular}%
}%
}
\caption{\small The table shows results where the tensor-covariate is drawn from Tensor $t$-distribution with $\nu=5$. It compares misclassification rate under different settings, where the signal strength of $\calB$ is kept fixed at 1.5 while varying sample size.}
\label{tab:t distribution}
\end{table}

\begin{table}[ht!]
\centering
\makebox[\linewidth]{%
    \hspace*{0.0cm}\resizebox{\textwidth+0.0cm}{!}{%
\begin{tabular}{>{\centering\arraybackslash}p{3.5cm}*{4}{>{\centering\arraybackslash}p{1.6cm}}*{4}{>{\centering\arraybackslash}p{1.6cm}}}
\toprule
\multirow{2}{2.8cm}{\centering Method \\ (Missing Rate)} & \multicolumn{4}{c}{$R_{\btheta}(\hat\delta_{\rm sample})$} & \multicolumn{4}{c}{$R_{\btheta}(\hat\delta_{\rm tucker})$} \\
\cmidrule(lr){2-5} \cmidrule(lr){6-9}
& \small $n = 800$ & \small $n = 1200$ & \small $n = 1600$ & \small $n = 2400$ & \small $n = 800$ & \small $n = 1200$ & \small $n = 1600$ & \small $n = 2400$ \\
\midrule
{\small Tucker-LDA-MD(0)} & 0.299\textsubscript{(0.04)} & 0.256\textsubscript{(0.03)} & 0.233\textsubscript{(0.03)} & 0.194\textsubscript{(0.03)} & 0.048\textsubscript{(0.02)} & 0.042\textsubscript{(0.01)} & 0.040\textsubscript{(0.01)} & 0.037\textsubscript{(0.01)} \\
{\small CATCH(0)} & / & / & / & / & 0.370\textsubscript{(0.04)} & 0.324\textsubscript{(0.03)} & 0.287\textsubscript{(0.03)} & 0.224\textsubscript{(0.01)} \\
\midrule
{\small Tucker-LDA-MD(.15)} & 0.302\textsubscript{(0.04)} & 0.265\textsubscript{(0.03)} & 0.241\textsubscript{(0.03)} & 0.207\textsubscript{(0.02)} & 0.056\textsubscript{(0.03)} & 0.043\textsubscript{(0.01)} & 0.040\textsubscript{(0.01)} & 0.039\textsubscript{(0.01)} \\
{\small Tucker-LDA-MD(.3)} & 0.319\textsubscript{(0.04)} & 0.286\textsubscript{(0.03)} & 0.257\textsubscript{(0.03)} & 0.222\textsubscript{(0.02)} & 0.067\textsubscript{(0.04)} & 0.048\textsubscript{(0.01)} & 0.046\textsubscript{(0.01)} & 0.040\textsubscript{(0.01)} \\
{\small Tucker-LDA-MD(.5)} & 0.359\textsubscript{(0.04)} & 0.316\textsubscript{(0.04)} & 0.287\textsubscript{(0.03)} & 0.249\textsubscript{(0.03)} & 0.109\textsubscript{(0.04)} & 0.061\textsubscript{(0.02)} & 0.046\textsubscript{(0.01)} & 0.046\textsubscript{(0.01)} \\
{\small Tucker-LDA-MD(.7)} & 0.395\textsubscript{(0.04)} & 0.365\textsubscript{(0.03)} & 0.331\textsubscript{(0.03)} & 0.293\textsubscript{(0.02)} & 0.237\textsubscript{(0.01)} & 0.138\textsubscript{(0.01)} & 0.085\textsubscript{(0.01)} & 0.052\textsubscript{(0.01)} \\
\bottomrule
\end{tabular}%
}%
}
\caption{\small The table shows the results where the tensor-covariate is drawn from Tensor Laplace distribution with exponential rate $\lambda=2$. It compares misclassification rate under different settings, where the signal strength of $\calB$ is kept fixed at 1.5 while varying sample size.}
\label{tab:laplace distribution}
\end{table}

\section{Real Data Analysis}
\label{sec:appl}
We demonstrate the use of the proposed Tensor LDA-MD method on a subset of the D\&D dataset, consisting of 338 protein structures. Each protein is represented as a graph, where nodes correspond to amino acids, and two nodes are connected by an edge if they are less than 6 Angstroms apart. The prediction task is to classify the protein structures as either enzymes (labeled 1) or non-enzymes (labeled 0). In this section, we apply Tucker-LDA-MD for this task and compare its classification accuracy with Tensor LDA using sample $\calB$, CATCH \citep{pan2019covariate} and DATE-D \citep{wang2024parsimonious}. Before applying these methods, we employ a dual-feature extraction approach \citep{wen2024tensorview}, combining spectral graph analysis of the adjacency matrix to capture topological characteristics with graph convolution operations to extract local structural information. The resulting feature representations are then concatenated to construct a high-dimensional learned feature tensor for each protein structure, with dimensions $64 \times 32 \times 32$. 

In each replicate, we randomly sample 68 observations as the testing set and use the rest 270 as the training set. The low-rank parameters $[r_1, r_2, r_3]$ are empirically determined by initially setting them to $[r+ \lfloor r/2 \rfloor, r, r]$ and then optimizing $r$ through leave-10-out cross-validation on the training set. We perform 100 replications and report the averaged classification accuracy below.

\begin{table}[htbp!]
    \centering
    \begin{tabular}{lcccc}
        \hline
        \textbf{Method} & \textbf{Sample $\calB$} & \textbf{Tucker} & \textbf{CATCH} & \textbf{DATE-D} \\
        \hline
        Classification Rate & $0.709_{(0.050)}$ & $0.830_{(0.043)}$ & $0.617_{(0.032)}$ & $0.672_{(0.037)}$ \\
        \hline
    \end{tabular}
    \caption{Averaged classification rates for each method with complete data (100 times replication).}
    \label{tab:5a}
\end{table}

To assess performance under incomplete data scenarios, we introduce missing data by the same missing mechanism in the simulation, at rates $\epsilon=0.1$ and $0.2$, respectively, and repeat the experiment. We present results only for the sample $\hat \calB$ and $\hat \calB^{\text{tucker}}$ methods, as CATCH and DATE-D are designed exclusively for complete data. 

\begin{table}[htbp!]
    \centering
    \begin{tabular}{lcccc}
        \hline
        \textbf{Method(Missing Rate)} & \textbf{Sample $\calB$ (0.1)} & \textbf{Tucker(0.1)} & \textbf{Sample $\calB$ (0.2)} & \textbf{Tucker(0.2)} \\
        \hline
        Classification Rate & $0.656_{(0.034)}$ & $0.723_{(0.033)}$ & $0.628_{(0.034)}$ & $0.702_{(0.029)}$ \\
        \hline
    \end{tabular}
    \caption{Averaged classification rates of each method with missing rates 0.1 and 0.2 (100 times replication).}
    \label{tab:5b}
\end{table}

It can be seen from Tabel \ref{tab:5a} and \ref{tab:5b} that the Tucker-LDA-MD method demonstrates superior performance and robustness across all scenarios, achieving the highest classification rates with complete data and consistently outperforming the Sample $\calB$ classification in the presence of missing data. Notably, even with a 20\% missing rate, Tucker's performance remains competitive with CATCH's and DATE-D's results on complete data, highlighting its effectiveness in handling incomplete information.

\section{Conclusion}
\label{sec:summ}
This paper introduces the Tensor LDA-MD method for high-dimensional tensor classification with incomplete data. Our key contributions are as follows: Methodologically, we extend mode-wise covariance estimation for tensor-variate data to the tensor-based MCR model and incorporate Tucker low-rank structure in the discriminant tensor estimation, using an iterative projection algorithm to refine the sample estimator. Theoretically, we establish the convergence rate of the discriminant tensor estimation error and minimax optimality bounds for the misclassification rate with incomplete data. We also derive large deviation bounds for generalized mode-wise covariance matrices and their inverses under the tensor-based MCR model. Our simulation studies and real data analysis demonstrate the excellent performance of our method, even with significant proportions of missing data, consistently outperforming existing approaches. This showcases its robustness and practical utility in real-world applications where incomplete data is common.


%% file: 0-appendix.tex

\section{Proofs for Main Theories} \label{append:proof:main}

\subsection{Proof of Theorem \ref{theorem: tucker}}
For clarity in the presentation, we focus on $M=3$ in the proof. The extension towards higher order tensor is straightforward. Let $C_{m,\sigma}=\tr(\otimes_{k\neq m}\Sigma_{k})/d_{-m}, \; C_\sigma = \prod_{m=1}^M C_{m,\sigma}=[\tr(\bSigma)/d]^{M-1}=[\prod_{m=1}^M \tr(\Sigma_m)/d]^{M-1}$.
Recall that from \eqref{eqn:lda-rule} and \eqref{eqn:lda-discrim-tensor},
\begin{equation*}
\calB = \bbrackets{\calM_{2}-\calM_{1}; \Sigma_1^{-1}, \Sigma_2^{-1}, \Sigma_3^{-1}} 
\quad \text{and} \quad 
\hat\calB = \bbrackets{\bar\calX^{(2)}-\bar\calX^{(1)}; \hat\Sigma_1^{-1}, \hat\Sigma_2^{-1}, \hat\Sigma_3^{-1}}  .
\end{equation*}
Recall
\begin{align*}
\hat C_{\sigma} =\frac{\prod_{m=1}^M \hat\Sigma_{m,11}}{\hat{\Var}(\cX_{1\cdots1})} ,
\end{align*}
where $\hat{\Var}(\cX_{1\cdots1})$ is the pooled generalized sample variance of the first element of $\cX_i^{(k)}$.
Define $\bD_m = \matk(\cM_{2}-\cM_{1})$. Let $\bU_{M+m}:=\bU_{m}$ for all $1\le m\le M$.

\noindent Let $\bDelta_m = \hat\Sigma_m^{-1} - C_{m,\sigma}^{-1}\Sigma_m^{-1}:=\hat\Sigma_m^{-1} - \tilde\Sigma_m^{-1}$ for $m=1,...,M-1$, and
$\bDelta_M = \hat\Sigma_M^{-1} - C_{M,\sigma}^{-1} C_{\sigma} \Sigma_m^{-1}:=\hat\Sigma_M^{-1} - \tilde\Sigma_M^{-1}$. Due to the identifiability issues associated with the tensor normal distribution, we can rescale the covariance matrices $\Sigma_{m}$ such that $C_{m,\sigma}=1$ for $1\le m\le M-1$. In this sense, we slightly abuse notation by using $\Sigma_m$ instead of $\tilde\Sigma_m$.

Then, we have the following decomposition of the error
\begin{align}\label{eqn: error decomposition}
 \cE &= \hat\calB - \calB = \bbrackets{ (\bar\calX^{(2)}-\bar\calX^{(1)}) 
- (\calM_{2}-\calM_{1});\; \Sigma_1^{-1}, \Sigma_2^{-1}, \Sigma_3^{-1}}  \notag\\
&\quad  + \bbrackets{(\bar\calX^{(2)}-\bar\calX^{(1)})- (\calM_{2}-\calM_{1});\; \bDelta_1, \Sigma_2^{-1}, \Sigma_3^{-1}} 
+ \bbrackets{(\bar\calX^{(2)}-\bar\calX^{(1)})- (\calM_{2}-\calM_{1});\; \Sigma_1^{-1}, \bDelta_2, \Sigma_3^{-1}} \notag\\
&\quad  + \bbrackets{(\bar\calX^{(2)}-\bar\calX^{(1)})- (\calM_{2}-\calM_{1});\; \Sigma_1^{-1}, \Sigma_2^{-1}, \bDelta_3} 
+ \bbrackets{(\bar\calX^{(2)}-\bar\calX^{(1)})- (\calM_{2}-\calM_{1});\; \bDelta_1, \bDelta_2, \Sigma_3^{-1}} \notag\\
&\quad  + \bbrackets{(\bar\calX^{(2)}-\bar\calX^{(1)})- (\calM_{2}-\calM_{1});\; \bDelta_1, \Sigma_2^{-1}, \bDelta_3} 
+ \bbrackets{(\bar\calX^{(2)}-\bar\calX^{(1)})- (\calM_{2}-\calM_{1});\; \Sigma_1^{-1}, \bDelta_2, \bDelta_3} \notag\\
&\quad  + \bbrackets{(\bar\calX^{(2)}-\bar\calX^{(1)}) - (\calM_{2}-\calM_{1});\; \bDelta_1, \bDelta_2, \bDelta_3} \notag\\
&\quad + 
\bbrackets{(\calM_{2}-\calM_{1});\; \bDelta_1, \Sigma_2^{-1}, \Sigma_3^{-1}} 
+ 
\bbrackets{ (\calM_{2}-\calM_{1});\; \Sigma_1^{-1}, \bDelta_2, \Sigma_3^{-1}} 
+ 
\bbrackets{ (\calM_{2}-\calM_{1});\; \Sigma_1^{-1}, \Sigma_2^{-1}, \bDelta_3} \notag\\
&\quad + 
\bbrackets{ (\calM_{2}-\calM_{1});\; \bDelta_1, \bDelta_2, \Sigma_3^{-1}} 
+ 
\bbrackets{ (\calM_{2}-\calM_{1});\; \bDelta_1, \Sigma_2^{-1}, \bDelta_3} 
+ 
\bbrackets{ (\calM_{2}-\calM_{1});\; \Sigma_1^{-1}, \bDelta_2, \bDelta_3} \notag\\
&\quad + \bbrackets{(\calM_{2}-\calM_{1});\; \bDelta_1, \bDelta_2, \bDelta_3} \notag\\
&:=\cE_0+\sum_{m=1}^3 \cE_{1,m} +\sum_{m=1}^3 \cE_{2,m} + \cE_3 +\sum_{m=1}^3 \cE_{4,m} +\sum_{m=1}^3 \cE_{5,m} + \cE_6,
\end{align}
where
\begin{align*}
\cE_0 & = \bbrackets{ (\bar\calX^{(2)}-\bar\calX^{(1)}) 
- (\calM_{2}-\calM_{1});\; \Sigma_1^{-1}, \Sigma_2^{-1}, \Sigma_3^{-1}} ,\\
\cE_{1,1} & = \bbrackets{(\bar\calX^{(2)}-\bar\calX^{(1)})- (\calM_{2}-\calM_{1});\; \bDelta_1, \Sigma_2^{-1}, \Sigma_3^{-1}} ,\\
\cE_{2,1} & = \bbrackets{(\bar\calX^{(2)}-\bar\calX^{(1)})- (\calM_{2}-\calM_{1});\; \bDelta_1, \bDelta_2, \Sigma_3^{-1}} ,\\
\cE_3 & = \bbrackets{(\bar\calX^{(2)}-\bar\calX^{(1)})- (\calM_{2}-\calM_{1});\; \bDelta_1, \bDelta_2, \bDelta_3} ,\\
\cE_{4,1} & = \bbrackets{(\calM_{2}-\calM_{1});\; \bDelta_1, \Sigma_2^{-1}, \Sigma_3^{-1}} ,\\
\cE_{5,1} & = \bbrackets{(\calM_{2}-\calM_{1});\; \bDelta_1, \bDelta_2, \Sigma_3^{-1}} ,\\
\cE_6 & = \bbrackets{(\calM_{2}-\calM_{1});\; \bDelta_1, \bDelta_2, \bDelta_3} .\\
\end{align*}

We divide the proof into 3 steps.

\noindent\textsc{Step I.} \textbf{Upper bound for initialization $\hat \bU_m^{(0)}$.}
In this first step, we consider the performance of the initialization step. 
We particularly prove that there exists sufficiently large constant $C_{\rm gap}>0$ such that whenever $\sigma_m \geq C_{\rm gap}  (\sqrt{(d_m+d_{-m})/n_0 } + \|\bD_m\|_2 \max_{k\le M} d_k/\sqrt{n_0 d}\; )$ for all $1\le m\le M$, we have
\begin{equation}
\label{eqn: initial upper bound}
\left\| \hat \bU_m^{(0)}\hat \bU_m^{(0)\top} - \bU_m \bU_m^\top \right\|_2 \le C \cdot \frac{d_m+d_{-m}}{n_0 \sigma_m^2} + C \cdot \frac{\sqrt{d_m}}{\sqrt{n_0}\sigma_m} + \frac{C\| \bD_m\|_2}{\sigma_m} \max_{1\le k\le M} \sqrt{\frac{d_k}{n_0 d_{-k}}},
\end{equation}
with probability at least $1-n_0^{-c_2}-\sum_{m=1}^M \exp(-c_2 d_m)$.

Define $\bB_m = \matk(\cB)$ and $\hat\bB_m = \matk(\hat\cB)$, for $m=1,2,3.$ By Tucker low-rank structure of $\cB$ in \eqref{eqn:lda-tucker}, $\bB_m$ is rank-$r_m$ with SVD $\bB_m = \bU_m\bLambda_m\bV_m^\top$, where $\bU_m \in \OO_{d_m, r_m}, \; \bV_m \in \OO_{d_{-m}, r_m}$ are orthonormal matrices. 
By the improved perturbation theory Lemma \ref{lm-pertubation}, as $\sigma_{m}=\sigma_{r_m}(\bB_m)$, we have for all $m=1,2,3$,
\begin{align}\label{eq:init_bdd}
\left\| \hat \bU_m^{(0)}\hat \bU_m^{(0)\top} - \bU_m \bU_m^\top \right\|_2 \le \frac{2\| \hat\bB_m -\bB_m \|_2^2}{\sigma_m^2}    + \frac{ 2\| (\hat\bB_m -\bB_m) \bV_m \|_2 }{\sigma_m}  .
\end{align}

Consider the decomposition of the error term in \eqref{eqn: error decomposition}.
The first term $\cE_0$ is a Gaussian tensor. Specifically, from the assumption on the tensor-variate $\cX$, we have $\cE_0 \sim \cT\cN(0; \bSigma^{-1}\check\bSigma\bSigma^{-1})$, where $\bSigma^{-1}= [\Sigma_m^{-1}]_{m=1}^M$ and $\check \Sigma_{i_1,...,i_M}= \frac{1}{n_{0,i_1,...,i_M}} \Sigma_{i_1,...,i_M}$. As $\| \otimes_{m=1}^M \Sigma_m\|_2 \le C_0$, by Lemma \ref{lemma:Gaussian matrix}, in an event $\Omega_{11}$ with probability at least $1-\sum_{m=1}^M \exp(-c_1 d_m)$, for all $m=1,2,3$, we have
\begin{align}
\left\|\matk(\cE_0) \right\|_2&\le C \frac{\sqrt{d_m}+\sqrt{d_{-m}}}{\sqrt{n_0}},    \label{eq:init_E0a}\\
\left\|\matk(\cE_0) \bV_m \right\|_2& \le C \frac{\sqrt{d_m}}{\sqrt{n_0}} .  \label{eq:init_E0b}
\end{align}
For the non-Gaussian terms $\cE_{1,m},\cE_{2,m},\cE_3$, the operator norm of their matricization is of a smaller order compared to that of a Gaussian tensor. By Lemma \ref{lemma:precision matrix} and as $\| \otimes_{m=1}^M \Sigma_m^{-1}\|_2 \le C_0$, in an event $\Omega_{12}$ with probability at least $1-n_0^{-c_1}-\sum_{m=1}^M \exp(-c_1 d_m)$,
\begin{align*}
\bDelta_m &\le   C \sqrt{ \frac{d_m}{n_0 d_{-m}}  }.
\end{align*}
For the bound of $\bDelta_3$, we need to use Lemma \ref{lemma:precision matrix}(ii) with $t_1\asymp t_2\asymp \log(n_0)$ to derive $|\hat C_{\sigma} - C_{\sigma}|=o(1)$ in the event $\Omega_{12}$.
Since $n_0 d_{-m}\gtrsim d_m$ for all $m$, in the event $\Omega_{12}$, $\bDelta_m\lesssim 1$.
Thus, by Lemma \ref{lemma:Gaussian matrix}, in the event $\Omega_{11}\cap \Omega_{12}$ with probability at least $1-n_0^{-c_2}-\sum_{m=1}^M \exp(-c_2 d_m)$, we have
\begin{align}
\left\| \matk(\cE_{1,1}) \right\|_2    & \le \left\| \matk\left( \bbrackets{(\bar\cX^{(2)}-\bar\cX^{(1)})- (\cM_{2}-\cM_{1});\; \bI_{d_1}, \Sigma_2^{-1}, \Sigma_3^{-1}} \right) \right\|_2 \cdot \|\bDelta_1\|_2 \notag \\
&\le C_1 \frac{\sqrt{d_m}+\sqrt{d_{-m}}}{\sqrt{n_0}} \cdot \sqrt{ \frac{d_1}{n_0 d_{-1}}  }  \notag \\
&\le C \frac{\sqrt{d_m}+\sqrt{d_{-m}}}{\sqrt{n_0}} , \label{eq:init_E1a}\\
\left\| \matk(\cE_{1,1}) \bV_m \right\|_2    &  \le C \frac{\sqrt{d_m}}{\sqrt{n_0}} , \label{eq:init_E1b}
\end{align}
for all $m=1,2,3$. Similarly, in the event $\Omega_{11}\cap \Omega_{12}$, $\| \matk(\cE_{1,k}) \|_2$ and $\| \matk(\cE_{1,k}) \bV_m \|_2$, $k=2,3$, have the same upper bound in \eqref{eq:init_E1a} and in \eqref{eq:init_E1b}, respectively.
Moreover, in the event $\Omega_{11}\cap \Omega_{12}$, we can also show that $\| \matk(\cE_{2,k}) \|_2$, $\| \matk(\cE_3) \|_2$ have the same upper bound in \eqref{eq:init_E1a}, and $\| \matk(\cE_{2,k}) \bV_m\|_2$, $\| \matk(\cE_3)\bV_m \|_2$ have the same upper bound in \eqref{eq:init_E1b}. 

Recall $\bD_m = \matk(\cM_{2}-\cM_{1})$, by lemma \ref{lemma:precision matrix}, in the event $\Omega_{12}$,
\begin{align}
\left\| \matk(\cE_{4,k} ) \right\|_2    &\le C  \|\bD_m\|_2 \sqrt{\frac{d_k}{n_0  d_{-k}}}, \qquad k=1,2,3, \label{eq:init_E2a}\\
\left\| \matk(\cE_{4,k} )\bV_m \right\|_2    &\le C  \|\bD_m\|_2 \sqrt{\frac{d_k}{n_0  d_{-k}}}, \qquad k=1,2,3 .  \label{eq:init_E2b}
\end{align}
Again, in the event $\Omega_{12}$, $\| \matk(\cE_{5,k}) \|_2$, $\| \matk(\cE_6) \|_2$, $\| \matk(\cE_{5,k}) \bV_m\|_2$, $\| \matk(\cE_6)\bV_m \|_2$ have the same upper bound in \eqref{eq:init_E2a}.

Substituting the bounds of error $\cE$, i.e. \eqref{eq:init_E0a}--\eqref{eq:init_E2b}, into \eqref{eq:init_bdd}, in the event $\Omega_1=\Omega_{11}\cap \Omega_{12}$ with probability at least $1-n_0^{-c_2}-\sum_{m=1}^M \exp(-c_2 d_m)$, we have
\begin{align}\label{eq:init_bdd2}
\left\| \hat \bU_m^{(0)}\hat \bU_m^{(0)\top} - \bU_m \bU_m^\top \right\|_2 &\le \frac{2\| \hat\bB_m -\bB_m \|_2^2}{\sigma_m^2}    + \frac{ 2\| (\hat\bB_m -\bB_m) \bV_m \|_2 }{\sigma_m}  \notag\\
&\le C \cdot \frac{d_m+d_{-m}}{n_0\sigma_m^2} + C \cdot \frac{\sqrt{d_m}}{\sqrt{n_0}\sigma_m} + \frac{C\| \bD_m\|_2}{\sigma_m} \max_{1\le k\le M} \sqrt{\frac{d_k}{n_0 d_{-k}}}  .
\end{align}

\begin{remark}
When $\Sigma_m=c_m \bI_{d_m}$ for all $1\le m\le M$, employing similar arguments in Theorem 3 and Lemma 4 in \cite{cai2018rate}, we can show in the event $\Omega_1$,
\begin{align}\label{eq:init_bdd2n}
\left\| \hat \bU_m^{(0)}\hat \bU_m^{(0)\top} - \bU_m \bU_m^\top \right\|_2 
&\le C \cdot \frac{\sqrt{d}}{\sqrt{n_0}\sigma_m^2} + C \cdot \frac{\sqrt{d_m}}{\sqrt{n_0}\sigma_m} + \frac{C\| \bD_m\|_2}{\sigma_m} \max_{1\le k\le M} \sqrt{\frac{d_k}{n_0 d_{-k}}}  .
\end{align}    
\end{remark}

\medskip
\noindent\textsc{Step II.}
\textbf{Error contraction after each iteration $t$.} 
After initialization using $\hat \bU_{m}^{(0)}$, the algorithm iteratively generates estimates $\hat \bU_m^{(t)}$ for each iteration $t\in[T]$. 
We define $L_m^{(t)}$ be the estimation error of $\hat \bU_m^{(t)}$ after $t$ iterations,
\begin{align*}
    L_m^{(t)} = \left\|\hat \bU_m^{(t)}\hat \bU_m^{(t)\top} - \bU_m \bU_m^\top\right\|_2, \quad L^{(t)} = \underset{m\in \{1,2,3\}}{\max} L_m^{(t)}. 
\end{align*}
Define the ideal version of the rate as
\begin{align*}
R_m^{\ideal}=\frac{\sqrt{d_m+r_{-m}}}{\sqrt{n_0}\sigma_m}+\frac{\|\bD_m\|_2\max_k d_k}{\sqrt{n_0 d} \sigma_m},    \quad R^{\ideal} = \underset{m\in \{1,2,3\}}{\max} R_m^{\ideal}  .
\end{align*}
Starting with a good initialization from \textsc{Step I}, we aim to show the contraction property of each iteration -- the estimation error reduces {\em geometrically} as iteration grows.
For notational simplicity, we sometimes give explicit expressions only in the case of $m=1$ and $M=3$. 

Recall that, since $n_0 d_{-m}\gtrsim d_m$ for all $m$, in the event $\Omega_{12}$, $\bDelta_m\lesssim 1$. By Lemmas \ref{lemma:precision matrix} and \ref{lemma:Guassian tensor projection}, in an event $\Omega_{21}$ with probability at least $1-n_0^{-c_3}-\sum_{m=1}^M \exp(-c_3 d_m)$, we have
\begin{align*}
\mathop{\max}\limits_{\bV_2 \in \mathbb{R}^{d_2 \times r_2} \atop
\bV_3 \in \mathbb{R}^{d_3 \times r_3}} \frac{\|\mat1(\cE_0) \cdot (\bV_2 \otimes \bV_3)\|_2}{\|\bV_2\|_2 \cdot \| \bV_3\|_2} &\le C \left(\sqrt{\frac{d_1+r_{-1}}{n_0}}+ \sqrt{\frac{d_2r_2+d_3 r_3}{n_0}} \right),\\
\mathop{\max}\limits_{\bV_2 \in \mathbb{R}^{d_2 \times r_2} \atop
\bV_3 \in \mathbb{R}^{d_3 \times r_3}} \frac{\|\mat1(\cE_{1,k}) \cdot (\bV_2 \otimes \bV_3)\|_2}{\|\bV_2\|_2 \cdot \| \bV_3\|_2 } &\le C_1 \|\Delta_k\|_2 \left(\sqrt{\frac{d_1+r_{-1}}{n_0}}+ \sqrt{\frac{d_2r_2+d_3r_3}{n_0}} \right),\\
&\le C \left(\sqrt{\frac{d_1+r_{-1}}{n_0}}+ \sqrt{\frac{d_2r_2+d_3r_3}{n_0}} \right), \\
\mathop{\max}\limits_{\bV_2 \in \mathbb{R}^{d_2 \times r_2} \atop
\bV_3 \in \mathbb{R}^{d_3 \times r_3}} \frac{\|\mat1(\cE_{2,k}) \cdot (\bV_2 \otimes \bV_3)\|_2}{\|\bV_2\|_2 \cdot \| \bV_3\|_2 } &\le C \left(\sqrt{\frac{d_1+r_{-1}}{n_0}}+ \sqrt{\frac{d_2r_2+d_3r_3}{n_0}} \right),\\
\mathop{\max}\limits_{\bV_2 \in \mathbb{R}^{d_2 \times r_2} \atop
\bV_3 \in \mathbb{R}^{d_3 \times r_3}} \frac{\|\mat1(\cE_{3}) \cdot (\bV_2 \otimes \bV_3)\|_2}{\|\bV_2\|_2 \cdot \| \bV_3\|_2 } &\le C \left(\sqrt{\frac{d_1+r_{-1}}{n_0}}+ \sqrt{\frac{d_2r_2+d_3r_3}{n_0}} \right),\\
\mathop{\max}\limits_{\bV_2 \in \mathbb{R}^{d_2 \times r_2} \atop
\bV_3 \in \mathbb{R}^{d_3 \times r_3}} \frac{\|\mat1(\cE_{4,k}) \cdot (\bV_2 \otimes \bV_3)\|_2}{\|\bV_2\|_2 \cdot \| \bV_3\|_2} &\le C \|\Delta_k\|_2 \|\bD_1\|_2 \le C \|\bD_1\|_2\sqrt{\frac{d_k}{n_0 d_{-k}}}, \\
\mathop{\max}\limits_{\bV_2 \in \mathbb{R}^{d_2 \times r_2} \atop
\bV_3 \in \mathbb{R}^{d_3 \times r_3}} \frac{\|\mat1(\cE_{5,k}) \cdot (\bV_2 \otimes \bV_3)\|_2}{\|\bV_2\|_2 \cdot \| \bV_3\|_2} &\le C \|\bD_1\|_2\max_k\sqrt{\frac{d_k}{n_0 d_{-k}}}, \\
\mathop{\max}\limits_{\bV_2 \in \mathbb{R}^{d_2 \times r_2} \atop
\bV_3 \in \mathbb{R}^{d_3 \times r_3}} \frac{\|\mat1(\cE_{6}) \cdot (\bV_2 \otimes \bV_3)\|_2}{\|\bV_2\|_2 \cdot \| \bV_3\|_2} &\le C \|\bD_1\|_2\max_k\sqrt{\frac{d_k}{n_0 d_{-k}}}, 
\end{align*}
for all $k=1,2,3$.
Combining the above bounds, in the event $\Omega_{21}$ with probability at least $1-n_0^{-c_3}-\sum_{m=1}^M \exp(-c_3 d_m)$, we have
\begin{align}\label{tucker assumption 1}
&\mathop{\max}\limits_{\bV_2 \in \mathbb{R}^{d_2 \times r_2} \atop
\bV_3 \in \mathbb{R}^{d_3 \times r_3}} \frac{\|\matk(\cE) \cdot (\bV_{m+1} \otimes \bV_{m+2})\|_2}{\|\bV_{m+1}\|_2 \cdot \| \bV_{m+2}\|_2} \le C \left(\sqrt{\frac{d_m+r_{-m}}{n_0}}+ \sqrt{\frac{\sum_{k\neq m}^M d_kr_k}{n_0}} \right) + C \|\bD_m\|_2\frac{\max_k d_k}{\sqrt{n_0 d}},
\end{align}
where $\bV_{4}:=\bV_{1},\bV_{5}:=\bV_{2}$, $1\le m\le 3$. Similarly, we can show that in an event $\Omega_{22}$ with probability at least $1-n_0^{-c_4}-\sum_{m=1}^M \exp(-c_4 d_m)$,
\begin{align}\label{tucker assumption 2}
&\|\matk(\cE) (\bU_{m+1} \otimes \bU_{m+2})\| \le C \sqrt{\frac{d_m+r_{-m}}{n_0}} + C\|\bD_m\|_2 \frac{\max_k d_k}{\sqrt{n_0 d}}. 
\end{align}

Define the quantities
\begin{align*}
&\hat\bB_m^{(t)} =  \matk \left( \hat\calB \times_{j=1}^{m-1} \hat\bU_j^{(t+1)\top} \times_{j=m+1}^M \hat\bU_j^{(t)\top} \right) \in \RR^{d_m \times r_{-m}}, \\
&\bB_m^{(t)} =  \matk \left( \calB \times_{j=1}^{m-1} \hat\bU_j^{(t+1)\top} \times_{j=m+1}^M \hat\bU_j^{(t)\top} \right)   \in \RR^{d_m \times r_{-m}}, \\
& \bZ_m^{(t)} =  \matk \left( \cE \times_{j=1}^{m-1} \hat\bU_j^{(t+1)\top} \times_{j=m+1}^M \hat\bU_j^{(t)\top} \right)  \in \RR^{d_m \times r_{-m}}.
\end{align*}
It is straightforward to deduce that $\hat \bU_m^{(t+1)} = \LSVD_{r_m}(\matk (\hat\bB_m^{(t)}) ).$
By definition, we have 
\begin{align}
\label{eqn: singular value}
\sigma_{r_1} \left(\bB_1^{(t)}\right) &= \sigma_{r_1} \left(\mat1(\calB) \cdot \left( \hat\bU_2^{(t)} \otimes \hat\bU_3^{(t)} \right) \right) = \sigma_{r_1}\left(\mat1(\calB) \cdot \calP_{\bU_2 \otimes \bU_3} \cdot \left( \hat\bU_2^{(t)} \otimes \hat\bU_3^{(t)} \right) \right) \notag \\
&= \sigma_{r_1} \left(\mat1(\calB) \cdot (\bU_2 \otimes \bU_3) \cdot(\bU_2 \otimes \bU_3)^\top \cdot \left( \hat\bU_2^{(t)} \otimes \hat\bU_3^{(t)} \right) \right) \notag\\
&\ge \sigma_{r_1} \left(\mat1(\calB) \cdot (\bU_2 \otimes \bU_3) \right) \cdot \sigma_{\min} \left((\bU_2 \otimes \bU_3)^\top \cdot \left( \hat\bU_2^{(t)} \otimes \hat\bU_3^{(t)} \right) \right) \notag\\
&= \sigma_{r_1}\left(\mat1(\calB)\right) \cdot \sigma_{\min}\left(\left(\bU_2^\top \hat\bU_2^{(t)}\right) \otimes \left(\bU_3^\top \hat\bU_3^{(t)}\right) \right) \notag\\ 
&\ge \sigma_{r_1}\left(\mat1(\calB) \right) \cdot \sigma_{\min} \left(\bU_2^\top \hat\bU_2^{(t)}\right) \cdot \sigma_{\min} \left(\bU_3^\top \hat\bU_3^{(t)}\right) \notag\\
&\ge \sigma_1 \cdot \left(1 - (L^{(t)})^2\right) ,
\end{align}
where the last inequality follows from the fact that $\| \hat\bU_m^{(t)}\hat \bU_m^{(t)\top} - \bU_m \bU_m^\top \|_2^2=1-\sigma_{\min}^2(\bU_m^\top \hat \bU_m^{(t)})$. If $L^{(t)}\le 1/2$, $\sigma_{r_1} \left(\bB_1^{(t)}\right) \ge \sigma_1/2$.
Note that, for general high order tensor $M \ge 3$, \eqref{eqn: singular value} can be expressed as
\begin{align*}
\sigma_{r_1} \left(\bB_1^{(t)}\right) \ge \sigma_1 \cdot \left(1 - (L^{(t)})^2\right)^{\frac{M-1}{2}}    .
\end{align*}
Meanwhile, by \eqref{tucker assumption 1} and \eqref{tucker assumption 2}, in the event $\Omega_{21}\cap \Omega_{22}$, we have
\begin{align}
\label{eqn: error bound 1}
&\left\|\bZ_1^{(t)}\right\|_2 = \left\|\mat1(\cE) \left(\hat\bU_2^{(t)} \otimes \hat\bU_3^{(t)}\right) \right\|_2 \notag\\
 \le& \left\|\mat1(\cE) \big(\calP_{\bU_2 \otimes \bU_3} \big) \left(\hat\bU_2^{(t)} \otimes \hat\bU_3^{(t)}\right) \right\|_2 + \left\|\mat1(\cE) \big(\calP_{\bU_{2\perp} \otimes \bU_3} \big) \left(\hat\bU_2^{(t)} \otimes \hat\bU_3^{(t)}\right) \right\|_2 \notag\\
&+ \left\|\mat1(\cE) \big(\calP_{\bU_2 \otimes \bU_{3\perp}} \big) \left(\hat\bU_2^{(t)} \otimes \hat\bU_3^{(t)}\right) \right\|_2 + \left\|\mat1(\cE) \big(\calP_{\bU_{2\perp} \otimes \bU_{3\perp}} \big) \left(\hat\bU_2^{(t)} \otimes \hat\bU_3^{(t)}\right) \right\|_2 \notag\\
=& \left\|\mat1(\cE) (\bU_2 \otimes \bU_3)(\bU_2 \otimes \bU_3)^\top \left(\hat\bU_2^{(t)} \otimes \hat\bU_3^{(t)}\right) \right\|_2
+ \left\|\mat1(\cE) \left((\calP_{\bU_{2\perp}} \hat\bU_2^{(t)}) \otimes (\calP_{\bU_{3}} \hat\bU_3^{(t)})\right) \right\|_2 \notag\\ 
&+ \left\|\mat1(\cE) \left((\calP_{\bU_{2}} \hat\bU_2^{(t)}) \otimes (\calP_{\bU_{3\perp}} \hat\bU_3^{(t)})\right) \right\|_2
+ \left\|\mat1(\cE) \left((\calP_{\bU_{2\perp}} \hat\bU_2^{(t)}) \otimes (\calP_{\bU_{3\perp}} \hat\bU_3^{(t)})\right) \right\|_2 \notag\\
\overset{\eqref{tucker assumption 1}}{\le}& \left\|\mat1(\cE) (\bU_2 \otimes \bU_3) \right\|_2 + C\left(\frac{\sqrt{d_1+r_{-1}} + \sqrt{d_2r_2+d_3r_3}}{\sqrt{n_0}} + \frac{\|\bD_1\|_2\max_k d_k}{\sqrt{n_0 d}}\right) \left( \left\|\calP_{\bU_{2\perp}} \hat\bU_2^{(t)}\right\|_2 \left\|\calP_{\bU_{3}} \hat\bU_3^{(t)} \right\|_2 \right.  \notag\\
&\qquad + \left. \left\|\calP_{\bU_{2\perp}} \hat\bU_2^{(t)}\right\|_2 \left\|\calP_{\bU_{3\perp}} \hat\bU_3^{(t)}\right\|_2 + \left\|\calP_{\bU_{2\perp}} \hat\bU_2^{(t)}\right\|_2 \left\|\calP_{\bU_{3\perp}} \hat\bU_3^{(t)}\right\|_2 \right) \notag\\
\le& C\left(\sqrt{\frac{d_1+r_{-1}}{n_0}} + \frac{\|\bD_1\|_2\max_k d_k}{\sqrt{n_0 d}} \right)+ C_2\left( \frac{\sqrt{d_1+r_{-1}} + \sqrt{d_2r_2+d_3r_3}}{\sqrt{n_0}} + \frac{\|\bD_1\|_2\max_k d_k}{\sqrt{n_0 d}}\right) \left(2L^{(t)} + (L^{(t)})^2\right) \notag\\
\le& C\left(\sqrt{\frac{d_1+r_{-1}}{n_0}} + \frac{\|\bD_1\|_2\max_k d_k}{\sqrt{n_0 d}} \right) + C\left( \frac{\sqrt{d_1+r_{-1}} + \sqrt{d_2r_2+d_3r_3}}{\sqrt{n_0}} + \frac{\|\bD_1\|_2\max_k d_k}{\sqrt{n_0 d}}\right) L^{(t)}.
\end{align}
Note that for general high order tensor $M \ge 3$, in the event $\Omega_{21}\cap \Omega_{22}$,
\begin{align*}
\left\|\bZ_1^{(t)}\right\|_2 \le C\left(\sqrt{\frac{d_1+r_{-1}}{n_0}} + \frac{\|\bD_1\|_2\max_k d_k}{\sqrt{n_0 d}} \right)+ C(2^{M-1}-1)\left( \frac{\sqrt{d_1+r_{-1}} + \sqrt{\sum_{k=2}^M d_kr_k}}{\sqrt{n_0}} + \frac{\|\bD_1\|_2\max_k d_k}{\sqrt{n_0 d}}\right) L^{(t)}.
\end{align*}
Since $\hat \bU_m^{(t+1)}$ and $\bU_m$ are the leading $r_m$ singular vectors of $\hat\bB_m^{(t)}$ and $\bB_m^{(t)}$ respectively, by Wedin's $\sin\Theta$ Theorem \citep{wedin1972perturbation}, in the event $\Omega_2=\Omega_{21}\cap\Omega_{22}$ with probability at least $1-n_0^{-c_5}-\sum_{m=1}^M \exp(-c_5 d_m)$, we have 
\begin{align*}
&\left\|\hat \bU_m^{(t+1)}\hat \bU_m^{(t+1)^\top} - \bU_m \bU_m^\top\right\|_2 \le \frac{2\| \bZ_m^{(t)} \|_2}{\sigma_{r_m} \left(\bB_m^{(t)}\right)} \\
 \le& \frac{C\sqrt{d_m+r_{-m}}}{\sqrt{n_0}\sigma_m} + \frac{C\|\bD_m\|_2\max_k d_k}{\sqrt{n_0 d} \sigma_m} + C\left( \frac{\sqrt{d_m+r_{-m}} + \sqrt{\sum_{k\neq m}^M d_kr_k}}{\sqrt{n_0}} + \frac{C\|\bD_m\|_2\max_k d_k}{\sqrt{n_0 d}}\right) \frac{ L^{(t)}}{\sigma_m}.
\end{align*}
Based on the signal to noise ratio condition, there exists $\rho<1$ such that
\begin{align*}
C\left( \frac{\sqrt{d_m+r_{-m}} + \sqrt{\sum_{k\neq m}^M d_kr_k}}{\sqrt{n_0}} + \frac{\|\bD_m\|_2\max_k d_k}{\sqrt{n_0 d}}\right) \frac{ 1}{\sigma_m} \le \rho<1.    
\end{align*}
It follows that in the event $\Omega_2$,
\begin{align}\label{eq:init_bdd3}
&\left\|\hat \bU_m^{(t+1)}\hat \bU_m^{(t+1)^\top} - \bU_m \bU_m^\top\right\|_2 \le  C R_m^{\ideal} + \rho L^{(t)}.
\end{align}
Putting \eqref{eq:init_bdd2} and \eqref{eq:init_bdd3} together, by induction, we find that in the event $\Omega_1 \cap \Omega_2$,
\begin{align*}
\left\|\hat \bU_m^{(t+1)}\hat \bU_m^{(t+1)^\top} - \bU_m \bU_m^\top\right\|_2 \le  C (1+\cdots+\rho^t) R_m^{\ideal} + \rho^{t+1} L^{(0)}    .
\end{align*}
Therefor, after at most $T=\lceil (C R^{\ideal})/((1-\rho)(\log\rho) L^{(0)})\rceil$ iterations, in the event $\Omega_1 \cap \Omega_2$, we have
\begin{align*}
\left\|\hat \bU_m^{(T)}\hat \bU_m^{(T)^\top} - \bU_m \bU_m^\top\right\|_2 \le  \frac{2C}{1-\rho} R^{\ideal}   .
\end{align*}

\medskip
\noindent\textsc{Step III.} \textbf{Upper bound for $\left\|\hat \cB^{\rm tucker}  - \cB\right\|_{\rm F}$.} 
In this step, we develop the upper bound for $\|\hat \cB^{\rm tucker}  - \cB\|_{\rm F}$.
Instead of working on $\hat \cB^{\rm tucker}$ and $\hat \bU_{m}^{(T)}$ directly, we take one step back and work on the evolution of $\hat \bU_{m}^{(T-1)}$ to $\hat \bU_{m}^{(T)}$.
Let $\hat \bU_{m} = \hat \bU_{m}^{(T)}$, $\hat \bU_{m,\perp}$ be the orthogonal complements of $\hat \bU_{m}$, $m=1,2,3$, and $\hat \calB^{\rm tucker} = \hat{\calB} \times_1 \mathcal{P}_{\hat \bU_{1}} \times_2 \mathcal{P}_{\hat \bU_{2}} \times_3 \mathcal{P}_{\hat \bU_{3}}$. In the previous step we have proved that in the event $\Omega_1 \cap \Omega_2$,
\begin{align*}
 \left\|\hat \bU_{m}\hat \bU_{m}^\top - \bU_{m} \bU_{m}^\top\right\|_{2} \le C_0 R^{\ideal} , \quad m=1,2,3 .
\end{align*}
Then we have the following decomposition for the estimation error 
\begin{align}
\label{eqn: tucker-decompostion}
\left\|\hat \cB^{\rm tucker}  - \cB\right\|_{\rm F}  \le & \left\|\cB - \cB \times_1 \mathcal{P}_{\hat \bU_{1}} \times_2 \mathcal{P}_{\hat \bU_{2}} \times_3 \mathcal{P}_{\hat \bU_{3}}\right\|_{\rm F} + \left\|\cE \times_1 \mathcal{P}_{\hat \bU_{1}} \times_2 \mathcal{P}_{\hat \bU_{2}} \times_3 \mathcal{P}_{\hat \bU_{3}}\right\|_{\rm F} \notag \\
=& \left\|\cB \times_1 \mathcal{P}_{\hat \bU_{1\perp}} + \cB \times_1 \mathcal{P}_{\hat \bU_{1}} \times_2 \mathcal{P}_{\hat \bU_{2\perp}} + \cB \times_1 \mathcal{P}_{\hat \bU_{1}} \times_2 \mathcal{P}_{\hat \bU_{2}} \times_3 \mathcal{P}_{\hat \bU_{3\perp}}\right\|_{\rm F} \notag \\
& + \left\|\cE \times_1 \mathcal{P}_{\hat \bU_{1}} \times_2 \mathcal{P}_{\hat \bU_{2}} \times_3 \mathcal{P}_{\hat \bU_{3}}\right\|_{\rm F} \notag \\
\le& \left\|\cB \times_1 \hat \bU_{1\perp}^\top\right\|_{\rm F} + \left\|\cB \times_2 \hat \bU_{2\perp}^\top \right\|_{\rm F} + \left\|\cB \times_3 \hat \bU_{3\perp}^\top\right\|_{\rm F} 
+ \left\|\cE \times_1 \hat \bU_{1}^\top \times_2 \hat \bU_{2}^\top \times_3 \hat \bU_{3}^\top\right\|_{\rm F} .
\end{align}
To obtain the upper bound of $\|\hat \calB^{\rm tucker}  - \calB\|_{\rm F}$, we only need to analyze the four terms above separately.

By Lemma \ref{lemma:Guassian tensor projection}, in an event with probability at least $1-\sum_{m=1}^M \exp(-c_6 d_m)$, we have 
\begin{equation*}
\left\|\cE_0 \times_1 \hat \bU_{1}^\top \times_2 \hat \bU_{2}^\top \times_3 \hat \bU_{3}^\top\right\|_{\rm F} \le \mathop{\max}\limits_{\bV_1, \bV_2, \bV_3\in \mathbb{R}^{d_m \times r_m} \atop
\|\bV_m\|_2 \le 1,\; m=1,2,3 } \left\| \cE_0 \times_1 \bV_1^\top \times_2 \bV_2^\top \times_3 \bV_3^\top \right\|_{\rm F} \le C \frac{\sqrt{r} + \sqrt{\sum_{m=1}^M d_m r_m}}{\sqrt{n_0}}.
\end{equation*}
By Lemma \ref{lemma:precision matrix} and \ref{lemma:Guassian tensor projection}, using the fact $\|AB\|_{\rm F}\le \|A\|_2 \|B\|_{\rm F}$, in an event with probability at least $1-n_0^{-c_6}-\sum_{m=1}^M \exp(-c_6 d_m)$, we have
\begin{align*}
&\left\|\cE_{1,1} \times_1 \hat \bU_{1}^\top \times_2 \hat \bU_{2}^\top \times_3 \hat \bU_{3}^\top\right\|_{\rm F} = \left\|\hat \bU_{1}^\top  \mat1(\cE_1)  (\hat \bU_{2} \otimes \hat \bU_{3})\right\|_{\rm F} \\
\le& \|\hat \bU_1 \bDelta_1\|_2 \cdot\mathop{\max}\limits_{\bV_1, \bV_2, \bV_3\in \mathbb{R}^{d_m \times r_m} \atop
\|\bV_m\|_2 \le 1,\; m=1,2,3} \left\| \bV_1^\top \mat1\left( \bbrackets{(\bar\cX^{(2)}-\bar\cX^{(1)})- (\cM_{2}-\cM_{1});\; \bI_{d_1}, \Sigma_2^{-1}, \Sigma_3^{-1}} \right) (\hat \bV_{2} \otimes \hat \bV_{3}) \right\|_{\rm F}  \\
\le&  C_3 \frac{\sqrt{r} + \sqrt{\sum_{m=1}^M d_m r_m}}{\sqrt{n_0}} \cdot \sqrt{\frac{d_1}{n_0 d_{-1}}}\\
\le&  C \frac{\sqrt{r} + \sqrt{\sum_{m=1}^M d_m r_m}}{\sqrt{n_0}} , \\
&\left\| \cE_{4,1} \times_1 \hat \bU_{1}^\top \times_2 \hat \bU_{2}^\top \times_3 \hat \bU_{3}^\top\right\|_{\rm F} \le \left\|\hat \bU_{1}^\top \bDelta_1\right\|_2 \cdot \left\|\mat1\left(\cM_{2} - \cM_{1}\right) \times_2 \left(\hat \bU_{2}^\top \Sigma_2^{(-1)}\right) \times_3 \left(\hat \bU_{3}^\top \Sigma_3^{(-1)} \right) \right\|_{\rm F} \\
\le&  C \|\bD_1 \|_{\rm F} \sqrt{\frac{d_1}{n_0 d_{-1}}} .
\end{align*}
Applying similar arguments in Step I or II, we can show that, in an event $\Omega_{3}$ with probability at least $1-n_0^{-c_6}-\sum_{m=1}^M \exp(-c_6 d_m)$,
\begin{align}\label{tucker assumption 3}
\left\| \cE \times_1 \hat \bU_{1}^\top \times_2 \hat \bU_{2}^\top \times_3 \hat \bU_{3}^\top \right\|_{\rm F} \le C \sqrt{ \frac{r+\sum_{m=1}^M d_m r_m}{n_0}  }  + C\|\bD_1 \|_{\rm F} \max_{1\le k\le M} \sqrt{\frac{d_k}{n_0 d_{-k}}} .
\end{align}

By \eqref{eqn: error decomposition}, define
\begin{align*}
\tilde \cB &= \bbrackets{\calM_{2} - \calM_{1}; \hat\Sigma_1^{-1}, \hat\Sigma_2^{-1}, \hat\Sigma_3^{-1}}=\cB+\sum_{m=1}^3 \cE_{4,m} +\sum_{m=1}^3 \cE_{5,m} + \cE_6   ,\\
\tilde \cE &= \cE_0+ \sum_{m=1}^3 \cE_{1,m} +\sum_{m=1}^3 \cE_{2,m} + \cE_3   .
\end{align*}
Then $\widetilde \cB$ has tucker rank $(r_1,r_2,r_3)$, and $\hat\cB=\tilde\cB+\tilde\cE$. 
Recall in Step II, we define
\begin{align*}
& \hat \bB_1^{(T-1)}:= \mat1 \left( \hat \cB \times_2 \hat \bU_{2}^{(T-1)\top} \times_3 \hat \bU_{3}^{(T-1)\top} \right) = \mat1 (\hat \cB)  \left(\hat \bU_{2}^{(T-1)} \otimes \hat \bU_{3}^{(T-1)}\right), \\
& \tilde \bB_1^{(T-1)}:= \mat1 \left( \tilde \cB \times_2 \hat \bU_{2}^{(T-1)\top} \times_3 \hat \bU_{3}^{(T-1)\top} \right)  ,\\
& \bW_1^{(T-1)} :=  \mat1 \left( \tilde\cE \times_2 \hat\bU_2^{(T-1)\top} \times_3 \hat\bU_3^{(T-1)\top} \right) ,
\end{align*}
and define $\bB_1^{(T-1)}$ similarly. By \eqref{eqn: singular value} and \eqref{eqn: error bound 1}, in the event $\Omega_1\cap \Omega_2$, we have 
\begin{align*}
\sigma_{r_1} (\bB_1^{(T-1)}) 
& \ge \sigma_{r_1}(\mat1(\bB)) \cdot \sigma_{\min} \left( \bU_2^\top \hat \bU_{2}^{(T-1)} \right) \cdot \sigma_{\min} \left( \bU_3^\top \hat \bU_{3}^{(T-1)} \right) \ge \frac{3}{4} \sigma_1,
\\
\left\|\bW_1^{(T-1)}\right\|_2 & \le C \sqrt{\frac{d_1+r_{-1}}{n_0}} .
\end{align*}
Since $\tilde \bB_1^{(T-1)} \in \RR^{d_1\times r_{-1}}$ has rank at most $r_1$ and $\hat \bU_1$ is the leading $r_1$ left singular vectors of $\hat \bB_1^{(T-1)} = \tilde\bB_1^{(T-1)} + \bW_1^{(T-1)}$,
\begin{align*}
\left\|\mathcal{P}_{\hat \bU_{1\perp}} \tilde\bB_1^{(T-1)} \right\|_{\rm F}   &\le  \sqrt{r_1} \left\|\mathcal{P}_{\hat \bU_{1\perp}} \tilde\bB_1^{(T-1)} \right\|_2 \le \sqrt{r_1} \left\|\mathcal{P}_{\hat \bU_{1\perp}} \left( \tilde\bB_1^{(T-1)}+\bW_1^{(T-1)} \right) \right\|_2 +\sqrt{r_1} \left\|\mathcal{P}_{\hat \bU_{1\perp}} \bW_1^{(T-1)} \right\|_2 \\
&= \sqrt{r_1} \min\limits_{\bY\in\RR^{d_1\times (d_2 d_3)}, \atop {\rm rank}(\bY)\le r_1} \left\| \hat\bB_1^{(T-1)} - \bY  \right\|_2 +\sqrt{r_1} \left\| \bW_1^{(T-1)} \right\|_2  \\
&\le \sqrt{r_1}  \left\| \hat\bB_1^{(T-1)} - \tilde\bB_1^{(T-1)}  \right\|_2 +\sqrt{r_1} \left\| \bW_1^{(T-1)} \right\|_2 \\
&= 2\sqrt{r_1} \left\| \bW_1^{(T-1)} \right\|_2.
\end{align*}
Similar to the derivation of \eqref{tucker assumption 3}, in the event $\Omega_3$, we can show,
\begin{align*}
\left\|\mathcal{P}_{\hat \bU_{m\perp}} \left( \tilde\bB_m^{(T-1)} - \bB_m^{(T-1)} \right) \right\|_{\rm F}   \le C\frac{\|\bD_m\|_{\rm F} \max_k d_k}{\sqrt{n_0 d}}
\end{align*}
Combing the above two bounds, it follows that, in the event $\Omega_1\cap \Omega_2\cap \Omega_3$, we have
\begin{align*}
&\left\| \mathcal{P}_{\hat \bU_{m\perp}} \mat1 \left(\calB \times_2 \hat \bU_{m+1}^{(T-1)\top} \times_3 \hat \bU_{m+2}^{(T-1)\top} \right) \right\|_{\rm F} = \left\|\mathcal{P}_{\hat \bU_{m\perp}} \bB_m^{(T-1)} \right\|_{\rm F} \\
\le & \left\|\mathcal{P}_{\hat \bU_{m\perp}} \tilde\bB_m^{(T-1)} \right\|_{\rm F} + \left\|\mathcal{P}_{\hat \bU_{m\perp}} \left( \tilde\bB_m^{(T-1)} - \bB_m^{(T-1)} \right) \right\|_{\rm F} \\
\le& C\left(\sqrt{\frac{d_m r_m+r}{n_0}} + \frac{\|\bD_m\|_{\rm F} \max_k d_k}{\sqrt{n_0 d}} \right)  .
\end{align*}
As a result, using $L^{(T-1)}\le 1/2$, in the event $\Omega_1\cap \Omega_2$, 
\begin{align}\label{eqn: orthogonal projection 1}
&\left\|\cB \times_m \mathcal{P}_{\hat \bU_{m\perp}} \right\|_{\rm F} = \left\|\mathcal{P}_{\hat \bU_{m\perp}}  \matk(\cB)  (\mathcal{P}_{\bU_{m+1}} \otimes \mathcal{P}_{\bU_{m+2}}) \right\|_{\rm F} = \left\|\mathcal{P}_{\hat \bU_{m\perp}}  \matk(\cB)  (\bU_{m+1} \otimes \bU_{m+2})\right\|_{\rm F} \notag\\
\le& \left\|\mathcal{P}_{\hat \bU_{m\perp}} \matk(\cB) \left(\hat \bU_{m+1}^{(T-1)} \otimes \hat \bU_{m+2}^{(T-1)} \right) \right\|_{\rm F} \cdot \sigma_{\min}^{-1} \left(\bU_{m+1}^\top \hat \bU_{m+1}^{(T-1)} \right)\cdot \sigma_{\min}^{-1} \left(\bU_{m+2}^\top \hat \bU_{m+2}^{(T-1)} \right) \notag \\
\le& C \left(\sqrt{\frac{d_m r_m+r}{n_0}} + \frac{\|\bD_m\|_{\rm F}\max_k d_k}{\sqrt{n_0 d}} \right) \frac{1}{\sqrt{1-(1/2)^2}} \frac{1}{\sqrt{1-(1/2)^2}} \notag \\
\le& 2C\left(\sqrt{\frac{d_m r_m+r}{n_0}} + \frac{\|\bD_m\|_{\rm F}\max_k d_k}{\sqrt{n_0 d}} \right).
\end{align}

Combining \eqref{eqn: tucker-decompostion}, \eqref{tucker assumption 3} and \eqref{eqn: orthogonal projection 1}, in the event $\Omega_{1}\cap \Omega_{2} \cap \Omega_{3}$ with probability at least $1-n_0^{-c}-\sum_{m=1}^M \exp(-c d_m)$, we have  
\begin{align*}
\left\|\hat \calB^{\rm tucker}  - \calB\right\|_{\rm F} &\le C \sqrt{ \frac{r+\sum_{m=1}^M d_m r_m}{n_0}  }  + C\|\bD_1 \|_{\rm F} \max_{1\le k\le M} \sqrt{\frac{d_k}{n_0 d_{-k}}} ,
\end{align*}
for some constant $C>0$.

\subsection{Proof of Theorem \ref{theorem: upper bound}}

For simplicity, we mainly focus on the proof for a simple scenario where the prior probabilities $\pi_1 = \pi_2 = 1/2$. Additionally, we will provide key steps of the proof for more general settings correspondingly. Let $\hat \Delta = \sqrt{\langle \hat \calB^{\rm tucker} \times_{m=1}^M \Sigma_m, \; \hat \calB^{\rm tucker} \rangle}$, the misclassification error of $\hat\delta_{\rm tucker}$ is
\begin{align*}
R_{\btheta}(\hat\delta_{\rm tucker}) &= \frac{n_1}{n_1 + n_2} \phi\left(\hat \Delta^{-1}\log(n_2/n_1) -\frac{\langle \hat \cM - \cM_1, \; \hat \calB^{\rm tucker} \rangle}{\hat \Delta} \right) \\
& + \frac{n_2}{n_1 + n_2} \bar \phi\left(\hat \Delta^{-1}\log(n_2/n_1) - \frac{\langle \hat \cM - \cM_2, \; \hat \calB^{\rm tucker} \rangle}{\hat \Delta} \right)
\end{align*} 
and the optimal misclassification error is
\begin{align*}
R_{\rm opt}=\pi_1\phi(\Delta^{-1}\log(\pi_2/\pi_1)-\Delta/2)+\pi_2 \bar \phi(\Delta^{-1}\log(\pi_2/\pi_1)+\Delta/2),    
\end{align*}
where $\phi$ is the CDF of the standard normal, and $\bar \phi(\cdot) = 1 - \phi(\cdot)$.
While the simpler version when $\pi_1 = \pi_2 = \frac{1}{2},$ are
\begin{align*}
R_{\btheta}(\hat\delta_{\rm tucker}) = \frac{1}{2} \phi\left(-\frac{\langle \hat \cM - \cM_1, \; \hat \calB^{\rm tucker} \rangle}{\hat \Delta} \right) + \frac{1}{2} \bar \phi\left(-\frac{\langle \hat \cM - \cM_2, \; \hat \calB^{\rm tucker} \rangle}{\hat \Delta} \right)    
\end{align*}
and $R_{\rm opt}=\phi(-\Delta/2)=\frac12\phi(-\Delta/2)+\frac12\bar\phi(\Delta/2),$ respectively.
Define an intermediate quantity
\begin{align*}
R^{*} = \frac{1}{2} \phi\left(-\frac{\langle \cD, \; \hat \calB^{\rm tucker} \rangle}{2 \hat \Delta} \right) + \frac{1}{2} \bar \phi\left(\frac{\langle \cD, \; \hat \calB^{\rm tucker} \rangle}{2 \hat \Delta} \right).    
\end{align*}
By Theorem \ref{theorem: tucker}, in an event $\Omega_1$ with probability at least $1-n_0^{-c_1} - \sum_{m=1}^M  e^{-c_1 d_m}$, \begin{align}\label{eq:B_tucker}
\|\hat\calB^{\rm tucker} - \cB\|_{\rm F} \le C \sqrt{ \frac{r+\sum_{m=1}^M d_m r_m}{n_0}  }  + \frac{C\left\| \calB \right\|_{\rm F} \max_{1\le m\le M}d_m}{\sqrt{n_0 d}} =o (\Delta ) .   
\end{align}

\noindent Firstly, we are going to show that $R^* -R_{\rm opt}(\btheta) \lesssim e^{-\Delta^2/8} \cdot \Delta^{-1} \cdot \|\hat \calB^{\rm tucker}  - \calB \|_{\rm F}^{2}$. 
Applying Taylor's expansion to the two terms in $R^{*}$ at $-\Delta/2$ and $\Delta/2$, respectively, we obtain 
\begin{equation}
\begin{split}
\label{eqn: taylor 1}
R^{*} - R_{\rm opt}(\btheta) =& \frac{1}{2}\left(\frac{\Delta}{2} -\frac{\langle \cD, \; \hat \calB^{\rm tucker} \rangle}{2\hat \Delta} \right) \phi^{\prime}(\frac{\Delta}{2}) + \frac{1}{2}\left(\frac{\Delta}{2} -\frac{\langle \cD, \; \hat \calB^{\rm tucker} \rangle}{2\hat \Delta} \right) \phi^{\prime}(-\frac{\Delta}{2}) \\
& + \frac{1}{2}\left(\frac{\langle \cD, \; \hat \calB^{\rm tucker} \rangle}{2\hat \Delta} - \frac{\Delta}{2} \right)^2 \phi^{\prime \prime}(t_{1,n}) + \frac{1}{2}\left(\frac{\langle \cD, \; \hat \calB^{\rm tucker} \rangle}{2\hat \Delta} - \frac{\Delta}{2} \right)^2 \phi^{\prime \prime}(t_{2,n})
\end{split}
\end{equation}
where $t_{1,n}, \; t_{2,n}$ are some constants satisfying $| t_{1,n} |, \; | t_{2,n} |$ are between $\frac{\Delta}{2}$ and $\frac{\langle \cD, \; \hat \calB^{\rm tucker} \rangle}{2\hat \Delta}$.

Since $\big(\frac{\Delta}{2} -\frac{\langle \cD, \; \hat \calB^{\rm tucker} \rangle}{2\hat \Delta} \big)$ frequently appears in \eqref{eqn: taylor 1}, we need to bound its absolute value. Let $\gamma = \calB \times_{m=1}^M \Sigma_m^{1/2}$ and $\hat \gamma = \hat \calB^{\rm tucker} \times_{m=1}^M \Sigma_m^{1/2}$, 
then by Lemma \ref{lemma:tensor norm inequality}, in the event $\Omega_1$, we have
\begin{align*}
& \left|  \Delta - \frac{\langle \cD, \; \hat \calB^{\rm tucker} \rangle}{\hat \Delta} \right| = \left| \|\gamma\|_{\rm F} -  \frac{\langle \gamma ,\; \hat \gamma \rangle}{\|\hat \gamma\|_{\rm F}} \right| = \left| \frac{\|\gamma\|_{\rm F} \cdot \|\hat\gamma\|_{\rm F} - \langle \gamma ,\; \hat \gamma \rangle}{\|\hat\gamma\|_2}  \right| \\
\lesssim& \frac{1}{\Delta}\|\hat \gamma - \gamma\|_{\rm F}^2 \lesssim \frac{1}{\Delta}\|\hat \calB^{\rm tucker} - \calB\|_{\rm F}^2.
\end{align*}
In fact, by triangle inequality,
\begin{align*}
| \hat \Delta - \Delta | &= \left\|\hat \calB^{\rm tucker}  \times_{m=1}^M \Sigma_m^{1/2} \right\|_{\rm F}- \left\| \calB \times_{m=1}^M \Sigma_m^{1/2} \right\|_{\rm F} \le \left\|\left(\hat \calB^{\rm tucker} - \calB\right) \times_{m=1}^M \Sigma_m^{1/2} \right\|_{\rm F} \le \left\| \hat \calB^{\rm tucker} - \calB \right\|_{\rm F} \prod_{m=1}^M \left\|\Sigma_m\right\|_{2}^{1/2} \\
& \lesssim \left\|\hat \calB^{\rm tucker} - \calB\right\|_{\rm F} \lesssim \sqrt{\frac{\sum_{m=1}^M d_m r_m + r}{n_0}} + \frac{\left\| \calB \right\|_{\rm F}\max_{m} d_m}{\sqrt{n_0 d}} = o(\Delta).
\end{align*}
Since $\|\hat \calB^{\rm tucker} - \calB\|_{\rm F} = o(\Delta)$, it follows that $\langle \cD, \; \hat \calB^{\rm tucker} \rangle/(2\hat \Delta) \rightarrow \Delta/2$.
Then, we have $| \phi^{\prime \prime}(t_{1,n}) | \asymp | \phi^{\prime \prime}(t_{2,n}) | \asymp \Delta  e^{-\frac{(\Delta/2)^2}{2}} = \Delta  e^{-\Delta^2/8}$.
Hence,
\begin{align*}
&\frac{1}{2}\Big(\frac{\langle \cD, \; \hat \calB^{\rm tucker} \rangle}{2\hat \Delta} - \frac{\Delta}{2} \Big)^2 \phi^{\prime \prime}(t_{1,n}) + \frac{1}{2}\Big(\frac{\langle \cD, \; \hat \calB^{\rm tucker} \rangle}{2\hat \Delta} - \frac{\Delta}{2} \Big)^2 \phi^{\prime \prime}(t_{2,n}) \\ 
&\asymp  \frac{1}{\Delta^2} \left\|\hat \calB^{\rm tucker} - \calB\right\|_{\rm F}^4 \cdot \Delta \cdot e^{-\Delta^2/8}  \asymp  \frac{1}{\Delta} e^{-\Delta^2/8} \left\|\hat \calB^{\rm tucker} - \calB\right\|_{\rm F}^4   .
\end{align*}
Then \eqref{eqn: taylor 1} can be further bounded such that
\begin{align*}
R^{*} - R_{\rm opt}(\btheta) &\asymp \Big(\frac{\Delta}{2} - \frac{\langle \cD, \; \hat \calB^{\rm tucker} \rangle}{2\hat \Delta} \Big) e^{-\frac{(\Delta/2)^2}{2}} + O\Big(\frac{1}{\Delta}  e^{-\Delta^2/8} \left\|\hat \calB^{\rm tucker} - \calB\right\|_{\rm F}^4 \Big)\\
    & \le e^{-\Delta^2/8} \cdot \left| \frac{\Delta}{2} - \frac{\langle \cD, \; \hat \calB^{\rm tucker} \rangle}{2\hat \Delta} \right| +  O\left( \frac{1}{\Delta}  e^{-\Delta^2/8} \left\|\hat \calB^{\rm tucker} - \calB\right\|_{\rm F}^4 \right) \\
    & \lesssim \frac{1}{\Delta} e^{-\Delta^2/8}  \left\|\hat \calB^{\rm tucker} - \calB\right\|_{\rm F}^2  +    \frac{1}{\Delta}  e^{-\Delta^2/8} \left\|\hat \calB^{\rm tucker} - \calB\right\|_{\rm F}^4  .
\end{align*}
Eventually we obtain $R^{*} - R_{\rm opt}(\btheta) \lesssim \Delta^{-1} e^{-\Delta^2/8} (\|\hat \calB^{\rm tucker} - \calB \|_{\rm F}^2 \vee \|\hat \calB^{\rm tucker} - \calB \|_{\rm F}^4)$ in the event $\Omega_1$ with probability at least $1-n_0^{-c_1} - \sum_{m=1}^M  e^{-c_1 d_m}$.

\noindent Next, focus on $R_{\btheta}(\hat\delta_{\rm tucker}) - R^{*}$. We apply Taylor's expansion to $R_{\btheta}(\hat\delta_{\rm tucker})$:
\begin{align}
\label{eqn: taylor 2}
   R_{\btheta}(\hat\delta_{\rm tucker}) =& \frac{1}{2} \left\{ \phi\Big(-\frac{\langle \cD, \; \hat \calB^{\rm tucker} \rangle}{2\hat \Delta} \Big) + \frac{\langle \cD, \; \hat \calB^{\rm tucker} \rangle/2 - \langle \hat \cM - \cM_1, \; \hat \calB^{\rm tucker} \rangle}{\hat \Delta}\phi^{\prime} \Big(\frac{\langle \cD, \; \hat \calB^{\rm tucker} \rangle}{2\hat \Delta} \Big)  \right.\notag\\
   & \left.+ O\left( \Delta \cdot e^{-\Delta^2/8} \right) \Big( \frac{\langle \hat \cM - \cM_1, \; \hat \calB^{\rm tucker} \rangle - \langle \cD, \; \hat \calB^{\rm tucker} \rangle/2}{\hat \Delta} \Big)^2 \;  \right\} \notag\\
   &+ \frac{1}{2} \left\{ \bar \phi \Big(\frac{\langle \cD, \; \hat \calB^{\rm tucker} \rangle}{2\hat \Delta} \Big) + \frac{\langle \cD, \; \hat \calB^{\rm tucker} \rangle/2 + \langle \hat \cM - \cM_2, \; \hat \calB^{\rm tucker} \rangle}{\hat \Delta}\phi^{\prime}\Big(\frac{\langle \cD, \; \hat \calB^{\rm tucker} \rangle}{2\hat \Delta} \Big) \right.\notag\\ 
   & \left.+ O\big( \Delta \cdot e^{-\Delta^2/8} \big) \Big( \frac{\langle \hat \cM - \cM_2, \; \hat \calB^{\rm tucker} \rangle + \langle \cD, \; \hat \calB^{\rm tucker} \rangle/2}{\hat \Delta} \Big)^2 \;  \right\} 
\end{align}
where the remaining term can be obtained similarly as \eqref{eqn: taylor 1} by using the fact that $|\phi^{\prime \prime}(t_n)| = O(\Delta \cdot e^{-\Delta^2/8})$. Now we aim to bound the following term:
\begin{align*} 
 \left| \frac{\langle \hat \cM - \cM_1, \; \hat \calB^{\rm tucker} \rangle - \langle \cD, \; \hat \calB^{\rm tucker} \rangle/2}{\hat \Delta} \right| 
\lesssim & \frac{1}{\Delta} \left| \langle \bar\calX^{(2)} - \cM_2 + \bar\calX^{(1)} - \cM_1, \; \hat \calB^{\rm tucker} \rangle \right| \notag\\
\lesssim & \frac{1}{\Delta} \left| \langle  \bar\calX^{(1)} - \cM_1, \; \hat \calB^{\rm tucker} \rangle \right|  + \frac{1}{\Delta} \left| \langle \bar\calX^{(2)} - \cM_2, \; \hat \calB^{\rm tucker} \rangle \right|.  
\end{align*}
Note that, in the event $\Omega_1$, $\|\hat \calB^{\rm tucker}\|_{\rm F} \le \|\hat\calB^{\rm tucker} - \cB\|_{\rm F} + \|\cB\|_{\rm F} \lesssim \Delta $. By Lemma \ref{lemma:low-rank-tensor}, in an event $\Omega_2$ with probability at least $1-e^{-c_2\sum_{m=1}^M d_m r_m}$,
\begin{align*}
\left| \langle  \bar\calX^{(k)} - \cM_k, \; \hat \calB^{\rm tucker} \rangle \right| &\lesssim \Delta \sqrt{\frac{\sum_{m=1}^M d_m r_m+ r}{n_0}} ,\quad k=1,2. 
\end{align*}
It follows that, in the event $\Omega_1\cap \Omega_2$, 
\begin{align}
\left| \frac{\langle \hat \cM - \cM_1, \; \hat \calB^{\rm tucker} \rangle - \langle \cD, \; \hat \calB^{\rm tucker} \rangle/2}{\hat \Delta} \right| &\lesssim \Delta \sqrt{\frac{\sum_{m=1}^M d_m r_m+ r}{n_0}},    \label{eqn: upper bound of taylor terms 2}\\
\left| \frac{\langle \hat \cM - \cM_2, \; \hat \calB^{\rm tucker} \rangle + \langle \cD, \; \hat \calB^{\rm tucker} \rangle/2}{\hat \Delta} \right| &\lesssim \Delta \sqrt{\frac{\sum_{m=1}^M d_m r_m+ r}{n_0}} .   \label{eqn: upper bound of taylor terms 3}
\end{align}
Substituting \eqref{eqn: upper bound of taylor terms 2} and \eqref{eqn: upper bound of taylor terms 3} into \eqref{eqn: taylor 2}, we obtain,
\begin{align*}
\left| R_{\btheta}(\hat\delta_{\rm tucker}) - R^{*} \right| \lesssim& \left| \frac{\langle \cD, \; \hat \calB^{\rm tucker} \rangle/2 - \langle \hat \cM - \cM_1, \; \hat \calB^{\rm tucker} \rangle}{\hat \Delta} \phi^{\prime}(\frac{\langle \cD, \; \hat \calB^{\rm tucker} \rangle}{2\hat \Delta})  \right.\\
& \left. + \frac{\langle \cD, \; \hat \calB^{\rm tucker} \rangle/2 + \langle \hat \cM - \cM_2, \; \hat \calB^{\rm tucker} \rangle}{\hat \Delta} \phi^{\prime}(\frac{\langle \cD, \; \hat \calB^{\rm tucker} \rangle}{2\hat \Delta}) + O \left(\Delta^3 e^{-\Delta^2/8} \left(\frac{\sum_{m=1}^M d_m r_m+ r}{n_0} \right) \right) \right|
\end{align*}
Since $\cD/2- (\hat \cM-\cM_1) + \cD/2 + (\hat \cM-\cM_2) = \cD - (\cM_2-\cM_1) = 0$, then it follows that
\begin{align*}
\left| R_{\btheta}(\hat\delta_{\rm tucker}) - R^{*} \right| \lesssim \Delta^3 e^{-\Delta^2/8} \left(\frac{\sum_{m=1}^M d_m r_m+ r}{n_0} \right)  .  
\end{align*}

Finally, combining the two pieces, we obtain
\begin{align*}
R_{\btheta}(\hat\delta_{\rm tucker}) -R_{\rm opt}(\btheta) \le& R_{\btheta}(\hat\delta_{\rm tucker}) -  R^{*} + R^{*} - R_{\rm opt}(\btheta) \\
\lesssim & \frac{1}{\Delta} e^{-\Delta^2/8}  \left\|\hat \calB^{\rm tucker} - \calB\right\|_{\rm F}^2  +    \frac{1}{\Delta}  e^{-\Delta^2/8} \left\|\hat \calB^{\rm tucker} - \calB\right\|_{\rm F}^4 + \Delta^3 e^{-\Delta^2/8} \left(\frac{\sum_{m=1}^M d_m r_m+ r}{n_0} \right) ,
\end{align*}
in the event $\Omega_1\cap \Omega_2$ with probability at least $1-n_0^{-c}-\sum_{m=1}^M e^{-cd_m }$.

Now consider the two case. On the one hand, when $\Delta = O(1)$, by \eqref{eq:B_tucker}, with probability at least $1-n_0^{-c}-\sum_{m=1}^M e^{-cd_m }$, we have 
\begin{align*}
R_{\btheta}(\hat\delta_{\rm tucker}) - R_{\rm opt}(\btheta) &\le C_0 \frac{\sum_{m=1}^M d_m r_m+ r}{n_0}  + C_0 \frac{\max_{1\le m\le M}d_m^2}{n_0 d}    \le C \frac{\sum_{m=1}^M d_m r_m+ r}{n_0}  .
\end{align*}
On the other hand, when $\Delta\to\infty$ as $n\to \infty$, by\eqref{eq:B_tucker}, with probability at least $1-n_0^{-c}-\sum_{m=1}^M e^{-cd_m }$, we have 
\begin{align*}
R_{\btheta}(\hat\delta_{\rm tucker}) - R_{\rm opt}(\btheta) &\le C \Delta^3 e^{-\Delta^2/8} \left(\frac{\sum_{m=1}^M d_m r_m+ r}{n_0} \right) +  C \Delta e^{-\Delta^2/8} \frac{\max_{1\le m\le M}d_m^2}{n_0 d} + C \Delta^3 e^{-\Delta^2/8} \left(\frac{\max_{1\le m\le M}d_m^2}{n_0 d} \right)^2 \\
&\le C \Delta^3 e^{-\Delta^2/8} \left(\frac{\sum_{m=1}^M d_m r_m+ r}{n_0}  +\frac{\max_{1\le m\le M}d_m^2}{n_0 d} \right)\\
&= C \exp\left(-\left(\frac18-\frac{3\log(\Delta)}{\Delta^2} \right)\Delta^2\right) \left(\frac{\sum_{m=1}^M d_m r_m+ r}{n_0}  +\frac{\max_{1\le m\le M}d_m^2}{n_0 d} \right),
\end{align*}
where $3\log(\Delta)/\Delta^2$ is an $o(1)$ term as $n\to\infty$.

\subsection{Proof of Theorem \ref{theorem: lower bound}}

Note that the proof is not straightforward, partly because the excess risk $R_{\btheta}(\hat\delta_{\rm tucker}) -R_{\rm opt}(\btheta)$ does not satisfy the triangle inequality required by standard lower bound techniques. A crucial approach in this context is establishing a connection to an alternative risk function.
For a general classification rule $\delta$, we define $L_{\btheta}(\delta)=\PP_{\btheta}(\delta(\cZ) \neq \delta_{\theta}(\cZ))$, where $\delta_{\theta}(\cZ)$ is the Fisher’s linear discriminant rule introduced in (\ref{eqn:lda-rule}). Lemma \ref{lemma:the first reduction} 
allows us to transform the excess risk $R_{\btheta}(\hat\delta_{\rm tucker}) -R_{\rm opt}(\btheta)$ into the risk function $L_{\btheta}(\hat \delta_{\rm tucker})$, as shown below:
\begin{equation} \label{eqn:loss function reduction}
R_{\btheta}(\hat\delta_{\rm tucker}) -R_{\rm opt}(\btheta) \ge \frac{\sqrt{2\pi}\Delta}{8} e^{\Delta^2/8} \cdot L_{\theta}^2(\hat \delta_{\rm tucker}).    
\end{equation}
We then apply Lemma \ref{lemma:Tsybakov variant} to derive the minimax lower bound for the risk function $L_{\btheta}(\hat \delta_{\rm tucker})$.

We carefully construct a finite collection of subsets of the parameter space $\calH$ that characterizes the hardness of the problem. 
Any $M$-th order tensor $\cM \in \mathbb{R}^{d_1 \times \cdots \times d_M}$ with Tucker rank $(r_1,...,r_M)$ can be expressed as $\cM= \cF \times_{m=1}^M \bA_m$. Here, the latent core tensor $\cF$ has dimensions $r_1 \times \cdots \times r_M$, and the loading matrices $\bA_m \in \mathbb{R}^{d_m \times r_m}$ for each mode $m=1,\ldots,M$. Note that due to the rotational ambiguity, the matrices $\bA_m$ are not necessarily orthogonal.

(a). Consider the complete data case.

First, let $\bA_m$ be a fixed matrix where the $(i,i)$-th elements, $i=1,...,r_m$, are set to one and all other elements are zero. Denote 
$\bA=\bA_M\otimes \bA_{M-1}\otimes \cdots\otimes \bA_1$. According to basic tensor algebra, this setup implies that 
$\vect(\cM) = \bA \vect(\cF)$ and $\|\cM\|_{\rm F} = \|\vect(\cF)\|_2$. Let $\be_1$ be the basis vector in the standard Euclidean space whose first entry is 1 and 0 elsewhere, and $\bI_{d}=[\bI_{d_m}]_{m=1}^M$.
Define the following parameter space
\begin{align*}
\cH_0 =& \big\{ \theta = (\cM_1, \; \cM_2, \; \bI_{d}) : \; \cM_1 = \cF \times_{m=1}^M \bA_m, \; \cM_2= -\cM_1; 
\ \vect(\cF)=\epsilon \bff+ \lambda \be_1, \bff \in \{ 0,1\}^r, \bff^\top \be_1=0 \big\},
\end{align*}
where $r=\prod_{m=1}^M r_m$, $\epsilon=c/\sqrt{n}$, $c=O(1)$ and $\lambda$ is chosen to ensure that $\theta\in\cH$ such that
\begin{align*}
\Delta=(\cM_2-\cM_1)^\top\bSigma^{-1} (\cM_2 -\cM_1) = 4 \| \epsilon \bff+ \lambda \be_1\|_2^2 = 4 \epsilon^2 \| \bff\|_2^2 + 4 \lambda^2. 
\end{align*}
In addition to $\cH_0$, we also define $\bA_{\ell}$, for $\ell\neq m$, as fixed matrices where the $(i,i)$-th elements, $i=1,...,r_{\ell}$, are set to one and all other elements are zero. Let $\cF$ be a tensor such that in $\matk(\cF)$, the $(i,i)$-th elements, $i=1,..., (r_m\wedge r_{-m}) $, are set to one, and all other elements are zero. As $r_m^2\le r$, it implies that $\|\cM\|_{\rm F} = \|\bA_m\matk(\cF)\|_{\rm F}=\| \bA_m\|_{\rm F}$. For $m=1,...,M$, define the following parameter spaces 
\begin{align*}
\cH_m =& \big\{ \theta = (\cM_1, \; \cM_2, \; \bI_{d}) : \; \cM_1 = \cF \times_{k=1}^M \bA_k, \; \cM_2= -\cM_1; 
\ \vect(\bA_m)=\epsilon \bg_m+ \lambda_m \be_1, \bg_m \in \{ 0,1\}^{d_m r_m}, \bg_m^\top \be_1=0 \big\},
\end{align*}
where $\epsilon=c/\sqrt{n}$, $c=O(1)$ and $\lambda_m$ is chosen to ensure that $\theta\in\cH$ such that
\begin{align*}
\Delta=(\cM_2-\cM_1)^\top\bSigma^{-1} (\cM_2 -\cM_1) = 4 \| \epsilon \bg_m+ \lambda_m \be_1\|_2^2 = 4 \epsilon^2 \| \bg_m\|_2^2 + 4 \lambda_m^2. 
\end{align*}
It is clear that $\cap_{\ell=0}^M \cH_{\ell} \subset \cH$. We shall show below separately for the minimax risks over each parameter space $\cH_{\ell}$.

First, consider $\cH_0$. By Lemma \ref{lemma:Varshamov-Gilbert Bound}, we can construct a sequence of $r$-dimensional vectors $\bff_1, \ldots ,\bff_N \in \{0,1\}^r$, such that $\bff_{i}^\top \be_1=0$, $\rho_H(\bff_i, \bff_j) \ge r/8, \; \forall 0 \le i < j \le N$, and $r \le (8/\log 2)\log N$, where $\rho_H$ denotes the Hamming distance.
To apply Lemma \ref{lemma:Tsybakov variant}, for $\forall \theta_{\bu}, \; \theta_{\bv} \in \cH_0, \; \theta_{\bu} \neq \theta_{\bv},$ we need to verify two conditions:  
\begin{enumerate}
\item[(i)] the upper bound on the Kullback-Leibler divergence between $\PP_{\theta_{\bu}}$ and $\PP_{\theta_{\bv}}$, and
\item[(ii)] the lower bound of $L_{\theta_{\bu}}(\hat \delta_{\rm tucker}) + L_{\theta_{\bv}}(\hat \delta_{\rm tucker})$ for $\bu \neq \bv$ and $\bu^\top \be_1=0, \bv^\top \be_1=0$.
\end{enumerate}

We calculate the Kullback-Leibler divergence first. For $\bff_{\bu} \in\{ 0,1\}^r$ and $\bff_{\bu}^\top \be_1=0$, define
\begin{align*}
\cM_{\bu} = \cF_{\bu} \times_{m=1}^M \bA_m, \; \vect(\cF_{\bu})= \epsilon \bff_{\bu} + \lambda \be_1, \; \theta_{\bu} = \big(\cM_{\bu}, \; -\cM_{\bu}, \; \bI_d \big) \in \cH_0 .   
\end{align*}
and consider the distribution $\cT\cN(\cM_{\bu}, \; \bI_d)$. 
Then the Kullback-Leibler divergence between $\PP_{\theta_{\bu}}$ and $\PP_{\theta_{\bv}}$ can be bounded by
\begin{align*}
    {\rm KL}(\PP_{\theta_{\bu}}, \PP_{\theta_{\bv}}) &= \frac{1}{2} \norm{\vect(\cM_{\bu}) - \vect(\cM_{\bv})}_2^2 = \frac{1}{2}\norm{\vect(\cF_{\bu}) - \vect(\cF_{\bv})}_2^2 \le \frac{c^2 r}{2n}  .
\end{align*}
In addition, by applying Lemma \ref{lemma:probability inequality}, we have that for any $\bff_{\bu}, \bff_{\bv} \in\{ 0,1\}^r$,
\begin{align*}
    L_{\theta_{\bu}}(\hat \delta_{\rm tucker}) + L_{\theta_{\bv}}(\hat \delta_{\rm tucker}) &\ge \frac{1}{\Delta} e^{-\Delta^2/8} \cdot \norm{\vect(\cF_{\bu}) - \vect(\cF_{\bv})}_2 \\
    &\ge \frac{1}{\Delta} e^{-\Delta^2/8} \sqrt{\frac{r}{8} \cdot \frac{c^2}{n}  } \\
    &\gtrsim \frac{1}{\Delta} e^{-\Delta^2/8} \sqrt{\frac{r}{n}} .
\end{align*}
So far we have verified the aforementioned conditions (i) and (ii). Lemma \ref{lemma:Tsybakov variant} immediately implies that, there exists some constant $C_{\alpha} > 0$, such that 
\begin{equation}\label{eqn:inter1}
\inf_{\hat \delta_{\rm tucker}} \sup_{\theta \in \calH_0} \PP\left(L_{\theta}(\hat \delta_{\rm tucker}) \ge C_{\alpha} \frac{1}{\Delta} e^{-\Delta^2/8} \sqrt{ \frac{r}{n} } \right) \ge 1-\alpha
\end{equation}
Combining \eqref{eqn:inter1} and \eqref{eqn:loss function reduction}, we have
\begin{equation}\label{eqn:inter2}
\inf_{\hat \delta_{\rm tucker}} \sup_{\theta \in \calH_0} \PP\left(R_{\btheta}(\hat\delta_{\rm tucker}) -R_{\rm opt}(\btheta) \ge C_{\alpha} \frac{1}{\Delta} e^{-\Delta^2/8} \cdot \frac{r}{n}  \right) \ge 1-\alpha   .
\end{equation}

Similarly, for each $\cH_m$, $m=1,...,M$, we can obtain that,  there exists some constant $C_{\alpha} > 0$, such that 
\begin{equation}\label{eqn:inter3}
\inf_{\hat \delta_{\rm tucker}} \sup_{\theta \in \calH_m} \PP\left(R_{\btheta}(\hat\delta_{\rm tucker}) -R_{\rm opt}(\btheta) \ge C_{\alpha} \frac{1}{\Delta} e^{-\Delta^2/8} \cdot \frac{d_m r_m }{n}  \right) \ge 1-\alpha   .
\end{equation}

Finally combining \eqref{eqn:inter2} and \eqref{eqn:inter3}, we obtain the desired lower bound for the excess misclassficiation error
\begin{equation}\label{eqn:mis}
\inf_{\hat \delta_{\rm tucker}} \sup_{\theta \in \calH} \PP\left(R_{\btheta}(\hat\delta_{\rm tucker}) -R_{\rm opt}(\btheta) \ge C_{\alpha} \frac{1}{\Delta} e^{-\Delta^2/8} \cdot \frac{\sum_{m=1}^M d_m r_m + r}{n}  \right) \ge 1-\alpha   .
\end{equation}

This implies that if $c_1<\Delta \le c_2$ for some $c_1,c_2>0$, we have 
\begin{align*}
\inf_{\hat \delta_{\rm tucker}} \sup_{\theta \in \calH} \PP\left(R_{\btheta}(\hat\delta_{\rm tucker}) -R_{\rm opt}(\btheta) \ge C_{\alpha} \cdot \frac{\sum_{m=1}^M d_m r_m + r}{n}  \right) \ge 1-\alpha   .    
\end{align*}
On the other hand, if $\Delta\to\infty$ as $n\to \infty$, then for any $\vartheta>0$
\begin{align*}
\inf_{\hat \delta_{\rm tucker}} \sup_{\theta \in \calH} \PP\left(R_{\btheta}(\hat\delta_{\rm tucker}) -R_{\rm opt}(\btheta) \ge C_{\alpha} \exp\left\{-\left(\frac18+\vartheta\right)\Delta^2 \right\} \frac{\sum_{m=1}^M d_m r_m + r}{n}  \right) \ge 1-\alpha   .    
\end{align*}

(b). Consider the incomplete data case with $n_0 \ge 1$. We focus on a special pattern of missingness $\mathscr S_0$:
\begin{align*}
\cS_{i,\cI}=1, i=1,...,n_0,\ \text{and}\ \cS_{i,\cI}=0, i=n_0+1,...,n, \forall \cI=(i_1,...,i_M), 1\le i_m \le d_m , \text{with probability 1} . 
\end{align*}
Under this missingness pattern, $n_*(\mathscr S_0)= n_0$ with probability 1, and the problem essentially becomes complete data problem with $n_0$ samples.


\section{Technical Lemmas}
We collect all technical lemmas that has been used in the theoretical proofs throughout the paper in this section. Let $d=d_1d_2\cdots d_M$ and $d_{-m}=d/d_m.$ Denote $\|A\|_2$ or $\|A\|$ as the spectral norm of a matrix $A$. Also $\otimes$ denotes the Kronecker product.

\begin{lemma}\label{lemma:epsilonnet}
Let $d, d_j, d_*, r\le d\wedge d_j$ be positive integers, $\epsilon>0$ and
$N_{d,\epsilon} = \lfloor(1+2/\epsilon)^d\rfloor$. \\
(i) For any norm $\|\cdot\|$ in $\RR^d$, there exist
$M_j\in \RR^d$ with $\|M_j\|\le 1$, $j=1,\ldots,N_{d,\epsilon}$,
such that
\begin{align*}
\max_{\|M\|\le 1}\min_{1\le j\le N_{d,\epsilon}}\|M - M_j\|\le \epsilon .    
\end{align*}
Consequently, for any linear mapping $f$ and norm $\|\cdot\|_*$,
\begin{align*}
\sup_{M\in \RR^d,\|M\|\le 1}\|f(M)\|_* \le 2\max_{1\le j\le N_{d,1/2}}\|f(M_j)\|_*.    
\end{align*}
(ii) Given $\epsilon >0$, there exist $U_j\in \RR^{d\times r}$
and $V_{j'}\in \RR^{d'\times r}$ with $\|U_j\|_{2}\vee\|V_{j'}\|_{2}\le 1$ such that
\begin{align*}
\max_{M\in \RR^{d\times d'},\|M\|_{2}\le 1,\text{rank}(M)\le r}\
\min_{j\le N_{dr,\epsilon/2}, j'\le N_{d'r,\epsilon/2}}\|M - U_jV_{j'}^\top\|_{2}\le \epsilon.    
\end{align*}
Consequently, for any linear mapping $f$ and norm $\|\cdot\|_*$ in the range of $f$,
\begin{equation}\label{lm-3-2}
\sup_{M, \widetilde M\in \RR^{d\times d'}, \|M-\widetilde M\|_{2}\le \epsilon
\atop{\|M\|_{2}\vee\|\widetilde M\|_{2}\le 1\atop
\text{rank}(M)\vee\text{rank}(\widetilde M)\le r}}
\frac{\|f(M-\widetilde M)\|_*}{\epsilon 2^{I_{r<d\wedge d'}}}
\le \sup_{\|M\|_{2}\le 1\atop \text{rank}(M)\le r}\|f(M)\|_*
\le 2\max_{1\le j \le N_{dr,1/8}\atop 1\le j' \le N_{d'r,1/8}}\|f(U_jV_{j'}^\top)\|_*.
\end{equation}
(iii) Given $\epsilon >0$, there exist $U_{j,k}\in \RR^{d_k\times r_k}$
and $V_{j',k}\in \RR^{d'_k\times r_k}$ with $\|U_{j,k}\|_{2}\vee\|V_{j',k}\|_{2}\le 1$ such that
\begin{align*}
\max_{M_k\in \RR^{d_k\times d_k'},\|M_k\|_{2}\le 1\atop \text{rank}(M_k)\le r_k, \forall k\le K}\
\min_{j_k\le N_{d_kr_k,\epsilon/2} \atop j'_k\le N_{d_k'r_k,\epsilon/2}, \forall k\le K}
\Big\|\otimes_{k=2}^K M_k - \otimes_{k=2}^K(U_{j_k,k}V_{j_k',k}^\top)\Big\|_{2}\le \epsilon (K-1).    
\end{align*}
For any linear mapping $f$ and norm $\|\cdot\|_*$ in the range of $f$,
\begin{equation}\label{lm-3-3}
\sup_{M_k, \widetilde M_k\in \RR^{d_k\times d_k'},\|M_k-\widetilde M_k\|_{2}\le\epsilon\atop
{\text{rank}(M_k)\vee\text{rank}(\widetilde M_k)\le r_k \atop \|M_k\|_{2}\vee\|\widetilde M_k\|_{2}\le 1\ \forall k\le K}}
\frac{\|f(\otimes_{k=2}^KM_k-\otimes_{k=2}^K\widetilde M_k)\|_*}{\epsilon(2K-2)}
\le \sup_{M_k\in \RR^{d_k\times d_k'}\atop {\text{rank}(M_k)\le r_k \atop \|M_k\|_{2}\le 1, \forall k}}
\Big\|f\big(\otimes_{k=2}^K M_k\big)\Big\|_*
\end{equation}
and
\begin{equation}\label{lm-3-4}
\sup_{M_k\in \RR^{d_k\times d_k'},\|M_k\|_{2}\le 1\atop \text{rank}(M_k)\le r_k\ \forall k\le K}
\Big\|f\big(\otimes_{k=2}^K M_k\big)\Big\|_*
\le 2\max_{1\le j_k \le N_{d_kr_k,1/(8K-8)}\atop 1\le j_k' \le N_{d_k'r_k,1/(8K-8)}}
\Big\|f\big(\otimes_{k=2}^K U_{j_k,k}V_{j_k',k}^\top\big)\Big\|_*.
\end{equation}
\end{lemma}

\begin{lemma}\label{lm-pertubation}
Let $r \le d_1\wedge d_2$, $M$ be a $d_1\times d_2$ matrix, 
$U$ and $V$ be, respectively, the left and right singular matrices associated 
with the $r$ largest singular values of $M$,
$U_{\perp}$ and $V_{\perp}$ be the orthonormal complements of $U$ and $V$, 
and $\lam_r$ be the $r$-th largest singular value of $M$. 
Let $\widehat M = M + \Delta$ be a noisy version of $M$, 
$\{\widehat U, \widehat V, {\widehat U}_{\perp}, {\widehat V}_{\perp}\}$ 
be the counterpart of $\{U,V,V_\perp,V_\perp\}$, and 
${\widehat \lam}_{r+1}$ be the $(r+1)$-th largest singular value of $\widehat M$. 
Let $\|\cdot\|$ be a
matrix norm satisfying $\|ABC\|\le \|A\|_{2}\|C\|_{2}\|B\|$, 
$\epsilon_1= \|U^\top \Delta {\widehat V}_{\perp}\|$ and $\epsilon_2 = \|{\widehat U}_{\perp}^\top \Delta V\|$. Then,
\begin{align}\label{wedin+}
\| U_{\perp}^\top \widehat U \| 
\le \frac{{\widehat \lam}_{r+1}\epsilon_1+\lam_r\epsilon_2}{\lambda_r^2 - {\widehat \lam}_{r+1}^2}
\le \frac{\epsilon_1\vee\epsilon_2}{\lambda_r - {\widehat \lam}_{r+1}}.   
\end{align}
In particular, for the spectral norm $\|\cdot\|=\|\cdot\|_{2}$, $\hbox{\rm error}_1 =\|\Delta\|_{2}/\lambda_r$ 
and $\hbox{\rm error}_2 =\epsilon_2/\lambda_r$, 
\begin{align}\label{wedin-2}
\|\widehat U \widehat U^\top  -U U^\top\|_{2}\le \frac{\hbox{\rm error}_1^2+\hbox{\rm error}_2}{1-\hbox{\rm error}_1^2}.
\end{align}
\end{lemma}

Lemma \ref{lemma:epsilonnet} applies an $\epsilon$-net argument for matrices, as derived in Lemma G.1 in \cite{han2020iterative}. Lemma \ref{lm-pertubation} enhances the matrix perturbation bounds of \cite{wedin1972perturbation}, as derived in Lemma 4.1 in \cite{han2020iterative}. This sharper perturbation bound, detailed in the middle of \eqref{wedin+}, improves the commonly used version of the \cite{wedin1972perturbation} bound on the right-hand side,  
compared with Theorem 1 of \cite{cai2018rate} and Lemma 1 of \cite{chen2022rejoinder}. As \cite{cai2018rate} pointed out, such variations of the \cite{wedin1972perturbation} bound offer more precise convergence rates when when $\hbox{\rm error}_2\le \hbox{\rm error}_1$ in \eqref{wedin-2},  typically in the case of $d_1\ll d_2$.

The following lemma characterizes the accuracy of estimating the inverse of the covariance matrix $\Sigma_m^{-1}, m=1,\dots,M$ using sample covariance, and the convergence rate of the normalization constant. 
\begin{lemma} \label{lemma:precision matrix}
(i) Let $\bX_1, \cdots ,\bX_n \in \RR^{d_m \times d_{-m}}$ be i.i.d. random matrices, each following the matrix normal distribution $\bX \sim \cM\cN_{d_m \times d_{-m}}(\mu, \; \Sigma_m, \; \Sigma_{-m})$. Consider missing completely at random (MCR), $\cS=\{\bS_i\in\{0,1\}^{d_m\times d_{-m}}, i=1,...,n  \}$, where $\bS_i$ can be either deterministic or random, but independent of $\bX_i$. Let
\begin{align*}
n_{0,jl, kl'}=\sum_{i=1}^n \bS_{i,jl}\bS_{i,kl'}    ,\quad n_{0,jl}=n_{0,jl, jl}=\sum_{i=1}^n \bS_{i,jl}^2=\sum_{i=1}^n \bS_{i,jl},\quad \text{and}\ n_0=\min_{j,k,l} n_{0,jl, kl}.
\end{align*}
To estimate $\Sigma_m$ and its inverse $\Sigma_m^{-1}$, we utilize the generalized sample mean $\bar \bX=(\bar \bX_{jl})$, the generalized sample covariance $\hat\Sigma_m := (\hat \Sigma_{m, jk}) $, and its inverse $\hat\Sigma_m^{-1}$, where
\begin{align*}
\bar\bX_{jl} &= \frac{1}{n_{0,jl}} \sum_{i=1}^n \bX_{i, jl} \bS_{i, jl}  ,\\
\hat \Sigma_{m, jk}&= \frac{1}{d_{-m}}\sum_{l=1}^{d_{-m}}  \frac{1}{n_{0,jl,kl}} \sum_{i=1}^n (\bX_{i,jl}-\bar \bX_{jl}) (\bX_{i,kl}-\bar \bX_{kl}) \bS_{i,jl} \bS_{i,kl}.
\end{align*}
Then there exists constants $C>0$ such that if $n_0 d_{-m} \gtrsim d_m(1 \vee \|\Sigma_{-m} \|_2^2)$ and $n_0\asymp n$, in an event with probability at least $1-\exp(-cd_m)$, we have
\begin{align}
& \left\|\hat\Sigma_m - C_{m,\sigma} \Sigma_m\right\|_2 \leq C \cdot \left\| \Sigma_m \right\|_2 \left\| \Sigma_{-m}\right\|_2 \sqrt{\frac{d_m}{n_0 d_{-m}}} ,  \label{eqn:covariance matrix} \\
& \left\| \hat\Sigma_m^{-1} - (C_{m,\sigma})^{-1} \Sigma_m^{-1} \right\|_2 \leq C (C_{m,\sigma}^{-2} \vee C_{m,\sigma}^{-4}) \left\|\Sigma_m^{-1} \right\|_2 \left\|\Sigma_{-m}\right\|_2 \sqrt{\frac{d_m}{n_0 d_{-m}}}, \label{eqn:inverse matrix}
\end{align}
where $C_{m,\sigma}=\tr(\Sigma_{-m})/d_{-m}$.

(ii) Let $\cX_1, \cdots ,\cX_n \in \RR^{d_1\times \cdots \times d_{M}}$ be i.i.d. random tensors, each following the tensor normal distribution $\cX \sim \cT\cN(\mu, \; \bSigma)$, where $\bSigma=[\Sigma_m]_{m=1}^M$ and $\Sigma_m\in\RR^{d_m \times d_m}$. Consider $\cS=\{\bS_i\in\{0,1\}^{d_1\times \cdots \times d_{M}}, i=1,...,n  \}$, where $\bS_i$ can be either deterministic or random, but independent of $\cX_i$. For all $\cI=(i_1, i_2, \dots, i_M), \cI'=(i_1', i_2', \dots, i_M')$, let
\begin{align*}
n_{0,\cI, \cI'}=\sum_{i=1}^n \bS_{i,\cI}\bS_{i,\cI'}    ,\quad n_{0,\cI}=n_{0,\cI, \cI}=\sum_{i=1}^n \bS_{i,\cI}^2=\sum_{i=1}^n \bS_{i,\cI},\quad \text{and}\ n_0=\min_{i_1,i_1',i_2,...,i_M} n_{0,(i_1,i_2,...,i_M), (i_1',i_2,...,i_M)}.
\end{align*}
Define $\hat{\Var}(\cX_{1\cdots1})=n_{0,1\cdots1}^{-1}\sum_{i=1}^n (\cX_{i,1\cdots1}- \bar \cX_{1\cdots1})^2 \bS_{i,1\cdots 1}$ as the sample variance of the first element of $\cX$, and $\Var(\cX_{1\cdots 1})=\prod_{m=1}^M\Sigma_{m,11}$, $\bar \cX_{1\cdots1}=n_{0,1\cdots1}^{-1}\sum_{i=1}^n\cX_{i,1\cdots1}\bS_{i,1\cdots1}$. Then in an event with probability at least $1-\exp(-c(t_1+t_2))$, we have
\begin{align}
& \left|\frac{\prod_{m=1}^M \hat\Sigma_{m,11}}{\hat{\Var}(\cX_{1\cdots1})} - C_{\sigma} \right| \leq C_{M} \Var(\cX_{1\cdots 1}) \left(\max_m \frac{\|\otimes_{k\neq m}\Sigma_{k}\|_2}{\sqrt{d_{-m}}} \cdot \sqrt{\frac{ t_1}{n_0 }} +   \sqrt{\frac{ t_2}{n_0 }}   \right),  \label{eqn:const} 
\end{align}
where $C_{\sigma}=\prod_{m=1}^M C_{m,\sigma} = [\prod_{m=1}^M\tr(\Sigma_m)/d]^{M-1} =  [\tr(\bSigma)/d]^{M-1}$ and $C_{M}$ depends on $M$.
\end{lemma}

\begin{proof}
We first show \eqref{eqn:covariance matrix}.
Note that $\EE[(\bX-\mu)(\bX-\mu)^\top] = \tr(\Sigma_{-m})\cdot \Sigma_m = C_{m,\sigma}d_{-m} \cdot \Sigma_m$, and thus $\Sigma_m = \EE[(\bX-\mu)(\bX-\mu)^\top] / ( C_{m,\sigma} d_{-m})$. Without loss of generality, assume $\mu=0$.
Consider a sequence of independent copies $\bZ_1, \dots, \bZ_n$ of $\bZ \in \RR^{d_m \times d_{-m}}$ with entries $z_{ij}$ that are i.i.d. and follow $N(0, 1)$. The Gaussian random matrices $\bX_i$ are then given by $\bX_i - \mu = \bA\bZ_i\bB^\top$, where $\bA\bA^\top = \Sigma_m$ and $\bB\bB^\top = \Sigma_{-m}$. 
Note that
\begin{align*}
\hat \Sigma_{m,jk} =& \frac{1}{d_{-m}}\sum_{l=1}^{d_{-m}}  \frac{1}{n_{0,jl,kl}} \sum_{i=1}^n (\bX_{i,jl}-\bar \bX_{jl}) (\bX_{i,kl}-\bar \bX_{kl}) \bS_{i,jl} \bS_{i,kl}  \\
=& \frac{1}{d_{-m}}\sum_{l=1}^{d_{-m}}  \frac{1}{n_{0,jl,kl}} \sum_{i=1}^n \bX_{i,jl}\bX_{i,kl} \bS_{i,jl} \bS_{i,kl} + \frac{1}{d_{-m}}\sum_{l=1}^{d_{-m}}  \frac{1}{n_{0,jl,kl}} \sum_{i=1}^n \bar \bX_{jl} \bar \bX_{kl} \bS_{i,jl} \bS_{i,kl} \\
&-\frac{1}{d_{-m}}\sum_{l=1}^{d_{-m}}  \frac{1}{n_{0,jl,kl}} \sum_{i=1}^n (\bX_{i,jl}\bar \bX_{kl} +\bar \bX_{jl}\bX_{i,kl} ) \bS_{i,jl} \bS_{i,kl} \\
=& \frac{1}{d_{-m}}\sum_{l=1}^{d_{-m}}  \frac{1}{n_{0,jl,kl}} \sum_{i=1}^n \bX_{i,jl}\bX_{i,kl} \bS_{i,jl} \bS_{i,kl} 
\\
&- \frac{1}{d_{-m}}\sum_{l=1}^{d_{-m}}  \sum_{i,i'=1}^n  \bX_{i,jl}  \bX_{i',kl} \left( \frac{\bS_{i,jl}\bS_{i,kl}\bS_{i',kl}}{n_{0,jl,kl}n_{0,kl}} + \frac{\bS_{i,jl}\bS_{i,kl}\bS_{i',jl}}{n_{0,jl,kl}n_{0,kl}}-\frac{\bS_{i,jl} \bS_{i',kl} }{n_{0,jl}n_{0,kl}} \right).
\end{align*}
It follows that
\begin{align*}
\hat\Sigma_{m,jk} - C_{m,\sigma} \Sigma_{m,jk } =& \frac{1}{d_{-m}}\sum_{l=1}^{d_{-m}}  \frac{1}{n_{0,jl,kl}} \sum_{i=1}^n \left( \bX_{i,jl}\bX_{i,kl} \bS_{i,jl} \bS_{i,kl} - C_{m,\sigma} \Sigma_{m,jk } \right)\\
&- \frac{1}{d_{-m}}\sum_{l=1}^{d_{-m}}  \sum_{i,i'=1}^n  \bX_{i,jl}  \bX_{i',kl} \left( \frac{\bS_{i,jl}\bS_{i,kl}\bS_{i',kl}}{n_{0,jl,kl}n_{0,kl}} + \frac{\bS_{i,jl}\bS_{i,kl}\bS_{i',jl}}{n_{0,jl,kl}n_{0,kl}}-\frac{\bS_{i,jl} \bS_{i',kl} }{n_{0,jl}n_{0,kl}} \right) \\
:=& \check\Sigma_{m,jk} - C_{m,\sigma} \Sigma_{m,jk } - \Delta_{2,jk}.
\end{align*}
For any unit vector $v \in \RR^{d_m}$, 
$v^\top(\bX_i-\mu) = v^\top \bA \bZ_i \bB^\top =\vec1(v^\top \bA \bZ_i \bB^\top) = (\bB\otimes v^\top \bA) \vec1(\bZ_i)$.
As $\EE \check\Sigma_{m,jk} - C_{m,\sigma} \Sigma_{m,jk }=0$, we have
\begin{align*} 
&  v^\top\left( \check\Sigma_{m} - C_{m,\sigma} \Sigma_{m} \right) v \\
=& \frac{1}{d_{-m}}\sum_{j,k=1}^{d_m}\sum_{l=1}^{d_{-m}}  \sum_{i=1}^n \left( \bX_{i,jl}\bX_{i,kl} \frac{ v_j v_k\bS_{i,jl} \bS_{i,kl} }{n_{0,jl,kl}} - \EE \bX_{i,jl}\bX_{i,kl} \frac{ v_j v_k\bS_{i,jl} \bS_{i,kl} }{n_{0,jl,kl}} \right)\\
=& \sum_{i=1}^n {\vec1}^\top (\bX_i) \bM^{(i)} \vec1(\bX_i) -  \EE {\vec1}^\top (\bX_i) \bM^{(i)} \vec1(\bX_i)   \\
=& \sum_{i=1}^n {\vec1}^\top (\bZ_i)  (\bB^\top \otimes \bA^\top )\bM^{(i)} (\bB \otimes \bA)\vec1(\bZ_i) -  \EE {\vec1}^\top (\bZ_i)  (\bB^\top \otimes \bA^\top )\bM^{(i)} (\bB \otimes \bA)\vec1(\bZ_i)  ,
\end{align*}
where $\vec1(\bX_i)= (\bB\otimes \bA) \vec1(\bZ_i)$, 
\begin{align*}
\bM^{(i)}_{jl,kl} = \frac{ v_j v_k\bS_{i,jl} \bS_{i,kl} }{d_{-m} n_{0,jl,kl}}, \quad \text{and}\ \bM^{(i)}_{jl,kl'}=0\ \text{for }l\neq l'.    
\end{align*}
Basic calculation shows
\begin{align*}
\left\| \bM^{(i)} \right\|_2 &\le \left\| (\bI_{d_{-m}}\otimes v) (\bI_{d_{-m}} \otimes v^\top ) \right\|_2/ (n_0 d_{-m}) \le \frac{1}{n_0 d_{-m}},    \\
\left\| \bM^{(i)} \right\|_{\rm F}^2 &\le \left\| (\bI_{d_{-m}}\otimes v) (\bI_{d_{-m}} \otimes v^\top ) \right\|_{\rm F}^2/ (n_0^2 d_{-m}^2) \le \frac{1}{n_0^2 d_{-m}}.
\end{align*}
By Hanson-Wright inequality, for any $t>0$,
\begin{align*}
&\PP\left(  v^\top\left( \check\Sigma_{m} - C_{m,\sigma} \Sigma_{m} \right) v \ge t \right) \\
\le& 2 \exp\left( - c \min\left\{ \frac{t^2}{16 \left\| \Sigma_{-m} \right\|_2^2 \left\| \Sigma_{m} \right\|_2^2 \sum_{i=1}^n \left\| \bM^{(i)} \right\|_{\rm F}^2  } , \frac{t }{4 \left\| \Sigma_{-m} \right\|_2 \left\| \Sigma_{m} \right\|_2 \max_i \left\| \bM^{(i)} \right\|_2 }\right\} \right) \\ 
\le& 2 \exp\left( - c \min\left\{ \frac{t^2}{16 \left\| \Sigma_{-m} \right\|_2^2 \left\| \Sigma_{m} \right\|_2^2 n/(n_0^2 d_{-m})  } , \frac{t }{4 \left\| \Sigma_{-m} \right\|_2 \left\| \Sigma_{m} \right\|_2 /(n_0 d_{-m})} \right\} \right).
\end{align*}

Similarly, consider $\Delta_2$. For any unit vector $v \in \RR^{d_m}$, we have
\begin{align*} 
& v^\top\left( \Delta_2 - \EE \Delta_2 \right) v \\
=& \sum_{j,k=1}^{d_m} \Delta_{2,jk} v_j v_k - \EE \Delta_{2,jk} v_j v_k  \\
=& \sum_{i,i'=1}^n {\vec1}^\top (\bX_i) \bM^{(i,i')} \vec1(\bX_{i'}) -  \EE {\vec1}^\top (\bX_i) \bM^{(i,i')} \vec1(\bX_{i'})   \\
=& \sum_{i,i'=1}^n {\vec1}^\top (\bZ_i)  (\bB^\top \otimes \bA^\top )\bM^{(i,i')} (\bB \otimes \bA)\vec1(\bZ_i) -  \EE {\vec1}^\top (\bZ_i)  (\bB^\top \otimes \bA^\top )\bM^{(i,i')} (\bB \otimes \bA)\vec1(\bZ_i)  ,
\end{align*}
where 
\begin{align*}
\bM^{(i,i')}_{jl,kl} =  v_j v_k \left( \frac{\bS_{i,jl}\bS_{i,kl}\bS_{i',kl}}{d_{-m} n_{0,jl,kl}n_{0,kl}} + \frac{\bS_{i,jl}\bS_{i,kl}\bS_{i',jl}}{d_{-m} n_{0,jl,kl}n_{0,kl}}-\frac{\bS_{i,jl} \bS_{i',kl} }{d_{-m}n_{0,jl}n_{0,kl}} \right) , \quad \text{and}\ \bM^{(i,i')}_{jl,kl'}=0\ \text{for }l\neq l'.    
\end{align*}
Basic calculation shows $\| \bM^{(i,i')} \|_2\lesssim 1/(n_0^2 d_{-m}) $ and $\| \bM^{(i,i')} \|_{\rm F}^2\lesssim 1/(n_0^4 d_{-m}) $.
By Hanson-Wright inequality, for any $t>0$,
\begin{align*}
&\PP\left(  v^\top\left( \Delta_2 - \EE \Delta_2 \right) v \ge t \right) \\
\le& 2 \exp\left( - c \min\left\{ \frac{t^2}{16 \left\| \Sigma_{-m} \right\|_2^2 \left\| \Sigma_{m} \right\|_2^2 \sum_{i,i'=1}^n \left\| \bM^{(i,i')} \right\|_{\rm F}^2  } , \frac{t }{4 \left\| \Sigma_{-m} \right\|_2 \left\| \Sigma_{m} \right\|_2 \max_{i,i'} \left\| \bM^{(i,i')} \right\|_2 }\right\} \right) \\ 
\le& 2 \exp\left( - c \min\left\{ \frac{t^2}{16 \left\| \Sigma_{-m} \right\|_2^2 \left\| \Sigma_{m} \right\|_2^2 n^2/(n_0^4 d_{-m})  } , \frac{t }{4 \left\| \Sigma_{-m} \right\|_2 \left\| \Sigma_{m} \right\|_2 /(n_0^2 d_{-m})} \right\} \right).
\end{align*}
By $\eps-net$ argument, there exist unit vectors $v_1,...,v_{5^p}$ such that for all $p\times p$ symmetric matrix $M$,
\begin{align}\label{eq:epsilon_net}
\left\| M\right\|_2 \le  4 \max_{i\le 5^p} \left| v_i^\top M v_i\right|.   
\end{align} 
See also Lemma 3 in \cite{cai2010optimal}. Then
\begin{align*}
\PP\left( \left\| \hat\Sigma_m - C_{m,\sigma} \Sigma_m \right\|_2 \ge x) \right) \le \PP\left( 4 \max_{i\le 5^{d_m}} \left| v_i^\top \left( \hat\Sigma_m - C_{m,\sigma} \Sigma_m \right) v_i\right| \ge x  \right)    \le 5^{d_m} \PP\left( 4 \left| v_i^\top \left( \hat\Sigma_m - C_{m,\sigma} \Sigma_m \right) v_i\right| \ge x  \right).
\end{align*}
As $d_m\lesssim n_0 d_{-m}$ and $n\asymp n_0$, this implies that with $x\asymp \|\Sigma_m\|_2 \|\Sigma_{-m}\|_2 \sqrt{d_m/(n_0 d_{-m})}$,
\begin{align*}
\PP\left( \left\| \hat\Sigma_m - C_{m,\sigma} \Sigma_m \right\|_2 \ge C \|\Sigma_m\|_2 \|\Sigma_{-m}\|_2 \sqrt{\frac{d_m}{n_0 d_{-m}}} \right) \le    5^{d_m} \exp\left(-c_0 d_{m} \right) \le \exp\left(-c d_m \right).
\end{align*}

\noindent Second, we prove \eqref{eqn:inverse matrix}. For simplicity, denote $\Delta_m:= \hat\Sigma_m - C_{m,\sigma} \Sigma_m=\hat\Sigma_m - \tilde \Sigma_m$ with $\tilde \Sigma_m=C_{m,\sigma} \Sigma_m$. Then write
\begin{equation*}
\hat\Sigma_m^{-1} =  \tilde\Sigma_m^{-1/2}\left( \bI_{d_m} + \tilde\Sigma_m^{-1/2}\Delta_m\tilde\Sigma_m^{-1/2}\right)^{-1} \tilde\Sigma_m^{-1/2}   .
\end{equation*}
Using Neumann series expansion, we obtain
\begin{align*}
\hat\Sigma_m^{-1} &= \tilde\Sigma_m^{-1/2} \sum_{k=0}^\infty \left(-\tilde\Sigma_m^{-1/2} \Delta_m \tilde\Sigma_m^{-1/2} \right)^k \tilde\Sigma_m^{-1/2} \\
&= \tilde\Sigma_m^{-1} + \tilde\Sigma_m^{-1/2} \sum_{k=1}^\infty \left(-\tilde\Sigma_m^{-1/2} \Delta_m \tilde\Sigma_m^{-1/2} \right)^k \tilde\Sigma_m^{-1/2}
\end{align*}
Rearranging the term, we have
\begin{align*}
\hat\Sigma_m^{-1} - \tilde\Sigma_m^{-1} = -\tilde\Sigma_m^{-1} \Delta_m \tilde\Sigma_m^{-1} + \tilde\Sigma_m^{-1} \Delta_m \tilde\Sigma_m^{-1/2} \sum_{k=0}^\infty \left(-\tilde\Sigma_m^{-1/2} \Delta_m \tilde\Sigma_m^{-1/2} \right)^k \tilde\Sigma_m^{-1/2} \Delta_m \tilde\Sigma_m^{-1}
\end{align*}
Employing similar arguments in the proof of \eqref{eqn:covariance matrix}, we can show that in an event $\Omega$ with probability at least $1-\exp(-cd_m)$,
\begin{align*}
\left\| \tilde\Sigma_m^{-1/2} \Delta_m \tilde\Sigma_m^{-1/2}  \right\|_2 \le C C_{m,\sigma}^{-1} \| \Sigma_{-m}\|_2 \sqrt{\frac{d_m}{n_0 d_{-m}} }    .
\end{align*}
As $d_m \|\Sigma_{-m}\|_2^2 \lesssim n_0 d_{-m}$, in the same event $\Omega$, 
\begin{align*}
\left\|\tilde\Sigma_m^{-1} \Delta_m \tilde\Sigma_m^{-1}\right\|_2 &\leq C C_{m,\sigma}^{-2}\|\Sigma_m^{-1}\|_2 \|\Sigma_{-m}\|_2 \sqrt{\frac{d_m}{n_0 d_{-m}}}, \\
\left\|\tilde\Sigma_m^{-1} \Delta_m \tilde\Sigma_m^{-1/2}\right\|_2 &\leq C C_{m,\sigma}^{-3/2} \norm{\Sigma_{-m}} \sqrt{\frac{d_m}{n_0 d_{-m}}}   , \\
\left\|\tilde\Sigma_m^{-1/2} \Delta_m \tilde\Sigma_m^{-1/2}\right\|_2 &\leq C C_{m,\sigma}^{-1} \| \Sigma_{-m}\|_2 \sqrt{\frac{d_m}{n_0 d_{-m}} } \le \frac12  C_{m,\sigma}^{-1} .
\end{align*}
Combining the above bounds together, we have, in the event $\Omega$,
\begin{align*}
\left\| \hat\Sigma_m^{-1} - \tilde\Sigma_m^{-1} \right\|_2 &\le C C_{m,\sigma}^{-2} \|\Sigma_m^{-1}\|_2 \| \Sigma_{-m}\|_2 \sqrt{\frac{d_m}{n_0 d_{-m}} } + C C_{m,\sigma}^{-3} \left( \| \Sigma_{-m}\|_2 \sqrt{\frac{d_m}{n_0 d_{-m}} } \right)^2 \cdot C_{m,\sigma}^{-1} \\
&\le C (C_{m,\sigma}^{-2} \vee C_{m,\sigma}^{-4})   \|\Sigma_m^{-1}\|_2 \| \Sigma_{-m}\|_2 \sqrt{\frac{d_m}{n_0 d_{-m}} } .
\end{align*}

\noindent Next, we prove \eqref{eqn:const}. Employing similar arguments in the proof of \eqref{eqn:covariance matrix}, we can show
\begin{align*}
\PP\left( | \hat\Sigma_{m,11} - C_{m,\sigma} \Sigma_{m,11}| \ge C \Sigma_{m,11}\|\otimes_{k\neq m}\Sigma_{k}\|_2  \sqrt{\frac{ t_1}{n_0 d_{-m}}} \right) \le\exp(-c_1 t_1).    
\end{align*}
Using tail probability bounds for $\chi_n^2$ (see e.g. Lemma D.2 in \cite{ma2013sparse}), we have
\begin{align*}
\PP\left( \left| \hat{\rm Var}(\cX_{1\cdots1}) - \prod_{m=1}^M\Sigma_{m,11} \right| \ge C \prod_{m=1}^M\Sigma_{m,11} \sqrt{\frac{ t_2}{n_0 }} \right) \le\exp(-c_1 t_2).    
\end{align*}
It follows that in an event with probability at least $1-\exp(-c(t_1+t_2))$,
\begin{align*}
\left|\frac{\prod_{m=1}^M \hat\Sigma_{m,11}}{\hat{\Var}(\cX_{1\cdots1})} - C_{\sigma} \right| \le  C_{M} \prod_{m=1}^M\Sigma_{m,11} \left(\max_m \frac{\|\otimes_{k\neq m}\Sigma_{k}\|_2}{\sqrt{d_{-m}}} \cdot \sqrt{\frac{ t_1}{n_0 }} +   \sqrt{\frac{ t_2}{n_0 }}   \right)
\end{align*}
where $C_{M}$ depends on $M$. As $\Var(\cX_{1\cdots 1})=\prod_{m=1}^M\Sigma_{m,11}$, we finish the proof of \eqref{eqn:const}.

\end{proof}

The following lemma presents the tail bound for the spectral norm of a Gaussian random matrix.
\begin{lemma} \label{lemma:Gaussian matrix}
Let $\bE$ be an $p_1 \times p_2$ random matrix with $\bE\sim \cM\cN_{p_1 \times p_2}(0, \; \Sigma_1, \; \Sigma_{2}).$ Then for any $t>0$, with constant $C>0$, we have
\begin{equation} \label{eqn: Gaussian tail bound}
\PP\left(\|\bE\|_2 \ge C\|\Sigma_1\|_2^{1/2} \|\Sigma_{2}\|_2^{1/2} (\sqrt{p_1} + \sqrt{p_2} + t) \right) \leq \exp(-t^2) 
\end{equation}
Let $\OO_{p_1, r} = \{\bU \in \RR^{p_1\times r}, \; \bU^\top \bU = \bI_{r} \}$ be the set of all $p_1 \times r$ orthonormal columns and let $\bU_{\perp}$ be the orthogonal complement of $\bU$. Denote $\bE_{12} = \bU^\top\bE\bV_{\perp}, \; \bE_{21} = \bU_{\perp}^\top\bE\bV$, where $\bU \in \OO_{p_1\times r}, \; \bV \in \OO_{p_2\times r}.$ We have
\begin{align}
&\PP \left( \|\bE_{21}\|_2 \geq C\|\Sigma_1\|_2^{1/2} \|\Sigma_{2}\|_2^{1/2} (\sqrt{p_1} + t) \right) \leq \exp(-t^2)   \\
& \PP \left( \|\bE_{12}\|_2 \geq C\|\Sigma_1\|_2^{1/2} \|\Sigma_{2}\|_2^{1/2} (\sqrt{p_2} + t) \right) \leq \exp(-t^2)
\end{align}
\end{lemma}

\begin{proof}
Similar to \eqref{eq:epsilon_net}, using $\eps$-net argument for unit ball (see e.g. Lemma 5 in \cite{cai2018rate}), we have
\begin{align*}
\PP \left(\|\bE\|_2 \ge 3 u \right) \le 7^{p_1 +p_2} \cdot \max_{\|x \|_2 \le 1, \|y\|_2\le 1} \PP\left( |x^\top \bE y|\ge u \right).    
\end{align*}
Decompose $\bE=\Sigma_1^{1/2}\bZ\Sigma_{2}^{1/2}$, where $\bZ \in \RR^{p_1\times p_2}, \; \bZ_{ij} \stackrel {\text{i.i.d}}{\sim} N(0, 1).$ 
Then,
\begin{align*}
x^\top \bE y=  \left( y^\top \Sigma_{2}^{1/2} \otimes x^\top \Sigma_{1}^{1/2} \right) \vec1 (\bZ) \sim N\left(0,  y^\top \Sigma_{2}y \cdot x^\top\Sigma_1 x \right)    .
\end{align*}
By Chernoff bound of Gaussian random variables,
\begin{align*}
\PP\left( |x^\top \bE y|\ge u \right) \le 2 \exp \left( -\frac{ u^2}{y^\top \Sigma_{2}y \cdot x^\top\Sigma_1 x }  \right) \le 2   \exp\left( -\frac{ u^2}{\| \Sigma_1\|_2 \|\Sigma_{2}\|_2 }  \right)  .
\end{align*}
Setting $u\asymp \|\Sigma_1\|_2^{1/2} \|\Sigma_{2}\|_2^{1/2} (\sqrt{p_1} + \sqrt{p_2} + t)$, for certain $C>0$, we have
\begin{align*}
\PP\left(\|\bE\|_2 \ge C\|\Sigma_1\|_2^{1/2} \|\Sigma_{2}\|_2^{1/2} (\sqrt{p_1} + \sqrt{p_2} + t) \right) \le 2 \cdot 7^{p_1 +p_2} \exp(-c(p_1+p_2)-t^2)  \le \exp(-t^2) .    
\end{align*}

For $\bE_{21}$, following the same $\epsilon$-net arguments, we have
\begin{align*}
\PP \left(\| \bE_{21}\|_2 \ge 3 u \right) \le 7^{(p_1-r) +r} \cdot \max_{\|x \|_2 \le 1, \|y\|_2\le 1} \PP\left( |x^\top \bE_{21} y|\ge u \right).    
\end{align*}
Note that
\begin{align*}
x^\top \bE_{21} y&=  x^\top \bU_{\perp}^\top \Sigma_1^{1/2} \bZ \Sigma_{2}^{1/2} \bV y=  \left( y^\top \bV^\top \Sigma_{2}^{1/2} \otimes x^\top \bU_{\perp}^\top \Sigma_{1}^{1/2} \right) \vec1 (\bZ) \\
&\sim N\left(0, y^\top \bV^\top \Sigma_{2} \bV y \cdot x^\top \bU_{\perp}^\top \Sigma_1 \bU_{\perp} x \right)    .
\end{align*}
By Chernoff bound of Gaussian random variables, setting $u\asymp \|\Sigma_1\|_2^{1/2} \|\Sigma_{2}\|_2^{1/2} (\sqrt{p_1} + t)$, we have
\begin{align*}
\PP\left(\|\bE_{21}\| \ge C\|\Sigma_1\|_2^{1/2} \|\Sigma_{2}\|_2^{1/2} (\sqrt{p_1} + t) \right) \le 2 \cdot 7^{p_1} \exp(-c(p_1)-t^2)  \le \exp(-t^2) .    
\end{align*}
Similarly, we can derive the tail bound of $\|\bE_{12}\|_2$.
\end{proof}

The following lemma characterizes the maximum of norms for zero-mean Gaussian tensors after any projections.
\begin{lemma} \label{lemma:Guassian tensor projection}
Let $\cE \in \RR^{d_1 \times d_2 \times d_3}$ be a Gaussian tensor, $\cE \sim \cT\cN(0; \; \frac{1}{n}\Sigma_1^{-1}, \Sigma_2^{-1}, \Sigma_3^{-1})$, where there exists a constant $C_0>0$ such that $C_0^{-1} \le \mathop{\min}\limits_{m \in \{1,2,3 \}} \lam_{\min}(\Sigma_m) \le \mathop{\max}\limits_{m \in \{1,2,3 \}} \lam_{\max}(\Sigma_m) \le C_0$. Then we have the following tail bound for projections,
\begin{align}
\label{eqn:matrice version}
\mathbb{P} & \left( \mathop{\max}\limits_{\bV_2 \in \mathbb{R}^{d_2 \times r_2}, 
\bV_3 \in \mathbb{R}^{d_3 \times r_3} \atop \|\bV_2\|_2 \le 1, \|\bV_3\|_2 \le 1} \left\| \mat1(\cE \times_2 \bV_2^\top \times_3 \bV_3^\top)\right\|_2 \ge C\|\Sigma_1^{-1/2}\|_2 \|\Sigma_{-1}^{-1/2}\|_2 \frac{\sqrt{d_1} + \sqrt{r_2r_3} + \sqrt{1+t}(\sqrt{d_2r_2} + \sqrt{d_3r_3})}{\sqrt{n}} \right) \notag\\
& \le C\exp\left(-Ct(d_2r_2 + d_3r_3) \right)
\end{align}
for any t>0. Similar results also hold for ${\rm mat}_2(\cE \times_1 \bV_1^\top \times_3 \bV_3^\top)$ and ${\rm mat}_3(\cE \times_1 \bV_1^\top \times_2 \bV_2^\top)$.

Meanwhile, we have 
\begin{align}
\label{eqn:tensor version}
\mathbb{P} & \left( \mathop{\max}\limits_{\bV_1, \bV_2, \bV_3\in \mathbb{R}^{d_m \times r_m} \atop
\|\bV_m\|_2 \le 1,\; m=1,2,3 } \left\| \cE \times_1 \bV_1^\top \times_2 \bV_2^\top \times_3 \bV_3^\top \right\|_{\rm F}^2 \ge C \|\Sigma_1^{-1}\|_2 \|\Sigma_2^{-1}\|_2 \|\Sigma_3^{-1}\|_2 \cdot \frac{r_1r_2r_3 + (1+t)(d_1r_1 + d_2r_2 + d_3r_3)}{n} \right) \notag\\
& \le \exp\left(-Ct(d_1r_1 + d_2r_2 + d_3r_3) \right)
\end{align}
for any $t>0.$
\end{lemma}

\begin{proof} 
The key idea for the proof of this lemma is via $\epsilon$-net. We first prove \eqref{eqn:matrice version}. By Lemma \ref{lemma:epsilonnet}, for $m=1,2,3$, there exists $\epsilon$-nets: $\bV_m^{(1)}, \dots ,\bV_m^{(N_m)}$ for $\{ \bV_m \in \RR^{d_m \times r_m}: \|\bV_m\|_2 \le 1 \}$, $|\cN_m| \le ((4+\eps)/\eps)^{d_mr_m}$, such that for any $\bV_m \in \RR^{d_m \times r_m}$ satisfying $\|\bV_m\|_2 \le 1$, there exists $\bV_m^{(j)}$ such that
$\|\bV_m^{(j)} - \bV_m\|_2 \le \eps.$

For fixed $\bV_2^{(i)}$ and $\bV_3^{(j)}$, we define
\begin{align*}
\bZ_1^{(ij)} = \mat1 \left( \calE \times_2 (\bV_2^{(i)})^\top \times_3 (\bV_3^{(j)})^\top \right) \in \RR^{d_1 \times (r_2r_3)}.  
\end{align*}
It is easy to obtain that
$\bZ_1^{(ij)} \sim \cM\cN_{d_1 \times r_2r_3}\left(0; \; \frac{1}{n}\Sigma_1^{-1}, \; (\bV_2^{(i)} \otimes \bV_3^{(j)}) \cdot \Sigma_{-1}^{-1} \cdot (\bV_2^{(i)} \otimes \bV_3^{(j)})^\top \right).$ Then employing similar arguments of Lemma \ref{lemma:Gaussian matrix},
\begin{equation*}
\PP\left(\|\bZ_1^{(ij)}\|_2 \le C \|\Sigma_1^{-1/2}\|_2 \|\Sigma_{-1}^{-1/2}\|_2 \left(\frac{\sqrt{d_1} + \sqrt{r_2r_3} + t}{\sqrt{n}}\right)\right) \ge 1 - 2\exp(-t^2).   
\end{equation*}
Then we further have:
\begin{equation}
\label{eqn: lemma 4 proof}
\PP\left( \max\limits_{i,j} \|\bZ_1^{(ij)}\|_2 \le C \|\Sigma_1^{-1/2}\|_2 \|\Sigma_{-1}^{-1/2}\|_2 \left(\frac{\sqrt{d_1} + \sqrt{r_2r_3} + t}{\sqrt{n}}\right) \right) \ge 1 - 2((4+\eps)/\eps)^{d_2r_2 + d_3r_3} \exp(-t^2) 
\end{equation}
for all $t>0$. Denote
\begin{align*}
\bV_2^{*}, \bV_3^{*} &= \mathop{\rm argmax}\limits_{\bV_2 \in \RR^{d_2 \times r_2}, \bV_3 \in \RR^{d_3 \times r_3} \atop
\|\bV_2\|_2 \le 1, \|\bV_3\|_2 \le 1} \left\|\mat1 \left(\cE \times_2 \bV_2^\top \times_3 \bV_3^\top\right) \right\|_2 \\
M &= \mathop{\max}\limits_{\bV_2 \in \RR^{d_2 \times r_2}, \bV_3 \in \RR^{d_3 \times r_3} \atop
\|\bV_2\|_2 \le 1, \|\bV_3\|_2 \le 1} \left\| \mat1 \left(\calE \times_2 \bV_2^\top \times_3 \bV_3^\top\right) \right\|_2
\end{align*}
Using $\eps$-net arguments, we can find $1 \le i \le N_2$ and $1 \le j \le N_3$ such that $\|\bV_2^{(i)} - \bV_2^{*}\|_2 \le \eps$ and $\|\bV_3^{(i)} - \bV_3^{*}\|_2 \le \eps$. In this case, under \eqref{eqn: lemma 4 proof},
\begin{align*}
M =& \left\| \mat1 \left( \cE \times_2 (\bV_2^{*})^\top \times_3 ( \bV_3^{*})^\top\right) \right\|_2 \\
\le & \left\|\mat1 \left(\cE \times_2 (\bV_2^{(i)})^\top \times_3 (\bV_3^{(j)})^\top\right) \right\|_2
+ \left\|\mat1 \left(\cE \times_2 (\bV_2^{*} -\bV_2^{(i)})^\top \times_3 (\bV_3^{(j)})^\top\right) \right\|_2 \\
+& \left\|\mat1 \left(\cE \times_2 (\bV_2^{*})^\top \times_3 (\bV_3^{*} - \bV_3^{(j)})^\top\right) \right\|_2 \\
\le & C \left\|\Sigma_1^{-1/2} \right\|_2 \left\|\Sigma_{-1}^{-1/2}\right\|_2 \left(\frac{\sqrt{d_1} + \sqrt{r_2r_3} + t}{\sqrt{n}} \right) + \eps M + \eps M,
\end{align*}
Therefore, we have
\begin{equation*}
\PP \left(M \le C\cdot \frac{1}{1-2\eps}\left\|\Sigma_1^{-1/2}\right\|_2 \left\|\Sigma_{-1}^{-1/2}\right\|_2 \left(\frac{\sqrt{d_1} + \sqrt{r_2r_3} + t}{\sqrt{n}}\right) \right) \ge 1 - 2((4+\eps)/\eps)^{d_2r_2 + d_3r_3} \exp(-t^2)    
\end{equation*}
By setting $\eps=1/3$, and $t^2 = 2\log(13)(d_2r_2 + d_3r_3)(1+x)$, we have proved the first part of the lemma.

\noindent Then we prove the claim \eqref{eqn:tensor version}. Consider a Gaussian tensor $\cZ \in \RR^{d_1 \times d_2 \times d_3}$ with entries $z_{ijk} \stackrel{\text{i.i.d}}{\sim} N(0, \frac{1}{n}),$ then we have $\cE \times_1 \bV_1^\top \times_2 \bV_2^\top \times_3 \bV_3^\top := \cZ \times_1 ( \bV_1^\top \Sigma_1^{-1/2}) \times_2 (\bV_2^\top \Sigma_2^{-1/2}) \times_3 (\bV_3^\top \Sigma_1^{-1/2}).$ 
By Lemma 8 in \cite{zhang2018tensor}, we know
\begin{align*}
&\PP \Bigg( \left\|\cZ \times_1 (\Sigma_1^{-1/2} \bV_1)^\top \times_2 (\Sigma_2^{-1/2} \bV_2)^\top \times_3 (\Sigma_3^{-1/2} \bV_3)^\top\right\|_{\rm F}^2 - \frac{1}{n}\left\|\left(\Sigma_1^{-1/2} \bV_1\right) \otimes \left(\Sigma_2^{-1/2} \bV_2\right) \otimes \left(\Sigma_3^{-1/2} \bV_3\right) \right\|_{\rm F}^2 \\ 
&\quad \ge \frac{2}{n} \sqrt{t \left\|\left(\bV_1^\top \Sigma_1^{-1} \bV_1\right) \otimes \left(\bV_2^\top \Sigma_2^{-1} \bV_2\right) \otimes \left(\bV_3^\top \Sigma_3^{-1} \bV_3\right) \right\|_{\rm F}^2}  + \frac{2t}{n} \left\|\left(\Sigma_1^{-1/2} \bV_1\right) \otimes \left(\Sigma_2^{-1/2} \bV_2\right) \otimes \left(\Sigma_3^{-1/2} \bV_3\right) \right\|_2^2 \Bigg) \\
&\le \exp(-t).
\end{align*}

Note that for any given $\bV_k \in \RR^{d_k \times r_k}$ satisfying $\|\bV_k\|_2 \le 1, k=1, 2, 3,$ we have 
\begin{equation*}
\left\|\left(\Sigma_1^{-1/2} \bV_1\right) \otimes \left(\Sigma_2^{-1/2} \bV_2\right) \otimes \left(\Sigma_3^{-1/2} \bV_3 \right)\right\|_2 \le \|\Sigma_1^{-1/2}\|_2 \|\Sigma_2^{-1/2}\|_2 \|\Sigma_3^{-1/2}\|_2 := C_\lam^{1/2}.    
\end{equation*}
Then,
\begin{equation*}
\left\|\left(\Sigma_1^{-1/2} \bV_1 \right) \otimes \left(\Sigma_2^{-1/2} \bV_2\right) \otimes \left(\Sigma_3^{-1/2} \bV_3 \right)\right\|_{\rm F}^2 \le C_\lam r_1r_2r_3,    
\end{equation*}
and 
\begin{align*}
\left\|\left( \bV_1^\top \Sigma_1^{-1} \bV_1\right) \otimes \left(\bV_2^\top \Sigma_2^{-1} \bV_2 \right) \otimes \left( \bV_3^\top \Sigma_3^{-1} \bV_3 \right) \right\|_{\rm F}^2 \le C_\lam^2 r_1r_2r_3,
\end{align*}
we have for any fixed $\bV_1, \bV_2, \bV_3$ and $t > 0$ that
\begin{equation*}
\PP \left( \left\|\cZ \times_1 (\Sigma_1^{-1/2} \bV_1)^\top \times_2 (\Sigma_2^{-1/2} \bV_2)^\top \times_3 (\Sigma_3^{-1/2} \bV_3)^\top\right\|_{\rm F}^2 
\ge \frac{C_\lam r_1r_2r_3 + 2C_\lam \sqrt{r_1r_2r_3 t} + 2C_\lam t}{n} \right) \le \exp(-t).    
\end{equation*}
By geometric inequality, $2\sqrt{r_1r_2r_3t} \le r_1r_2r_3 + t$, then we further have
\begin{equation*}
\PP \left( \left\|\cE \times_1 \bV_1^\top \times_2 \bV_2^\top \times_3 \bV_3^\top\right\|_{\rm F}^2 
\ge \frac{C_\lam(2r_1r_2r_3 + 3t)}{n} \right) \le \exp(-t).   
\end{equation*}

The rest proof for this claim is similar to the first part. One can find three $\eps$-nets: $\bV_m^{(1)}, \dots ,\bV_m^{(N_m)}$ for $\{ \bV_m \in \RR^{d_m \times r_m}: \|\bV_m\|_2 \le 1 \}$, $N_m \le ((4+\eps)/\eps)^{d_mr_m}$, and we have the tail bound:
\begin{align} \label{eqn:lemma 4 proof 2}
        &\max\limits_{\bV_1^{(i)}, \bV_2^{(j)}, \bV_3^{(k)}} \PP \left( \left\|\cE \times_1 (\bV_1^{(i)})^\top \times_2 (\bV_2^{(j)})^\top \times_3 (\bV_3^{(k)})^\top\right\|_{\rm F}^2 \ge \frac{C_\lam (2r_1r_2r_3 + 3t)}{n} \right) \notag \\
        &\le \exp(-t) \cdot ((4+\eps)/\eps)^{d_1r_1 + d_2r_2 + d_3r_3} ,
\end{align}
for all $t > 0$. Assume
\begin{align*}
\bV_1^{*}, \bV_2^{*}, \bV_3^{*} &= \mathop{\rm argmax}\limits_{\bV_m \in \mathbb{R}^{d_m \times r_m} \atop
\|\bV_m\|_2 \le 1} \left\|\cE \times_1 \bV_1^\top \times_2 \bV_2^\top \times_3 \bV_3^\top\right\|_{\rm F}^2 \\
T &= \left\|\cE \times_1 (\bV_1^*)^\top \times_2 (\bV_2^*)^\top \times_3 (\bV_3^*)^\top\right\|_{\rm F}^2
\end{align*}
Then we can find $\bV_1^{(i)}, \bV_2^{(j)}, \bV_3^{(k)}$ in the corresponding $\eps$-nets such that $\|\bV_m^{*} - \bV_m\|_2 \le \eps$, $m=1,2,3$. And
\begin{align*}
T =& \left\|\cE \times_1 (\bV_1^*)^\top \times_2 (\bV_2^*)^\top \times_3 (\bV_3^*)^\top\right\|_{\rm F}^2 \\
\le & \left\|\cE \times_1 (\bV_1^{(i)})^\top \times_2 (\bV_2^{(j)})^\top \times_3 (\bV_3^{(k)})^\top\right\|_{\rm F}^2
+ \left\|\cE \times_1 (\bV_1^{(i)} - \bV_1^{*})^\top \times_2 (\bV_2^{*})^\top \times_3 (\bV_3^{*})^\top\right\|_{\rm F}^2 \\
+ & \left\|\cE \times_1 (\bV_1^{(i)})^\top \times_2 (\bV_2^{(j)} - \bV_2^*)^\top \times_3 (\bV_3^*)^\top\right\|_{\rm F}^2
+ \left\|\cE \times_1 (\bV_1^{(i)})^\top \times_2 (\bV_2^{(j)})^\top \times_3 (\bV_3^{(k)} - \bV_3^{*})^\top\right\|_{\rm F}^2 \\
\le & \frac{C_\lam (2r_1r_2r_3 + 3t)}{n} + 3\eps T
\end{align*}
which implies 
\begin{align*}
\left\|\cE \times_1 (\bV_1^*)^\top \times_2 (\bV_2^*)^\top \times_3 (\bV_3^*)^\top\right\|_{\rm F}^2 \le \frac{C_\lam \cdot (2r_1r_2r_3 + 3t)}{n(1-3\eps)}.    
\end{align*}
If we set $\eps = 1/9$ and $t=(1+x)\log(37)\cdot (d_1r_1 + d_2r_2 + d_3r_3)$, by \eqref{eqn:lemma 4 proof 2} we have proved \eqref{eqn:tensor version}. 
\end{proof}

\begin{lemma}\label{lemma:low-rank-tensor}
Suppose that $\cX_1,...,\cX_n \in \RR^{d_1\times\cdots\times d_M}$ are i.i.d. $\cT\cN(\mu,\bSigma)$ with $\bSigma=[\Sigma_m]_{m=1}^M$, and $\bar\cX$ is the sample mean. Then, with probability at least $1-e^{-c\sum_{m=1}^M d_m r_m}$,
\begin{align}
\sup_{\substack{\cV\in \RR^{d_1\times\cdots\times d_M}, \|\cV\|_{\rm F}\le 1 \\ \text{rank}(\matk(\cV))\le r_m, \forall m\le M }}  \left| <\bar\cX-\mu,\cV> \right| \lesssim \sqrt{\frac{\sum_{m=1}^M d_m r_m+ r}{n}},
\end{align}
where $r=r_1r_2\cdots r_M$.
\end{lemma}

\begin{proof}
For any tensor $\cV$, consider a Tucker decomposition, $\cV=\cS\times_{m=1}^M U_m$ with $\cS \in \RR^{r_1\times\cdots\times r_M}$ and $U_m$ are unitary matrices. Let $N_0=\lfloor(1+2/\epsilon)^{r}\rfloor$ and $N_m=\lfloor(1+2/\epsilon)^{d_mr_m}\rfloor$. There exists $\epsilon$-nets $\cS_{j_0}^* \in \RR^{r_1\times\cdots\times r_M}$ with $\|\cS_{j_0}^*\|_{\rm F}\le 1, j_0=1,...,N_0$, and $U_{m,j_m}^* \in \RR^{d_m\times r_m}$ with $\|U_{m,j_m}^*\|_{2}\le 1,  j_m=1,...,N_m, 1\le m\le M$, such that
\begin{align*}
\max_{\|\cS\|_{\rm F}\le 1} \min_{1\le j_0\le N_0} \|\cS-\cS_{j_0}^*\|_{\rm F}\le \epsilon, \qquad \max_{\|U_m\|_{2}\le 1} \min_{1\le j_m\le N_m} \|U_m- U_{m,j_m}^*\|_{2}\le \epsilon, 1\le m\le M.   
\end{align*}
Note that $\|\cV\|_{\rm F}\le 1$ is equivalent to $\|\cS\|_{\rm F}\le 1$. Let $\cY=\bar X-\mu$.
Then by the ``subtraction argument",
\begin{align*}
&\sup_{\substack{\cV\in \RR^{d_1\times\cdots\times d_M}, \|\cV\|_{\rm F}\le 1 \\ \text{rank}(\matk(\cV))\le r_m, \forall m\le M }} \left| <\cY,\cV> \right| -   \max_{\substack{\|\cS_{j_0}^*\|_{\rm F}\le 1, \|U_{m,j_m}^*\|_{2}\le 1 ,\\ j_0\le N_0,   j_m\le N_m, \forall 1\le m\le M}} \left| <\cY,\cS_{j_0}^*\times_{m=1}^M U_{m,j_m}^*> \right|   \\
=& \sup_{\substack{\cV\in \RR^{d_1\times\cdots\times d_M}, \|\cV\|_{\rm F}\le 1 \\ \text{rank}(\matk(\cV))\le r_m, \forall m\le M }} \left| <\cY,\cV> \right| -   \max_{\substack{\|\cS_{j_0}^*\|_{\rm F}\le 1, j_0\le N_0}} \left| <\cY,\cS_{j_0}^*\times_{m=1}^M U_{m}> \right| \\
&+\sum_{k=0}^{M-1} \left( \max_{\substack{\|\cS_{j_0}^*\|_{\rm F}\le 1, \|U_{m,j_m}^*\|_{2}\le 1 ,\\ j_0\le N_0,   j_m\le N_m, \forall m\le k}} \left| <\cY,\cS_{j_0}^*\times_{m=1}^k U_{m,j_m}^*\times_{m=k+1}^M U_{m}> \right| \right.\\
&\qquad\quad \left. - \max_{\substack{\|\cS_{j_0}^*\|_{\rm F}\le 1, \|U_{m,j_m}^*\|_{2}\le 1 ,\\ j_0\le N_0,   j_m\le N_m, \forall m\le k+1}} \left| <\cY,\cS_{j_0}^*\times_{m=1}^{k+1} U_{m,j_m}^*\times_{m=k+2}^M U_{m}> \right| \right)\\
\le& (M+1)\epsilon  \sup_{\substack{\cV\in \RR^{d_1\times\cdots\times d_M}, \|\cV\|_{\rm F}\le 1 \\ \text{rank}(\matk(\cV))\le r_m, \forall m\le M }} \left| <\cY,\cV> \right|
\end{align*}
Setting $\epsilon=1/(2M+2)$,
\begin{align*}
\sup_{\substack{\cV\in \RR^{d_1\times\cdots\times d_M}, \|\cV\|_{\rm F}\le 1 \\ \text{rank}(\matk(\cV))\le r_m, \forall m\le M }} \left| <\cY,\cV> \right| \le 2     \max_{\substack{\|\cS_{j_0}^*\|_{\rm F}\le 1, \|U_{m,j_m}^*\|_{2}\le 1 ,\\ j_0\le N_0,   j_m\le N_m, \forall 1\le m\le M}} \left| <\cY,\cS_{j_0}^*\times_{m=1}^M U_{m,j_m}^*> \right| .
\end{align*}
It follows that
\begin{align*}
\PP\left(\sup_{\substack{\cV\in \RR^{d_1\times\cdots\times d_M}, \|\cV\|_{\rm F}\le 1 \\ \text{rank}(\matk(\cV))\le r_m, \forall m\le M }} \left| <\cY,\cV> \right| \ge x\right) 
\le& \PP\left(     \max_{\substack{\|\cS_{j_0}^*\|_{\rm F}\le 1, \|U_{m,j_m}^*\|_{2}\le 1 ,\\ j_0\le N_0,   j_m\le N_m, \forall 1\le m\le M}} \left| <\cY,\cS_{j_0}^*\times_{m=1}^M U_{m,j_m}^*> \right|  \ge x/2 \right)\\
\le& \prod_{k=0}^M N_k \cdot \PP\left(     \left| <\cY,\cS_{j_0}^*\times_{m=1}^M U_{m,j_m}^*> \right|  \ge x/2 \right) \\
\le& (4M+5) ^{\sum_{m=1}^M d_m r_m +r} \cdot \PP\left(     \left| <\cY,\cS_{j_0}^*\times_{m=1}^M U_{m,j_m}^*> \right|  \ge x/2 \right).
\end{align*}
Since $\cY=\bar X-\mu\sim \cT\cN(0,\frac1n \bSigma)$, we can show $| <\cY,\cS_{j_0}^*\times_{m=1}^M U_{m,j_m}^*> |$ is a $n^{-1/2}\prod_{m=1}^M \|\Sigma_m\|_{2}^{1/2}$ Lipschitz function, and $\E{| <\cY,\cS_{j_0}^*\times_{m=1}^M U_{m,j_m}^*> |}\le \sqrt{2/\pi} n^{-1/2}\prod_{m=1}^M \|\Sigma_m\|_{2}^{1/2}$. Then, by Gaussian concentration inequalities for Lipschitz functions,
\begin{align*}
\PP\left(     \left| <\cY,\cS_{j_0}^*\times_{m=1}^M U_{m,j_m}^*> \right|  \ge \E{| <\cY,\cS_{j_0}^*\times_{m=1}^M U_{m,j_m}^*> |} + t \right) \le \exp\left( -\frac{nt^2}{\prod_{m=1}^M \|\Sigma_m\|_{2}} \right).
\end{align*}
Setting $x\asymp t\asymp \sqrt{(\sum_{m=1}^M d_m r_m +r)/n}$, in an event with at least probability at least $1-e^{-c\sum_{m=1}^M d_m r_m}$,
\begin{align*}
\sup_{\substack{\cV\in \RR^{d_1\times\cdots\times d_M}, \|\cV\|_{\rm F}\le 1 \\ \text{rank}(\matk(\cV))\le r_m, \forall m\le M }}  \left| <\bar\cX-\mu,\cV> \right| \lesssim \sqrt{\frac{\sum_{m=1}^M d_m r_m+ r}{n}}.
\end{align*}
\end{proof}

The following lemma gives an inequality in terms of Frobenius norm between two tensors. 
\begin{lemma} 
\label{lemma:tensor norm inequality}
For two $M$-th order tensors $\gamma, \; \hat \gamma \in \RR^{d_1 \times \cdots \times d_M}$, if $\norm{\gamma - \hat\gamma}_{\rm F} = o\left( \norm{\gamma}_{\rm F} \right)$ as $n \rightarrow \infty$, and $\norm{\gamma}_{\rm F} \ge c$ for some constant $c>0$, then when $n \rightarrow \infty$,
\begin{equation*}
\norm{\gamma}_{\rm F} \cdot \norm{\hat\gamma}_{\rm F} - \langle \gamma ,\; \hat \gamma \rangle \asymp \norm{\gamma - \hat\gamma}_{\rm F}^2.    
\end{equation*}
\end{lemma}

\begin{proof}
Let $\calE = \hat \gamma - \gamma$, when $\norm{\gamma - \hat \gamma}_{\rm F} = o(\|\gamma\|_F)$ and $\norm{\gamma}_{\rm F} \ge c,$ we have
\begin{align*}
\norm{\gamma}_{\rm F} \cdot \norm{\hat \gamma}_{\rm F} - \langle \gamma, \; \hat \gamma \rangle &= \norm{\gamma}_{\rm F} \cdot \norm{\gamma + \calE}_{\rm F} - \langle \gamma, \; \gamma+\calE \rangle \\
&= \norm{\gamma}_{\rm F} \sqrt{\norm{\gamma}_{\rm F}^2 + 2\langle \gamma, \; \calE \rangle + \norm{\calE}_{\rm F}^2} - \norm{\gamma}_{\rm F}^2 - \langle \gamma, \; \calE \rangle \\
&= \norm{\gamma}_{\rm F}^2 \sqrt{1 + \frac{2\langle \gamma, \; \calE \rangle + \norm{\calE}_{\rm F}^2}{\norm{\gamma}_{\rm F}^2}} - \norm{\gamma}_{\rm F}^2 - \langle \gamma, \; \calE \rangle \\
&\asymp \norm{\gamma}_{\rm F}^2 \big( 1+\frac{1}{2} \frac{2\langle \gamma, \; \calE \rangle + \norm{\calE}_{\rm F}^2}{\norm{\gamma}_{\rm F}^2} \big) - \norm{\gamma}_{\rm F}^2 - \langle \gamma, \; \calE \rangle \\
&= \frac{\norm{\calE}_{\rm F}^2}{2} \asymp \norm{\hat \gamma - \gamma}_{\rm F}^2.
\end{align*}
\end{proof}

The following lemma illustrates the relationship between the risk function $R_{\btheta}(\hat\delta) -R_{\rm opt}(\btheta)$ and a more “standard” risk function $L_{\theta}(\hat \delta)$, which fulfills a role similar to that of the triangle inequality, as demonstrated in Lemma \ref{lemma:probability inequality}. Lemmas \ref{lemma:the first reduction} and \ref{lemma:probability inequality} correspond to Lemmas 3 and 4, respectively, in \cite{cai2019high}.

\begin{lemma} \label{lemma:the first reduction}
Let $\cZ \sim \frac{1}{2}\cT\cN(\cM_1; \; \bSigma) + \frac{1}{2}\cT\cN(\cM_2; \; \bSigma)$ with parameter $\theta = (\cM_1, \; \cM_2, \; \bSigma)$ where $\bSigma= [\Sigma_m]_{m=1}^M$. If a classifier $\hat \delta$ satisfies $L_{\theta}(\hat \delta) = o(1)$ as $n \rightarrow \infty$, then for sufficiently large n, 
\begin{align*}
R_{\btheta}(\hat\delta_{\rm tucker}) -R_{\rm opt}(\btheta) \ge \frac{\sqrt{2\pi}\Delta}{8} e^{\Delta^2/8} \cdot L_{\theta}^2(\hat \delta_{\rm tucker}) .    
\end{align*}
\end{lemma}

\begin{lemma} \label{lemma:probability inequality}
Let $\theta = (\cM, \; -\cM, \; [I_{d_m}]_{m=1}^M)$ and $\Tilde{\theta} = (\Tilde{\cM}, \; -\Tilde{\cM}, \; [I_{d_m}]_{m=1}^M)$ with $\norm{\cM}_{\rm F} = \norm{\Tilde{\cM}}_{\rm F} = \Delta/2$. For any classifier $\delta$, if $\norm{\cM - \Tilde{\cM}}_{\rm F} = o(1)$ as $n \rightarrow \infty$, then for sufficiently large n,
\begin{align*}
L_{\theta}(\delta) + L_{\Tilde{\theta}}(\delta) \ge \frac{1}{\Delta} e^{-\Delta^2/8} \cdot \norm{\cM - \Tilde{\cM}}_{\rm F} .    
\end{align*}
\end{lemma}


Although $L_{\theta}(\hat \delta)$ is not a distance function and does not satisfy an exact triangle inequality, the following lemma provides a variant of Fano's lemma. It suggests that it suffices to provide a lower bound for $L_{\theta}(\hat \delta)$, and $L_{\theta}(\hat \delta)$ satisfies an approximate triangle inequality (Lemma \ref{lemma:probability inequality}). 

\begin{lemma}[\cite{tsybakov2009}] \label{lemma:Tsybakov variant}
Let $N \ge 2$ and $\theta_0, \theta_1, \ldots ,\theta_N \in \Theta_d$. For some constants $\alpha_0 \in (0, 1/8), \gamma > 0$ and any classifier $\hat\delta$, if ${\rm KL}(\PP_{\theta_i}, \PP_{\theta_j}) \le \alpha_0 \log N/n$ for all $1 \le i \le N$, and $L_{\theta_i}(\hat\delta) < \gamma$ implies $L_{\theta_j}(\hat\delta) \ge \gamma$ for all $0 \le i \neq j \le N$, then 
\begin{align*}
\inf_{\hat\delta} \sup_{i \in [N]} \PP_{\theta_i}(L_{\theta_i}(\hat\delta)) \ge \gamma) \ge \frac{\sqrt{N}}{1+\sqrt{N}} (1-2\alpha_0-\sqrt{\frac{2\alpha_0}{\log N}}) >0  .    
\end{align*}
\end{lemma}

\begin{lemma}[Varshamov-Gilbert Bound, \cite{tsybakov2009}] \label{lemma:Varshamov-Gilbert Bound}
Consider the set of all binary sequences of length m: $\Omega = \big\{ \omega= (\omega_1,\ldots,\omega_m), \omega_i \in \{ 0,1\} \big\} = \{0,1\}^m$. Let $m \ge 8$, then there exists a subset $\{ \omega^{(0)}, \omega^{(1)},$ $\ldots, \omega^{(N)} \}$ of $\Omega$ such that $\omega^{(0)}=(0,\ldots, 0)$, $\rho_H(\omega^{(i)}, \omega^{(j)}) \ge m/8, \forall 0 \le i < j \le N$, and $N \ge 2^{m/8}$, where $\rho_H$ denotes the Hamming distance.
\end{lemma}